\documentclass{article}

\usepackage{silence} 
\usepackage{microtype}
\usepackage{graphicx}
\usepackage{subfigure}
\usepackage{booktabs} %

\usepackage{hyperref}

\usepackage[accepted]{icml2025}

\usepackage{src/macros}

\usepackage{amsmath}
\usepackage{amssymb}
\usepackage{mathtools}
\usepackage{amsthm}

\usepackage[capitalize,noabbrev]{cleveref}

\theoremstyle{plain}
\newtheorem{theorem}{Theorem}[section]
\newtheorem{proposition}[theorem]{Proposition}
\newtheorem{lemma}[theorem]{Lemma}
\newtheorem{corollary}[theorem]{Corollary}
\newtheorem{definition}[theorem]{Definition}

\theoremstyle{remark}
\newtheorem{remark}[theorem]{Remark}

\usepackage[textsize=tiny]{todonotes}

\icmltitlerunning{Mechanisms of Projective Composition}

\begin{document}

\twocolumn[
\icmltitle{Mechanisms of Projective Composition of Diffusion Models}

\icmlsetsymbol{equal}{*}

\begin{icmlauthorlist}
\icmlauthor{Arwen Bradley}{equal,comp}
\icmlauthor{Preetum Nakkiran}{equal,comp}
\icmlauthor{David Berthelot}{comp}
\icmlauthor{James Thornton}{comp}
\icmlauthor{Joshua M. Susskind}{comp}
\end{icmlauthorlist}

\icmlaffiliation{comp}{Apple, Cupertino, CA, USA}

\icmlcorrespondingauthor{Arwen Bradley}{arwen\_bradley@apple.com}
\icmlcorrespondingauthor{Preetum Nakkiran}{p\_nakkiran@apple.com}

\icmlkeywords{Machine Learning, ICML}

\vskip 0.3in
]

\printAffiliationsAndNotice{\icmlEqualContribution} %

\begin{abstract}

We study the theoretical foundations of composition in diffusion models, with a particular focus on out-of-distribution extrapolation and length-generalization.
Prior work has shown that composing distributions
via linear score combination can achieve promising results,
including length-generalization in some cases \citep{du2023reduce,liu2022compositional}.
However, our theoretical understanding of how and why such compositions work remains incomplete. In fact, it is not even entirely clear what it means for composition to ``work''.
This paper starts to address these fundamental gaps.
We begin by precisely defining one possible desired result of composition, which we
call \emph{projective composition}.
Then, we investigate: (1) when linear score combinations provably achieve projective composition, (2) whether reverse-diffusion sampling can generate the desired composition, and (3) the conditions under which composition fails. We connect our theoretical analysis to prior empirical observations where composition has either worked or failed, for reasons that were unclear at the time. Finally, we propose a simple heuristic to help predict the success or failure of new compositions.
\end{abstract}

\begin{figure}[!ht]
\vskip 0.2in
\begin{center}
\centerline{\includegraphics[width=\columnwidth]{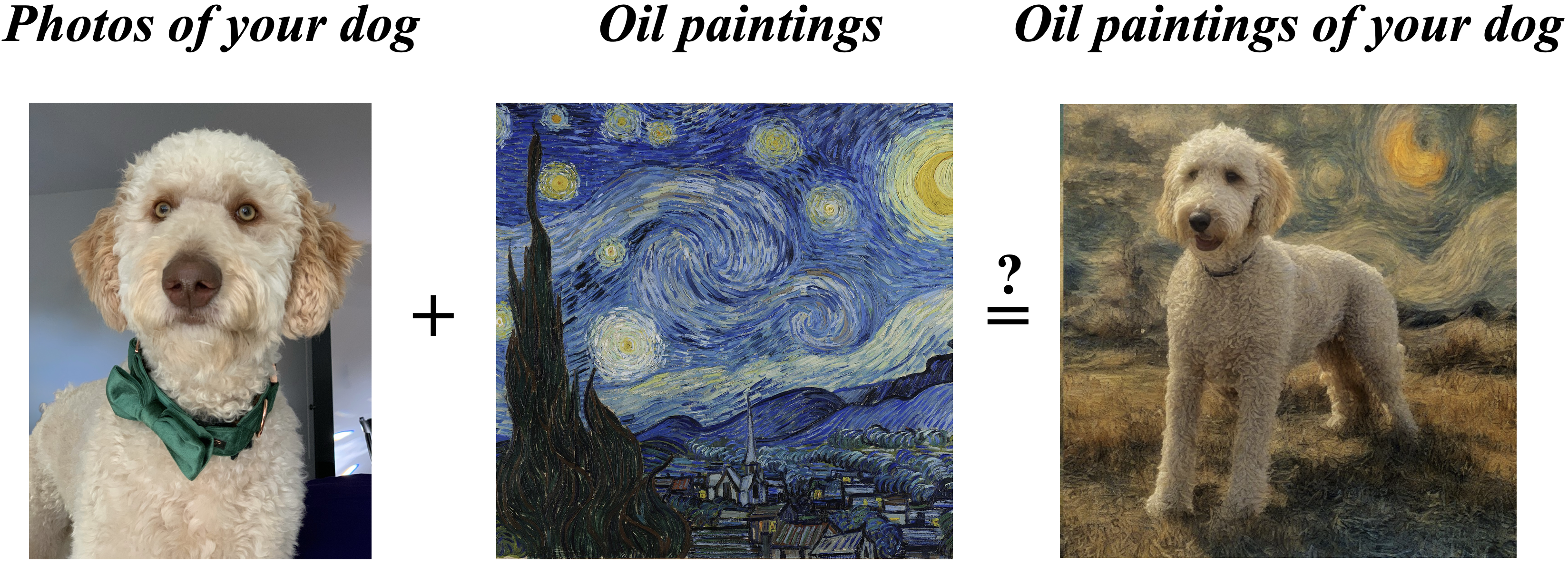}}
\caption{Composing diffusion models via score combination.
Given two diffusion models, it is sometimes
possible to sample in a way that
composes content from one model (e.g. your dog) 
with style of another model (e.g. oil paintings). We aim to theoretically understand this empirical behavior.
Figure generated via score composition
with SDXL fine-tuned on the author's dog; details in Appendix~\ref{app:sdxl_detail}.
}
\label{fig:style-content}
\end{center}
\vskip -0.2in
\end{figure}
\begin{figure}[!htb]
\vskip 0.1in
\begin{center}
\centerline{\includegraphics[width=\columnwidth]{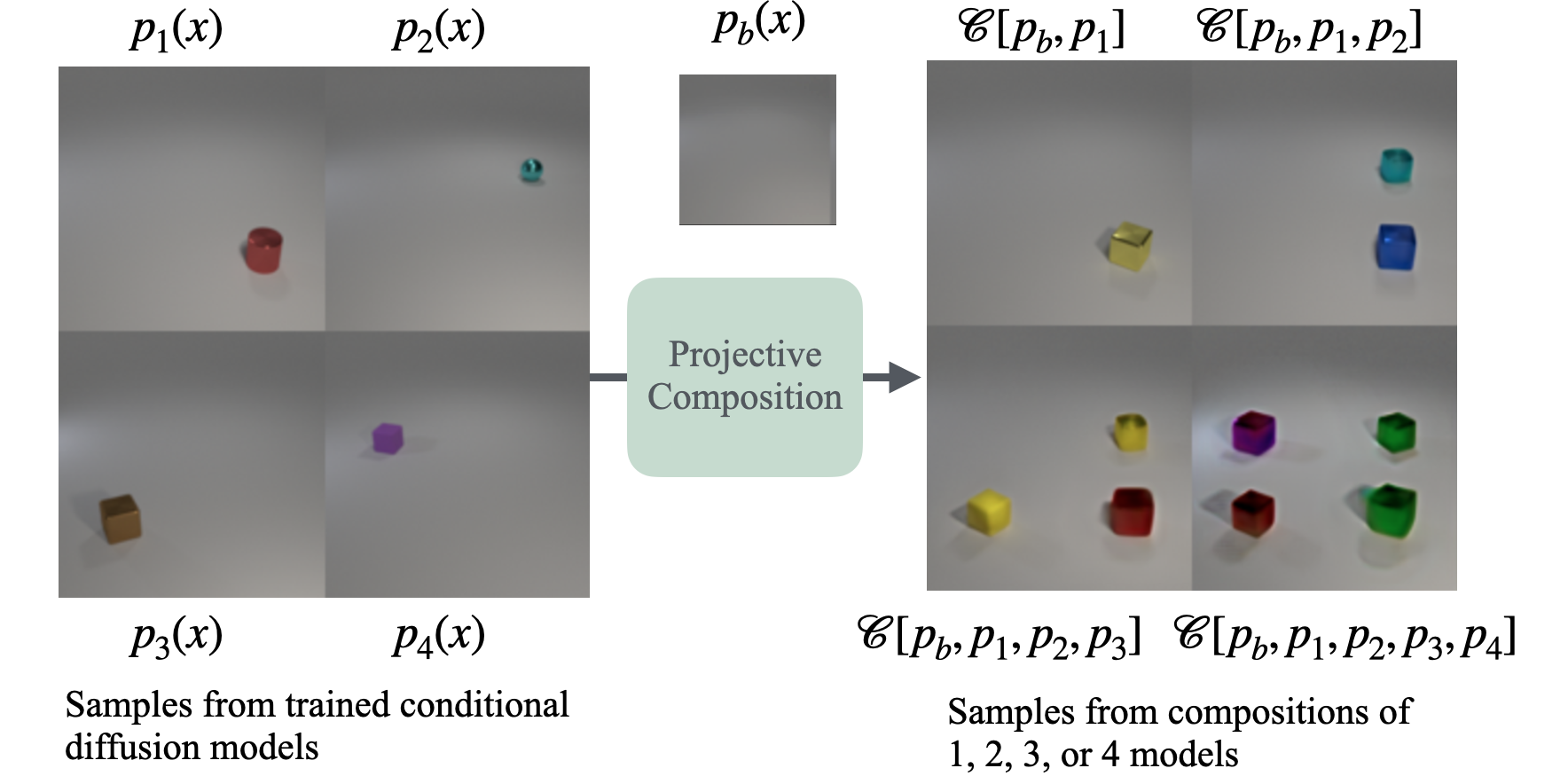}}
\caption{Length-generalization, another capability of composition enabled by our framework. Diffusion models trained to generate a single object conditioned on location (left) can be composed at inference-time to generate images
of multiple objects at specified locations (right).
Notably, such images are strictly out-of-distribution
for the individual models being composed. (Additional samples in Figure \ref{fig:len_gen_extra}.)
}
\label{fig:len_gen}
\end{center}
\vskip -0.3in
\end{figure}

\section{Introduction}

The possibility of \emph{composing} different concepts represented by pretrained models has been of both theoretical and practical interest for some time \citep{jacobs1991adaptive, hinton2002training, du2024compositional}, with diverse applications including image and video synthesis \cite{du2023reduce, du2020visualenergy, liu2022compositional, liu2021learning, nie2021controllable, yang2023probabilistic, wang2024concept}, planning \cite{ajay2024compositional, janner2022planning}, constraint satisfaction \citep{yang2023compositional}, parameter-efficient training \citep{hu2022lora, ilharco2022editing}, and many others \citep{wu2024compositional, su2024decomposition, urain2023composable, zhang2025scaling}.
One central goal in this field is to build novel compositions at inference time using only the outputs of pretrained models (either entirely separate models, or different conditionings of a single model), to create generations that are potentially more complex than any model could produce individually.
As a concrete example to keep in mind, suppose we have two diffusion models,
one trained on your personal photos of your dog and another trained on a collection of oil paintings,
and we want to somehow combine these to generate oil paintings of your dog.
Note that in order to achieve this goal,
compositions must be able to generate images that are out-of-distribution (OOD) with respect to each of the individual models, since for example,
there was no oil painting of your dog in either model's training set.
Prior empirical work has shown that this ambitious vision is at least partially achievable in practice.
However, the theoretical foundations of how and why composition works in practice, as well as its limitations, are still incomplete.

The goal of this work is to advance our theoretical understanding of composition---
we will take a specific family of methods used for composing diffusion models,
and we will analyze conditions under which this method 
provably generates the ``correct'' composition.
Specifically, are there sufficient properties of the distributions we are composing that can guarantee that composition will work ``correctly''?
And what does correctness even mean, formally?

We focus our study on
composing diffusion models by linearly combining their scores,
a method introduced by \citet{du2023reduce, liu2022compositional}
(though many other interesting constructions are possible, see Section \ref{sec:related_work}).
Concretely, suppose we have three separate diffusion models, one for the
distribution of dog images $p_{dog}$,
another for oil-paintings $p_{oil}$,
and another unconditional model for generic images $p_u$.
Then, we can use the individual score estimates $\grad_x \log p(x)$ 
given by the models to construct a composite score:
{\begin{align}
    \label{eq:intro-style-content}
    \grad_x &\log \hat p(x) := \\
    &\grad_x \log p_{dog}(x) + \grad_x \log p_{oil}(x) - \grad_x \log p_u(x).\notag
\end{align}
}%
This implicitly defines a distribution which we will call a ``product composition'':
$\hat{p}(x) \propto p_{dog}(x)p_{oil}(x)/p_u(x)$.
Finally, we can try to sample from $\hat{p}$
by using these scores with a generic score-based sampler,
or even reverse-diffusion.
This method of composition often achieves good results in practice,
yielding e.g. oil paintings of dogs,
but it is unclear why it works theoretically.

We are particularly interested in the OOD generalization capabilities of this style of composition. By this we mean the compositional method's ability to generate OOD with respect to each of the individual models being composed -- which may be possible even if none of the individual models are themselves capable of OOD generation.
A specific desiderata is \emph{length-generalization}, understood as the ability to compose arbitrarily many concepts.
For example, consider the CLEVR \citep{johnson2017clevr} setting shown in Figure \ref{fig:len_gen}.
Given conditional models trained on images each containing a single object
and conditioned on its location,
we want to generate images containing $k > 1$ objects composed in the same scene.
How could such length-generalizing composition be possible?
Here is one illustrative toy example---
consider the following construction, inspired by but slightly different from \citet{du2023reduce, liu2022compositional}.
Suppose $p_b$ is a distribution of empty background images,
and each $p_i$ a distribution of images with a single object at location $i$,
on an otherwise empty background. Assume all locations we wish to compose are non-overlapping.
Then, performing reverse-diffusion sampling using the following score-composition will work
--- meaning will produce images with $k$ objects at appropriate locations:
\begin{align}
    \label{eq:product_comp}
    \grad_x \log p^t_b(x)
    + \sum_{i=1}^k \underbrace{\left( \grad_x \log p^t_i(x) - \grad_x \log p^t_b(x) \right)}_{\textrm{score delta $\delta_i \in \R^n$}}.
\end{align}
Above, the notation $p_i^t$ denotes the
distribution $p_i$
after time $t$ in the forward diffusion
process (see Appendix \ref{app:samplers}).
Intuitively this works because during the reverse-diffusion process,
the update performed by model $i$
modifies only pixels in the vicinity of location $i$, and otherwise leaves them identical to the background.
Thus the different models do not interact, and the sampler acts as if each model individually ``pastes''
an object onto an empty background.
Formally, sampling works
because the score delta vectors $\delta_i$ are mutually orthogonal, and in fact have disjoint supports.
Notably, we can sample from this composition with 
a \emph{standard diffusion sampler}, in contrast to \citet{du2023reduce}'s observations that more sophisticated samplers are necessary.
This construction would not be guaranteed to work, however, if the ``background'' $p_b$
was chosen to be the unconditional distribution $p_u$
(as in Equation~\ref{eq:intro-style-content}),
a common choice in many prior works \citep{du2023reduce, liu2022compositional}.

The remainder of this paper is devoted to trying to generalize this example as far as possible,
and understand both its explanatory power and its limitations.
It turns out the core mechanism can be generalized surprisingly far, 
and does not depend on ``orthogonality'' as strongly as the above example may suggest.
We will encounter some subtle aspects along the way, starting from 
formal definitions
of what it means for composition to succeed ---
a definition that
can capture both composing objects (as in Figure~\ref{fig:len_gen}),
and composing other attributes (such as style + content, in Figure~\ref{fig:style-content}).

\begin{figure*}[t]
\vskip 0.2in
\begin{center}
\centerline{
\includegraphics[width=1.0\textwidth]{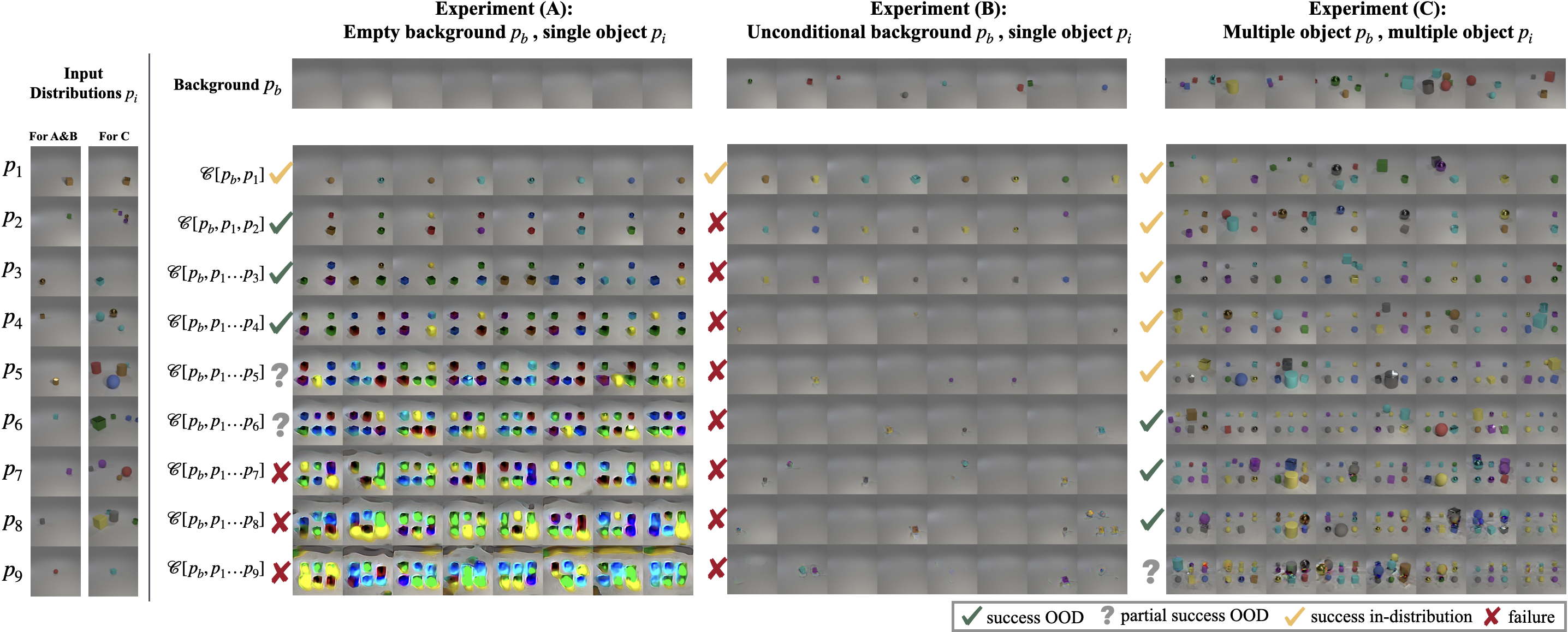}
}
\caption{\textbf{Attempted compositional length-generalization up to 9 objects.} We attempt to compose via linear score combination the distributions $p_1$ through $p_9$ shown on the far left, where each $p_i$ is conditioned on a specific object location
as described below.
Settings (A) and (C) approximately satisfy the conditions of our theory of projective composition, and thus are expected to length-generalize at least somewhat, while setting (B) does not even approximately satisfy our conditions and indeed fails to length-generalize. 
{\bf Experiment (A):} In this experiment, the distributions $p_i$ each contain a single object at a fixed location, and the background $p_b$ is empty. In this case any successful composition of more than one object represents length-generalization.
We find that composition succeeds up to several objects,
but then degrades as number of objects increases
(see \cref{sec:clevr-details} for details). 
{\bf Experiment (B):} Here the distributions $p_i$ are identical to (A),
but the background $p_b$ is chosen
as the unconditional distribution
(i.e. a single object at a random location)--- this the ``Bayes composition'' (\cref{sec:problematic-compositions}).
This composition entirely fails--- remarkably, trying
to compose many objects often produces no objects!
{\bf Experiment (C):}
Here each distribution $p_i$ contains
an object at a fixed location $i$, and 0-4 other objects
(sampled uniformly) in random locations.
The background distribution $p_b$ is a distribution of 1-5 objects (sampled uniformly) in random locations.
In this case length-generalization means composition of more than 5 objects.
This composition can length-generalize,
but artifacts appear for large numbers of objects. %
See \cref{sec:clevr-details} for further discussion, and \cref{tab:clevr_counts} for a quantitative analysis. In \Cref{app:clevr_multi} we explore another configuration that improves length-generalization quality.}
\label{fig:len_gen_monster}
\end{center}
\vskip -0.2in
\end{figure*}

\subsection{Contributions and Organization}
In this work we introduce a theoretical framework
to help understand 
the empirical success of certain methods of 
composing diffusion models,
with emphasis on
understanding how compositions can sometimes
length-generalize. 
We start by discussing the limitations
of several prior definitions of composition
in Section~\ref{sec:problematic-compositions}.
In Section~\ref{sec:composition} we offer a formal definition of
``what we want composition to do'',
given precise information about which aspects we want to compose, which we call \emph{Projective Composition} (Definition~\ref{def:proj_comp}).
(Note that there are many other valid notions of composition;
we are merely formalizing one particular goal.)
Then,
we study how projective composition can be achieved.
In Section~\ref{sec:comp_coord} we introduce formal criteria called
\emph{Factorized Conditionals} (Definition~\ref{def:factorized}),
which is a type of independence criteria along both distributions and coordinates. We prove that when this criteria holds, 
projective composition
can be achieved by linearly combining scores (as in Equation~\ref{eq:product_comp}),
and can be sampled via standard reverse-diffusion samplers.
In Section~\ref{sec:comp_feature} we show that parts of this result can be extended much further to apply even in
nonlinear feature spaces; but interestingly, even when projective composition is achievable, it may be difficult to sample.
We find that in many important cases existing constructions approximately satisfy our conditions, but the theory also helps characterize and explain certain limitations.
Finally in Section~\ref{sec:practical} we discuss how our results can help explain existing experimental results in the literature where composition worked or failed, for reasons that were unclear at the time. We also suggest a simple practical heuristic to help predict whether new sets of concepts will compose correctly.

\section{Related Work}
\label{sec:related_work}

\textbf{Single vs. Multiple Model Composition.}
First, we distinguish the kind of composition we study in this paper from 
approaches that rely a single model 
but with OOD conditioners;
for example, passing OOD text prompts to text-to-image models \citep{nichol2021glide, podell2023sdxl},
or works like \citet{okawa2024compositional, park2024emergence, bradley2025local}.
In contrast, we study compositions which recombine the outputs of
multiple separate models at inference time, where each model only sees
in-distribution conditionings.
Among compositions involving multiple models,
many different variants have been explored. %
Some are inspired by logical operators like AND and OR,
which are typically implemented as
product $p_0(x)p_1(x)$ and sum $p_0(x) + p_1(x)$ \citep{du2023reduce,du2024compositional,liu2022compositional}.
Some composition methods are based on diffusion models, 
while others use energy-based models \citep{du2020visualenergy, du2023reduce, liu2021learning} or densities \cite{skreta2024superposition}. In this work, we focus specifically on product-style compositions %
implemented with diffusion models via a linear combinations of scores as in \citet{du2023reduce, liu2022compositional}. Our goal is not to propose
a new method of composition but to improve theoretical understanding of existing methods.

\textbf{Learning and Generalization.}
In this work we focus only on mathematical aspects
of composition,
and we do not consider any learning-theoretic aspects
such as inductive bias or sample complexity.
Our work is thus complementary to \citet{kamb2024analytic},
which studies how a type of compositional generalization
can arise from inductive bias in the learning procedure.
Additional related works are discussed in Appendix~\ref{app:related}.

\section{Prior Definitions of Composition}
\label{sec:problematic-compositions}

In this section we will describe why two popular mathematical definitions of composition are insufficient
for our purposes: the ``simple product'' definition, and the Bayes composition.
Specifically, neither of these notions can describe
the outcome of the CLEVR length-generalization experiment from Figure~\ref{fig:len_gen}.
Our observations here will thus motivate us to propose a new definition of composition,
in the following section. As a running example, we will consider a subset of the CLEVR experiment from Figure~\ref{fig:len_gen}.
Suppose we are trying to compose two distributions $p_1, p_2$ of images each containing a single object in an otherwise empty scene, where the object is in the lower-left corner under $p_1$, and the upper-right corner under $p_2$.
We would like the composed distribution $\hat{p}$ to place objects in at least
the lower-left and upper-right, simultaneously.

\subsection{The Simple Product}
The simple product is perhaps the most familiar type of composition:
Given two distributions $p_1$ and $p_2$ over $\R^n$,
the simple product is defined\footnote{The geometric mean $\sqrt{p_1(x)p_2(x)}$ is also often used; our discussion applies equally to this as well.} as
$\hat{p}(x) \propto p_1(x) p_2(x)$.
The simple product can represent some interesting types of composition, but it 
has a key limitation: 
the composed distribution can never be truly ``out-of-distribution'' w.r.t. $p_1$ or $p_2$,
since $\hat{p}(x) = 0$ whenever $p_1(x) =0$ or $p_2(x) = 0$.
This presents a problem for our CLEVR experiment.
Using the simple product definition,
we must have $\hat{p}(x) = 0$ for any image $x$ with two objects,
since neither $p_1$ nor $p_2$ was supported on images with two objects.
Therefore, the simple product definition cannot represent our
desired composition. %

\subsection{The Bayes Composition}
Another candidate definition for composition, which we will call the ``Bayes composition'', was introduced and studied by \citet{du2023reduce,liu2022compositional}.
The Bayes composition is theoretically justified
when the desired composed distribution is formally a conditional distribution
of the model's training distribution.
However, it is not formally capable of generating truly
out-of-distribution samples, as our example below will illustrate.

Let us attempt to apply the Bayes composition methodology to our CLEVR example.
We interpret our two distributions $p_1, p_2$ as conditional distributions,
conditioned on an object appearing in the lower-left or upper-right, respectively.
Thus we write $p(x|c_1) \equiv p_1(x)$, where $c_1$ is the event that
an object appears in the lower-left of image $x$, and
$c_2$ the event an object appears in the upper-right.
Now, since we want both objects simultaneously, we define the composition as
$\hat{p}(x) := p(x | c_1, c_2)$.
Because the two events $c_1$ and $c_2$ are conditionally independent given $x$
(since they are deterministic functions of $x$), we can compute $\hat{p}$
in terms of the individual conditionals:
\begin{align}
    \hat{p}(x) := p(x | c_1, c_2) %
    & \propto p(x|c_1) p(x | c_2) / p(x) .\label{ln:indep}
\end{align}
Equivalently in terms of scores: $\grad_x \log\hat{p}_t(x) := 
\grad_x \log p(x|c_1) + \grad_x \log p(x|c_2) - \grad_x \log{p}(x)$.
Line~\eqref{ln:indep} thus serves as our definition of the Bayes composition $\hat{p}$,
in terms of the conditional distributions $p(x|c_1)$ and $p(x|c_2)$,
and the unconditional $p(x)$.

The definition of composition above seems natural: we want both objects to appear simultaneously,
so let us simply condition on both these events.
However, there is an obvious error in the conclusion:
$\hat{p}(x)$ must be $0$ whenever $p(x|c_1)$ or $p(x|c_2)$ is zero (by Line~\ref{ln:indep}).
Since neither conditional distribution have support on images with two objects,
the composition $\hat{p}$ cannot contain images of two objects either.

Where did this go wrong?
The issue is: 
$p(x | c_1, c_2)$ is not well-defined in our case.
We intuitively imagine some unconditional distribution $p(x)$
which allows both objects simultaneously,
but no such distribution has been defined,
or encountered by the models during training.
Thus, the definition of $\hat{p}$
in Line~\eqref{ln:indep} does not actually
correspond to our intuitive notion of
``conditioning on both objects at once.''
More generally, this example illustrates how the Bayes composition cannot produce truly out-of-distribution samples, with respect to the distributions being composed.\footnote{
Although \citet{du2023reduce} use the Bayes composition
to achieve a kind of length-generalization, our discussion
shows that the Bayes justification does not 
explain the experimental results.}
Figure~\ref{fig:len_gen_monster}b shows that the Bayes composition does
not always work experimentally either:
for diffusion models trained in a CLEVR setting similar to Figure~\ref{fig:len_gen}, the Bayes composition of $k > 1$ locations typically fails to produce $k$ objects (and sometimes produces zero).
The difficulties discussed lead us to propose a
precise definition of what we actually ``want'' composition to do in this case.

\section{Our Proposal: Projective-Composition}
\label{sec:composition}

\begin{figure}
    \centering
    \includegraphics[width=0.9\linewidth]{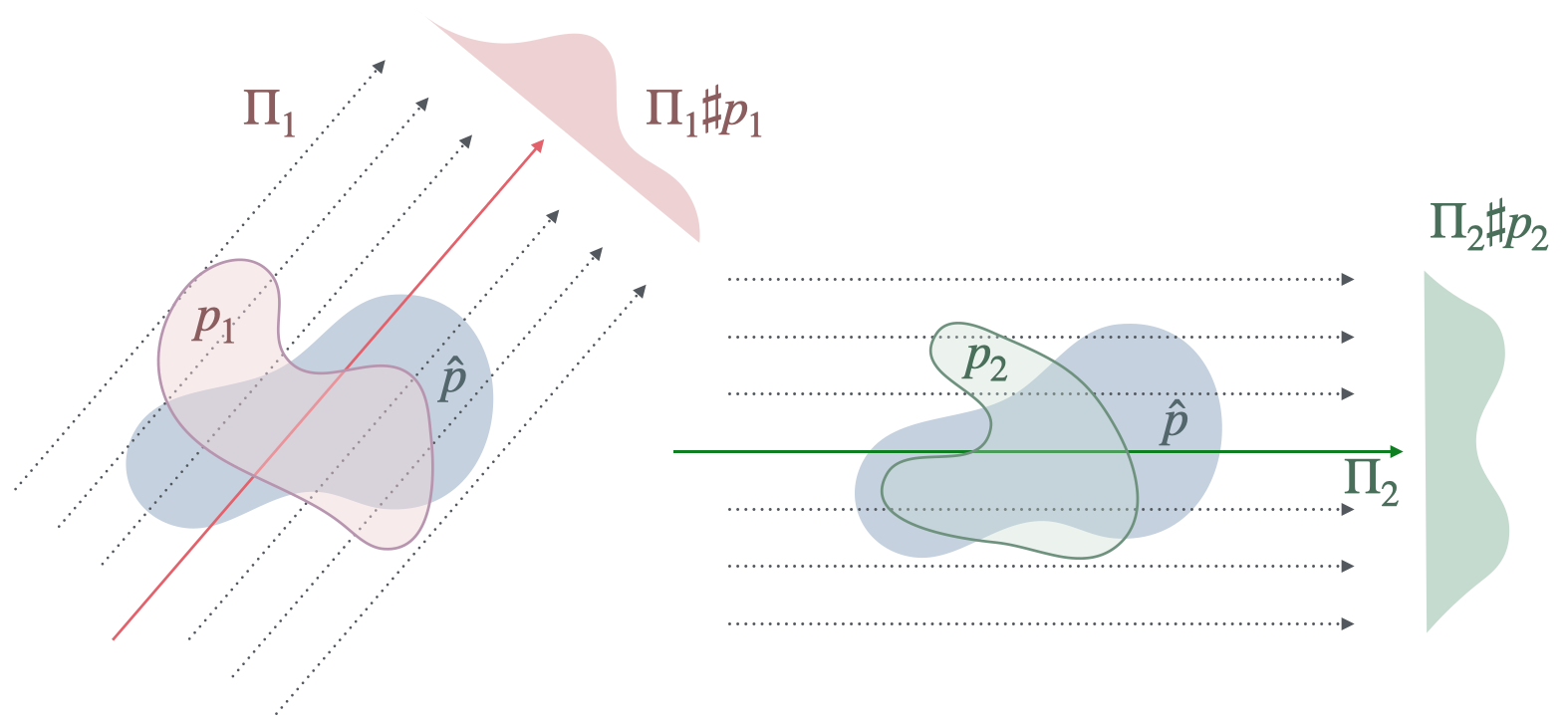}
    \caption{
    Distribution $\hat{p}$ is a projective composition
    of $p_1$ and $p_2$ w.r.t. projection functions $(\Pi_1, \Pi_2)$,
    because $\hat{p}$ has the same marginals as $p_1$ when 
    both are post-processed by $\Pi_1$, and analogously for $p_2$.
    }
    \label{fig:projection-vis}
\end{figure}

We now present our formal definition of what it means to ``correctly compose'' distributions.
Our main insight here is, a realistic definition of composition should not
purely be a function of distributions $\{p_1, p_2, \dots \}$, in the way 
the simple product $\hat{p}(x) = p_1(x) p_2(x)$ is purely a function of $p_1, p_2$.
We must also somehow specify 
\emph{which aspects} of each distribution we care about preserving in the composition.
For example, informally, we may want a composition that mimics the style of $p_1$
and the content of $p_2$.
Our definition below of \emph{projective composition} allows us this flexibility.

Roughly speaking, our definition requires specifying a ``feature extractor''
$\Pi_i: \R^n \to \R_k$ associated with every distribution $p_i$.
These functions can be arbitrary, but we usually imagine them as projections\footnote{
We use the term ``projection'' informally here, to convey intuition;
these functions $\Pi_i$ are not necessarily coordinate projections, although this is an important special case (Section~\ref{sec:comp_coord}).
} in
some feature-space, e.g, $\Pi_1(x)$ may be a transform of $x$ which extracts only its style,
and $\Pi_2(x)$ a transform which extracts only its content.
Then, a projective composition is any distribution $\hat{p}$ which
``looks like'' distribution $p_i$ when both are viewed through $\Pi_i$
(see Figure~\ref{fig:projection-vis}).
Formally:

\begin{definition}[Projective Composition] 
\label{def:proj_comp}
Given a collection of distributions $\{p_i\}$ along with
associated ``projection'' functions $\{\Pi_i: \R^n \to \R^k\}$,
we call a distribution $\hat{p}$ a \emph{projective composition} if\footnote{
The notation $\sharp$ refers to push-forward of a probability measure.
}
\begin{equation}
\label{eqn:proj_comp}
\forall i: \quad
\Pi_i \sharp \hat{p} = \Pi_i \sharp p_i.
\end{equation}
That is, when $\hat{p}$ is projected by each $\Pi_i$,
it yields marginals identical to those of $p_i$.
\end{definition}

There are a few aspects of this definition worth emphasizing,
which are conceptually different from 
many prior notions of composition.
First, our definition above does not \emph{construct} a composed distribution;
it merely specifies what properties the composition must have.
For a given set of $\{(p_i, \Pi_i)\}$, there may be many possible distributions $\hat{p}$
which are projective compositions; or in other cases, a projective composition
may not even exist.
Separately, the definition of projective composition does not posit any sort of ``true'' underlying
distribution, nor does it require that the distributions $p_i$ 
are conditionals of an underlying joint distribution.
In particular, projective compositions can be truly ``out of distribution'' with respect to the $p_i$: $\hat{p}$ can be
supported on samples $x$ where none of the $p_i$ are supported.
\begin{figure}[t]
    \centering
    \includegraphics[width=1.0\linewidth]{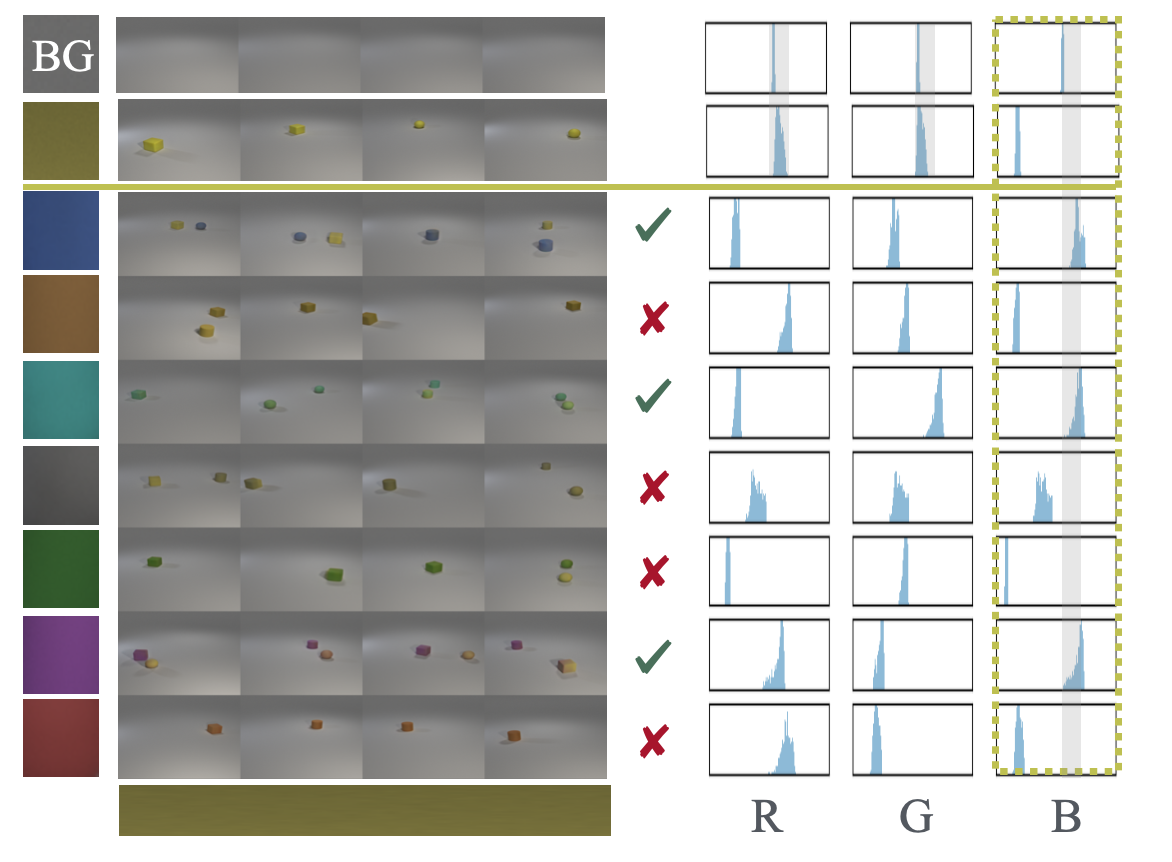}
    \caption{\textbf{Composing yellow objects with objects of other colors.} Yellow objects successfully compose with blue, cyan and magenta objects but not with brown, gray, green, or red objects. Per the histograms (left), in RGB-colorspace yellow has R, G distributed like the background (gray) while B has a distinct distribution peaked closer to zero.
    Taking $M_\text{yellow} \approx \{B\}$, \cref{lem:compose} predicts that standard diffusion can sample from compositions of yellow with any color
    where the B channel is distributed like the background: namely, blue, cyan, magenta per the histograms. (Other colors may theoretically compose per \cref{lem:transform_comp}, but be difficult to sample.) (Additional samples in \cref{fig:clever_color_comp_extra}.)}
    \label{fig:clevr_color_comp}
\end{figure}
\paragraph{Examples.}
We have already discussed the style+content composition of Figure~\ref{fig:style-content}
as an instance of projective composition.
Another even simpler example to keep in mind is 
the following coordinate-projection case.
Suppose we take $\Pi_i: \R^n \to \R$ to be the
projection onto the $i$-th coordinate.
Then, a projective composition of distributions $\{p_i\}$
with these associated functions $\{\Pi_i\}$
means: a distribution where the first coordinate is
marginally distributed identically to the first coordinate of $p_1$,
the second coordinate is marginally distributed as $p_2$, and so on.
(Note, we do not require any independence between coordinates).
This notion of composition would be meaningful if, for example,
we are already working in some disentangled feature space,
where the first coordinate controls the style of the image
the second coordinate controls the texture, and so on.
The CLEVR length-generalization example from Figure~\ref{fig:len_gen}
can also be described as a projective composition in almost an identical way,
by letting $\Pi_i: \R^n \to \R^k$ be a restriction onto the set of
pixels neighboring location $i$. We describe this 
explicitly later in Section~\ref{sec:clevr-details}.

\section{Simple Construction of Projective Compositions}
\label{sec:comp_coord}

It is not clear apriori that projective compositional distributions satisfying Definition \ref{def:proj_comp} ever exist, much less that there is any straightforward way to sample from them.
To explore this, we first restrict attention to perhaps the simplest setting, where the projection functions $\{\Pi_i\}$ are
just coordinate restrictions.
This setting is meant to generalize the intuition we had
in the CLEVR example of Figure~\ref{fig:len_gen},
where different objects were composed in disjoint regions of the image.
We first define the construction of the composed distribution,
and then establish its theoretical properties.

\subsection{Defining the Construction}
Formally, suppose we have a set of distributions
$(p_1, p_2, \ldots, p_k)$ that we wish to compose;
in our running CLEVR example, each $p_i$ is the distribution of images
with a single object at position $i$.
Suppose also we have some reference distribution $p_b$,
which can be arbitrary, but should be thought of as a 
``common background'' to the $p_i$s.
Then, one popular way to construct a composed distribution
is via the \emph{compositional operator} defined below.
(A special case of this construction is used in \citet{du2023reduce}, for example).

\begin{definition}[Composition Operator]
    \label{def:comp_oper}
    Define the \emph{composition operator} $\cC$ acting on an arbitrary set of distributions $(p_b, p_1, p_2, \ldots)$ by
    \begin{align}
    \label{eq:comp_oper}
    \cC[\vec{p}] := \cC[p_b, p_1, p_2, \dots](x) := \frac{1}{Z} p_b(x) \prod_i \frac{p_i(x)}{p_b(x)},
    \end{align}
    where $Z$ is the appropriate normalization constant. We name $\cC[\vec{p}]$ the \emph{composed distribution}, and the score of $\cC[\vec{p}]$ the \emph{compositional score}:
    \begin{align}
    \label{eqn:comp_score}
    &\grad_x \log \cC[\vec{p}](x)  \\
    &= \grad_x \log p_b(x) + \sum_i \left( \grad_x \log p_i(x) - \grad_x \log p_b(x) \right). \notag
    \end{align}
\end{definition}
Notice that if $p_b$ is taken to be the unconditional distribution then this is exactly the Bayes-composition.

\vspace{-0.5em}
\subsection{When does the Composition Operator Work?}
We can always apply the composition operator to any set of distributions,
but when does this actually yield a ``correct'' composition
(according to Definition~\ref{def:proj_comp})?
One special case is when each distribution $p_i$ is
``active'' on a different, non-overlapping set of coordinates.
We formalize this property below
as \emph{Factorized Conditionals} (Definition~\ref{def:factorized}).
The idea is, 
each distribution $p_i$
must have a particular set of ``mask'' coordinates $M_i \subseteq [n]$ which it
samples in a characteristic way,
while independently sampling all other coordinates
from a common background distribution.
If a set of distributions $(p_b, p_1, p_2, \ldots)$ has this
\emph{Factorized Conditional} structure, then 
the composition
operator will produce a projective composition (as we will prove below).

\begin{definition}[Factorized-Conditionals]
\label{def:factorized}

We say a set of distributions $(p_b, p_1, p_2, \dots p_k)$
over $\R^n$
are \emph{Factorized Conditionals} if
there exists a partition of coordinates $[n]$
into disjoint subsets $M_b, M_1, \dots M_k$ such that:
\begin{enumerate}
    \setlength{\itemsep}{1pt}
    \item $(x|_{M_i}, x|_{M_i^c})$ are independent under $p_i$.
    \item $(x|_{M_b}, x|_{M_1}, x|_{M_2}, \dots, x|_{M_k})$
    are mutually independent under $p_b$.
    \item $p_i(x|_{M_i^c}) = p_b(x|_{M_i^c})$.
\end{enumerate}

Equivalently, if we have:
\begin{align}
    p_i(x) &= p_i(x|_{M_i}) p_b(x|_{M_i^c}), \text{ and} \label{eqn:cc-cond}\\
    p_b(x) &= p_b(x|_{M_b}) \prod_{i \in [k]} p_b(x|_{M_i}). \notag
\end{align}
\end{definition}
\vspace{-1em}
Equation~\eqref{eqn:cc-cond} means that each $p_i$
can be sampled by first sampling $x \sim p_b$,
and then overwriting the coordinates of $M_i$
according to some other distribution (which can be specific to distribution $i$).
For instance, the experiment of Figure~\ref{fig:len_gen}
intuitively satisfies this property, since 
each of the conditional distributions could essentially be sampled
by first sampling an empty background image ($p_b$), then ``pasting''
a random object in the appropriate location (corresponding to pixels $M_i$).
If a set of distributions obey this Factorized Conditional structure,
then we can prove that the composition operator $\cC$
yields a correct projective composition,
and reverse-diffusion correctly samples from it.
Below, let $N_t$ denote the noise operator of the
diffusion process\footnote{Our results are agnostic to the specific diffusion noise-schedule and scaling used.} at time $t$.

\begin{theorem}[Correctness of Composition]
\label{lem:compose}
Suppose a set of distributions $(p_b, p_1, p_2, \dots p_k)$
satisfy Definition~\ref{def:factorized},
with corresponding masks $\{M_i\}_i$.
Consider running the reverse-diffusion SDE 
using the following compositional scores at each time $t$:
\begin{align}
s_t(x_t) &:= \grad_x \log \cC[p_b^t, p_1^t, p_2^t, \ldots](x_t),
\end{align}
where $p_i^t := N_t[p_i]$ are the noisy distributions.
Then, the distribution of the generated sample $x_0$ at time $t=0$ is:
\begin{align}
\label{eqn:p_hat}
\hat{p}(x) := p_b(x|_{M_b}) \prod_i p_i(x|_{M_i}).
\end{align}
In particular,
$\hat{p}(x|_{M_i}) = p_i(x|_{M_i})$ for all $i$,
and so
$\hat{p}$ is a projective composition
with respect to projections $\{\Pi_i(x) := x|_{M_i}\}_i$,
per Definition \ref{def:proj_comp}.
\end{theorem}

Unpacking this, Line \ref{eqn:p_hat} says that the final generated distribution
$\hat{p}(x)$ can be sampled by
first sampling
the coordinates $M_b$ according to $p_b$ (marginally),
then independently sampling 
coordinates $M_i$ according to $p_i$ (marginally) for each $i$.
Similarly, by assumption, $p_i(x)$ can be sampled by first sampling the coordinates $M_i$ in some specific way, and then independently sampling the remaining coordinates according to $p_b$. Therefore Theorem \ref{lem:compose} says that $\hat{p}(x)$ samples the coordinates \emph{$M_i$ exactly as they would be sampled by $p_i$}, for each $i$ we wish to compose. 

\begin{proof}(Sketch) \small
Since $\vec{p}$ satisfies Definition \ref{def:factorized}, we have
\begin{align*}
&\cC[\vec{p}](x) := p_b(x) \prod_i \frac{p_i(x)}{p_b(x)} \notag 
= p_b(x) \prod_i \frac{p_b(x|_{M_i^c}) p_i(x|_{M_i})}{p_b(x|_{M_i^c})p_b(x|_{M_i})} \notag \\
&= p_b(x) \prod_i \frac{p_i(x|_{M_i})}{p_b(x|_{M_i})} \notag 
= p_b(x|_{M_b}) \prod_i p_i(x_t|_{M_i}) := \hat{p}(x).
\end{align*}
The sampling guarantee follows from the commutativity of composition with the diffusion noising process, i.e. $\cC[\vec{p^t}]= N_t[\cC[\vec{p}]]$. 
The complete proof is in Appendix \ref{app:compose_pf}.
\end{proof}

\begin{remark}
In fact, Theorem~\ref{lem:compose} still holds under any orthogonal transformation of the variables,
because the diffusion noise process commutes with orthogonal transforms.
We formalize this as Lemma~\ref{lem:orthogonal_sampling}.
\end{remark}

\begin{remark}
Compositionality is often thought of in terms of orthogonality between scores.
Definition \ref{def:factorized} implies orthogonality between the score differences that appear in the composed score \eqref{eqn:comp_score}:
$\grad_x \log p_i^t(x_t) - \grad_x \log p_b^t(x_t),$
but the former condition is strictly stronger
(c.f. Appendix \ref{app:score_orthog}).
\end{remark}

\begin{remark}
Notice that the composition operator $\cC$
can be applied to a set of Factorized Conditional
distributions
without knowing the coordinate partition $\{M_i\}$.
That is, we can compose distributions and compute scores
without knowing apriori exactly ``how'' these distributions are supposed to compose
(i.e. which coordinates $p_i$ is active on).
This is already somewhat remarkable, and we will see a much
stronger version of this property in the next section.
\end{remark}

\textbf{Importance of background.}
Our derivations highlight the crucial role of the background
distribution $p_b$ for the composition operator  
(Definition~\ref{def:comp_oper}).
While prior works have taken $p_b$ to be an unconditional distribution and the $p_i$'s its associated conditionals,
our results suggest this is not always the optimal choice -- in particular,
it may not satisfy a Factorized Conditional structure (Definition~\ref{def:factorized}). Figure~\ref{fig:len_gen_monster} demonstrates this empirically: settings (a) and (b) attempt to compose the same distributions using different backgrounds -- empty (a) or unconditional (b) -- with very different results.

\textbf{A partial relaxation}
Factorized Conditionals (\Cref{def:factorized}) is a stringent assumption in \Cref{lem:compose}, but it turns out that we can relax it somewhat and still show that resulting the composition is approximately projective. The proof is in \Cref{app:compose_pf}.
\begin{lemma}[Relaxed Correctness]
Suppose that $(p_b, p_1, p_2, \ldots)$ satisfy the following conditions, which partially relax \Cref{def:factorized}. We still require $M_1, M_2, \ldots$ to be independent under $p_b$, but we only require ``approximate independence'' between $M_i$ and $M_i^c$ (and $M_b$ and $M_b^c$):
\begin{align*}
    p_b(x) &:= p_b(x|_{M_b} | x|_{M_b^c}) \prod_i p_b(x|_{M_i})\\
    p_i(x) &:= p_i(x|_{M_i} | x|_{M_i^c}) p_b(x|_{M_i^c}), \quad \text{with} \\
    \sup_{x|_{M_i^c}} &\quad \text{KL}[ p_i(x|_{M_i}) || p_i(x|_{M_i} | x|_{M_i^c})] \le \epsilon_i, \quad i = b, 1, 2, \ldots
\end{align*}
Defining the ``ideal'' projective composition by \\
$\cC^\star[\vec{p}] := p_b(x|_{M_b}) \prod_i p_i(x|_{M_i})$,
we have
$$KL(\cC^\star[\vec{p}] || \cC[\vec{p}]) \le \epsilon_b + \sum_i \epsilon_i.$$
\label{lem:fc_relax}    
\end{lemma}
 
\subsection{Approximate Factorized Conditionals in CLEVR.}
\label{sec:clevr-details}

In \cref{fig:len_gen_monster} we explore compositional length-generalization (or lack thereof) in three different setting, two of which (\cref{fig:len_gen_monster}a and \ref{fig:len_gen_monster}c) approximately satisfy \cref{def:factorized}. In this section we explicitly describe how our definition of Factorized Conditionals approximately captures the CLEVR settings of Figures \ref{fig:len_gen_monster}a and \ref{fig:len_gen_monster}c. The setting of \ref{fig:len_gen_monster}b does not satisfy our conditions, as discussed in \cref{sec:problematic-compositions}.

\textbf{Single object distributions with empty background.}
This is the setting of both \cref{fig:len_gen} and \cref{fig:len_gen_monster}a.
The background distribution $p_b$ 
over $n$ pixels is images of an empty scene with no objects.
For each $i \in \{1,\ldots,L\}$ (where $L=4$ in \cref{fig:len_gen} and $L=9$ in \cref{fig:len_gen_monster}a), define the set $M_i \subset [n]$ 
as the set of pixel indices surrounding location $i$.
($M_i$ should be thought of as a ``mask'' that
that masks out objects at location $i$).
Let $M_b := (\cup_i M_i)^c$ be the remaining
pixels in the image.
Then, we claim the distributions $(p_b, p_1, \ldots, p_L)$
form approximately
Factorized Conditionals, with corresponding
coordinate partition $\{M_i\}$.
This is essentially because each distribution $p_i$
matches the background $p_b$ on all pixels except those surrounding
location $i$ (further detail in Appendix~\ref{app:clevr-details}).
Note, however, that the conditions of Definition~\ref{def:factorized}
do not \emph{exactly} hold in the experiment of Figure~\ref{fig:len_gen} -- there is still some dependence between
the masks $M_i$, since objects can cast shadows or even occlude each other.
Empirically, these deviations 
have greater impact
when composing many objects, as seen in \cref{fig:len_gen_monster}a.

\textbf{Bayes composition with cluttered distributions.}
In \cref{fig:len_gen_monster}c we replicate CLEVR experiments in  \citet{du2023reduce, liu2022compositional} where the images contain many objects (1-5) and the conditions label the location of one randomly-chosen object. It turns out the unconditional together with the conditionals can approximately act as Factorized Conditionals in ``cluttered'' settings like this one. The intuition is that if the conditional distributions each contain one specific object plus many independently sampled random objects (``clutter''), then the unconditional distribution \emph{almost} looks like independently sampled random objects, which together with the conditionals \emph{would} satisfy Definition \ref{def:factorized} (further discussion in Appendix \ref{app:clevr-details} and \ref{app:bayes_connect}). This helps to explain the length-generalization observed in \citet{liu2022compositional} and verified in our experiments (\cref{fig:len_gen_monster}c).

\textbf{Compositions of multi-conditioned models}
In \cref{app:clevr_multi}, we examine models trained on multiple objects and conditioned on the locations or colors of all objects, following the approach described in \cite{bradley2025local}. 
We investigate compositions of this model conditioned on individual locations, using a background generated by conditioning the same model on zero locations. We find that this composition achieves superior length generalization compared to the compositions shown in \cref{fig:len_gen_monster}. We hypothesize that this improvement arises from the background model transitioning from generating an empty background to acting as a local unconditional denoiser \cite{kamb2024analytic}.

\section{Projective Composition in Feature Space}
\label{sec:comp_feature}

\begin{figure}
    \centering
    \includegraphics[width=1.0\linewidth]{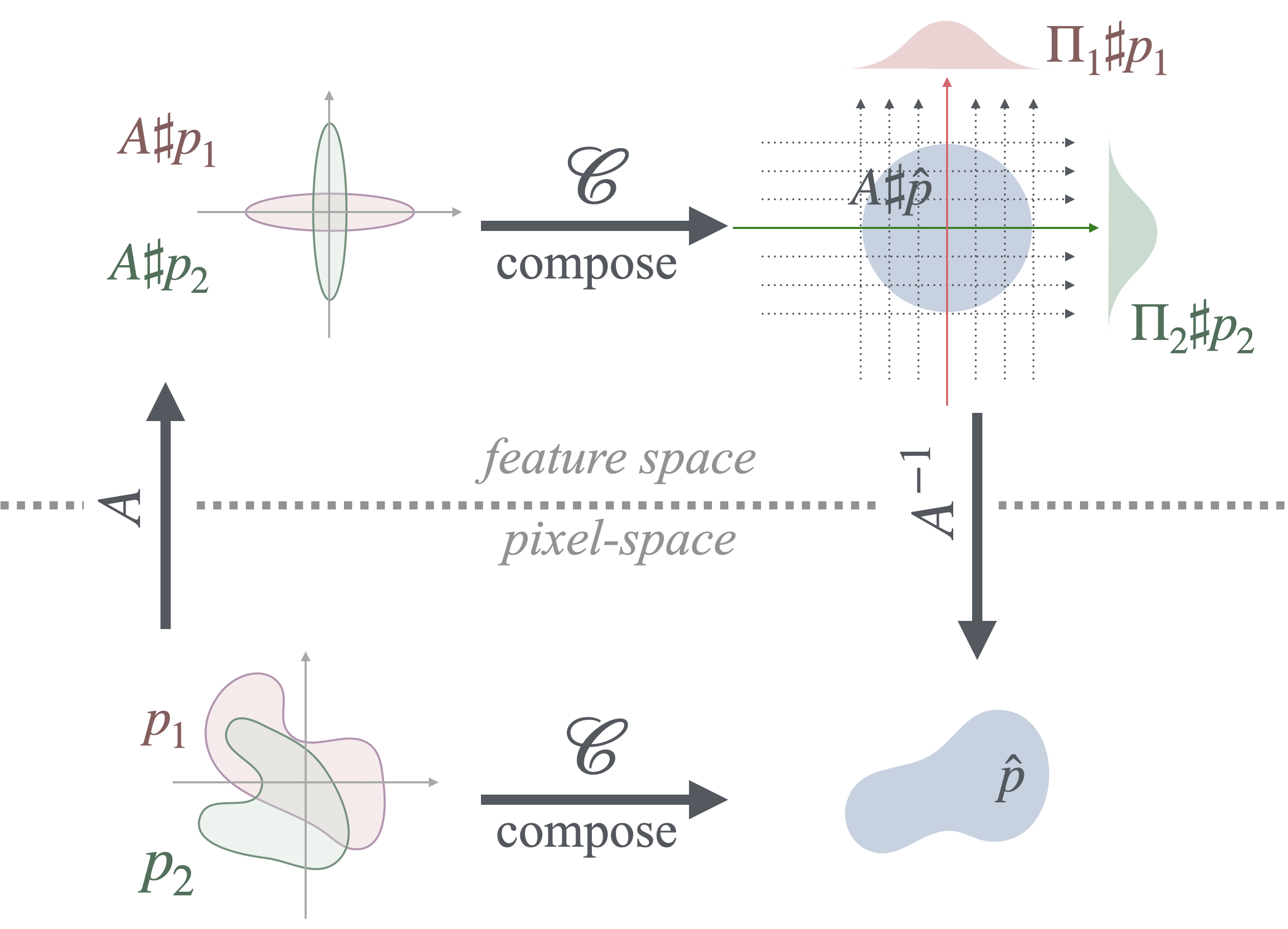}
    \caption{A commutative diagram illustrating Theorem~\ref{lem:transform_comp}.
    Performing composition in pixel space is equivalent 
    to encoding into a feature space ($\cA$),
    composing there,
    and decoding back
    to pixel space ($\cA^{-1}$).
    This holds for \emph{all} feature spaces
    subject to smoothness conditions.
    Thus, if there exists some feature space 
    where distributions $p_1, p_2$ projectively compose
    (e.g. due to orthogonality as illustrated here),
    then we can achieve this same composition
    by simply operating in pixel space,
    without needing to know the feature space.
    }
    \label{fig:feat-space-vis}
\end{figure}

So far we have focused on the setting where the projection functions $\Pi_i$ are simply projections onto coordinate subsets $M_i$ in the native space (e.g. pixel space).
This covers simple examples like Figure~\ref{fig:len_gen} but does not include more realistic situations such as Figure~\ref{fig:style-content},
where the properties to be composed are more abstract.
For example a property like ``oil painting'' does not correspond to projection
onto a specific subset of pixels in an image.
However, we may hope that there exists some conceptual feature space
in which ``oil painting'' does correspond to a particular subset of variables.
In this section, we extend our results to the case where the composition occurs in some conceptual feature space, and each distribution to be composed
corresponds to some particular subset of \emph{features}.

Our main result is a featurized analogue of Theorem~\ref{lem:compose}:
if there exists \emph{any} invertible transform $\cA$
mapping into a feature space
where Definition \ref{def:factorized} holds,
then the composition operator (Definition~\ref{def:comp_oper})
yields a projective composition in this feature space, as shown in Figure~\ref{fig:feat-space-vis}.

\begin{theorem}[Feature-space Composition]
\label{lem:transform_comp}
Given distributions $\vec{p} := (p_b, p_1, p_2, \dots p_k)$,
suppose there exists a $C^1$ diffeomorphism $\cA: \R^n \to \R^n$ (that is, $\cA$ and $\cA^{-1}$ should be differentiable)
such that
$(\cA \sharp p_b, \cA \sharp p_1, \dots \cA \sharp p_k)$
satisfy Definition~\ref{def:factorized},
with corresponding partition $M_i \subseteq [n]$.
Then, the composition $\hat{p} := \cC[\vec{p}]$ satisfies:
\begin{align}
\label{eqn:p_hat_A}
\cA \sharp \hat{p}(z)
\equiv
(\cA \sharp p_b (z))|_{M_b} \prod_{i=1}^k (\cA \sharp p_i(z))|_{M_i}.
\end{align}
Therefore, $\hat{p}$
is a projective composition of $\vec{p}$ w.r.t. projection functions
$\{\Pi_i(x) := \cA(x)|_{M_i}\}$.
\end{theorem}
This theorem is remarkable because it means we can
compose distributions $(p_b, p_1, p_2, \dots)$ in the base space,
and this composition will ``work correctly'' in the feature space
automatically (Equation~\ref{eqn:p_hat_A}),
without us ever needing to compute or even know the feature transform $\cA$
explicitly.

Theorem~\ref{lem:transform_comp} may apriori seem too strong
to be true, since it somehow holds for all feature spaces $\cA$
simultaneously.
The key observation underlying Theorem~\ref{lem:transform_comp} 
is that the composition operator $\cC$ behaves
well under reparameterization.
\begin{lemma}[Reparameterization Equivariance]
\label{lem:reparam}
The composition operator of Definition~\ref{def:comp_oper}
is reparameterization-equivariant. That is,
for all diffeomorphisms $\cA: \R^n \to \R^n$
and all tuples of distributions $\vec{p} = (p_b, p_1, p_2, \dots, p_k)$,
\begin{align}
 \cC[ \cA \sharp \vec{p}] =  \cA \sharp \cC[\vec{p}].
\end{align}
\end{lemma}
\arxiv{\footnote{
For example (separate from our goals in this paper):
Classifier-Free-Guidance can be seen as an instance of the composition operator.
Thus, Lemma~\ref{lem:reparam} implies that performing CFG
in latent space is \emph{equivalent} to CFG in pixel-space,
assuming accurate score-models in both cases.}}
\arxiv{This lemma is potentially of independent interest:
reparametrization-equivariance
is a very strong property which is typically not satisfied by
standard operations between probability distributions---
for example, the ``simple product'' $p_1(x)p_2(x)$ does not satisfy it---
so it is mathematically notable that the composition operator 
has this structure.
Lemma~\ref{lem:reparam} and Theorem~\ref{lem:transform_comp}
are proved in Appendix \ref{app:param-indep}.}

This lemma is potentially of independent interest:
equivariance distinguishes the composition operator
from many other common operators
(e.g. the simple product).
Lemma ~\ref{lem:reparam} and Theorem~\ref{lem:transform_comp}
are proved in Appendix \ref{app:param-indep}.

\section{Sampling from Compositions.}
The feature-space Theorem~\ref{lem:transform_comp} is weaker than Theorem~\ref{lem:compose}
in one important way: it does not provide a sampling algorithm.
That is, Theorem~\ref{lem:transform_comp} guarantees that $\hat{p} := \cC[\vec{p}]$
is a projective composition, but does not guarantee that reverse-diffusion
is a valid sampling method.

There is one special case where diffusion sampling \emph{is} guaranteed to work, namely, for orthogonal transforms (which can seen as a straightforward extension of the coordinate-aligned case of \cref{lem:compose}):
\begin{lemma}[Orthogonal transform enables diffusion sampling]
\label{lem:orthogonal_sampling}
If the assumptions of Lemma \ref{lem:transform_comp} hold for $\cA(x) = Ax$, where $A$ is an orthogonal matrix, then running a reverse diffusion sampler with scores $s_t = \grad_x \log \cC[\vec{p}^t]$ generates the composed distribution $\hat{p} = \cC[\vec{p}]$ satisfying \eqref{eqn:p_hat_A}.
\end{lemma}
The proof is given in \cref{app:orthog_sample_pf}.

However, for general invertible transforms, we have no such sampling guarantees.
Part of this is inherent: in the feature-space setting, the 
diffusion noise operator $N_t$ no longer commutes
with the composition operator $\cC$ in general,
 so scores of the noisy composed 
distribution $N_t[\cC[\vec{p}]]$
cannot be computed from scores
of the noisy base distributions $N_t[\vec{p}]$.
Nevertheless, one may hope to sample from the distribution $\hat{p}$
using other samplers besides diffusion, 
such as annealed Langevin Dynamics
or
Predictor-Corrector methods \citep{song2020score}.
We find that the situation is surprisingly subtle:
composition $\cC$ produces distributions which
are in some cases easy to sample (e.g. with diffusion),
yet in other cases apparently hard to sample.
For example, in the
setting of Figure~\ref{fig:clevr_color_comp}, 
our Theorem~\ref{lem:transform_comp} implies
that all pairs of colors should compose equally well
at time $t=0$, since there exist diffeomorphisms
(indeed, linear transforms) between different colors.
However, as we saw,
the diffusion sampler
fails to sample from compositions 
of non-orthogonal colors--- and 
empirically, even more sophisticated
samplers such as Predictor-Correctors
also fail in this setting.
At first glance, it may seem odd that
composed distributions are so hard to sample,
when their constituent distributions are relatively easy to sample.
One possible reason for this below is that the composition operator has extremely poor Lipchitz constant,
so it is possible for a set of distributions $\vec{p}$ to ``vary smoothly''
(e.g. diffusing over time) while their composition $\cC[\vec{p}]$
changes abruptly.
We formalize this in \cref{lem:lipschitz} (further discussion and proof in Appendix \ref{app:lipschitz}).
\begin{lemma}[Composition Non-Smoothness]
\label{lem:lipschitz}
For any set of distributions $\{p_b, p_1, p_2, \dots, p_k\}$,
and any noise scale $t := \sigma$,
define the noisy distributions 
$p_i^t := N_{t}[p_i]$,
and let $q^t$ denote the composed distribution at time $t$: $q^t := \cC[\vec{p}^t]$. Then, for any choice of $\tau > 0$,
there exist distributions $\{p_b, p_1, \dots p_k\}$ over $\R^n$
such that
\begin{enumerate}
    \setlength{\itemsep}{0pt}
    \item For all $i$, the annealing path of $p_i$ is 
    $\cO(1)$-Lipshitz:
    $\forall t, t': W_2(p_i^{t}, p_i^{t'}) \leq \cO(1) |t - t'|$.
    \item The annealing path of $q$ has Lipshitz constant
    at least $\Omega(\tau^{-1})$:
    $\exists t, t': W_2(q^{t}, q^{t'}) \geq \frac{|t - t'|}{2\tau}.$
\end{enumerate}
\end{lemma}
Intuitively, this means that, even if projective composition is possible at $t=0$, reverse diffusion (or indeed any annealing method) may not be able to correctly sample from it.

\section{Practical implications}
\label{sec:practical}

We have presented a mathematical theory of composition.
Although this theoretical model is a simplification of reality (we do not claim its assumptions hold exactly in practice) we believe the spirit of our results carries over to practical settings,
and can help both understand empirical observations from prior work, and make predictions about the success or failure of new compositions.

\subsection{Connections with prior work}
\textbf{Independence Assumptions and Disentangled Features.}
Our theory relies on a type of independence 
between distributions, related to orthogonality between scores, which we formalize as Factorized Conditionals.
While such conditional structure typically does not exist in pixel-space,
it is plausible that a factorized structure exists in an appropriate \emph{feature space}, permitted by our theory (Section~\ref{sec:comp_feature}).
In particular, a feature space and distribution with perfectly ``disentangled'' features \citep{chen2018isolating, kim2018disentangling, yang2023disdiff, locatello2019challenging}
would satisfy our assumptions.
Conversely, if distributions are not appropriately disentangled,
our theory predicts that linear score combinations will fail to compose correctly.
This effect is well-known; see \cref{fig:dog-horse-hat}
for an example;
similar failure cases are highlighted
in \citet{liu2022compositional} as well
(such as ``A bird'' failing to compose with ``A flower'').
Regarding successful cases, style and content compositions
consistently work well in practice,
and are often taken to be disentangled features
(e.g. \citet{karras2019style,kotovenko2019content,gatys2016image,zhu2017unpaired}).

\textbf{Text conditioning with location information. }
Conditioning on location is a straightforward way to achieve factorized conditionals (provided the objects in different locations are approximately independent) since the required disjointness already holds in pixel-space. 
Many successful text-to-image compositions in \citet{liu2022compositional} use location information in the prompts, either explicitly (e.g. ``A blue bird on a tree'' + ``A red car behind the tree'')
or implicitly
(``A horse'' + ``A yellow flower field''; since horses are typically in the foreground and fields in the background).

\textbf{Unconditional Backgrounds.}
Most prior works on diffusion composition use the Bayes composition, with substantial practical success. 
As discussed in \cref{sec:clevr-details}, Bayes composition may be approximately projective in ``cluttered'' settings,
helping to explain its practical success in text-to-image settings, where images often contain many different possible objects and concepts.

\subsection{A practical heuristic}
We now discuss a simple heuristic to help predict whether Factorized Conditionals (FC) holds for a particular set of concepts. Although it is not sufficient to guarantee Projective Composition, it is easy-to-apply in practice to get an initial ``hint'' about whether composition is likely to work.

We start with a lemma showing that FC concepts satisfy a simple (necessary, but not sufficient) orthogonality condition. The short proof is in \cref{app:score_orthog}.
\begin{lemma}
Let $(p_b, p_1, \dots, p_k)$ be Factorized Conditionals. Let $\mu_i = \E_{p_i}[x]$ for $i=1, \ldots, k$ and $\mu_b = \E_{p_b}[x]$ denote the mean vectors. Then, for any $i \neq j$:
\begin{align}
  (\mu_i - \mu_b)^T (\mu_j - \mu_b) = 0.
\label{eq:mu_orthog}  
\end{align}
Thus, orthogonality between the mean difference vectors $\{\mu_i - \mu_b\}$ is a necessary (but not sufficient) condition for Factorized Conditionals.
\label{lem:heuristic}
\end{lemma}
This lemma makes precise the common intuition that some type of approximate ``concept-space orthogonality'' is required for successful composition in practice, such as LoRA task arithmetic \cite{zhang2023composing, ilharco2022editing}.

We can apply the lemma to help predict whether a new composition may be successful. Importantly, since we often do not expect FC to hold in pixel-space, we need to propose some feature-space in which FC might hold, and verify \cref{lem:heuristic} there. We choose the CLIP \citep{radford2021learning} feature-space as the proposal since it is simple-to-use and is known to provide a reasonably disentangled representation (though other feature spaces could also be used). In order to test \cref{lem:heuristic} in CLIP-space, the procedure is: for each concept distribution $p_i$, collect several representative images, compute their CLIP embeddings, and average them to obtain the mean $\mu_i$. Similarly, estimate the background mean $\mu_b$ using arbitrary images (representing the unconditional distribution). Then check whether \cref{eq:mu_orthog} approximately holds.

CLIP's goal of text-image-alignment suggest an even easier heuristic-for-the-heuristic: simply using the text-embedding for each concept. That is, we could approximate
$\mu_i$ as the CLIP text-embedding of a text description of concept $p_i$. Although CLIP has been shown to suffer from a ``modality gap'' (text and image embeddings are not perfectly aligned) \citep{liang2022gap}, the orthogonality structure we care about -- related to angles between centered-concepts -- may still be fairly well-preserved. In \cref{fig:clip} we show a preliminary experiment with both the image and text heuristics.

\begin{figure}
    \centering
    \includegraphics[width=1.0\linewidth]{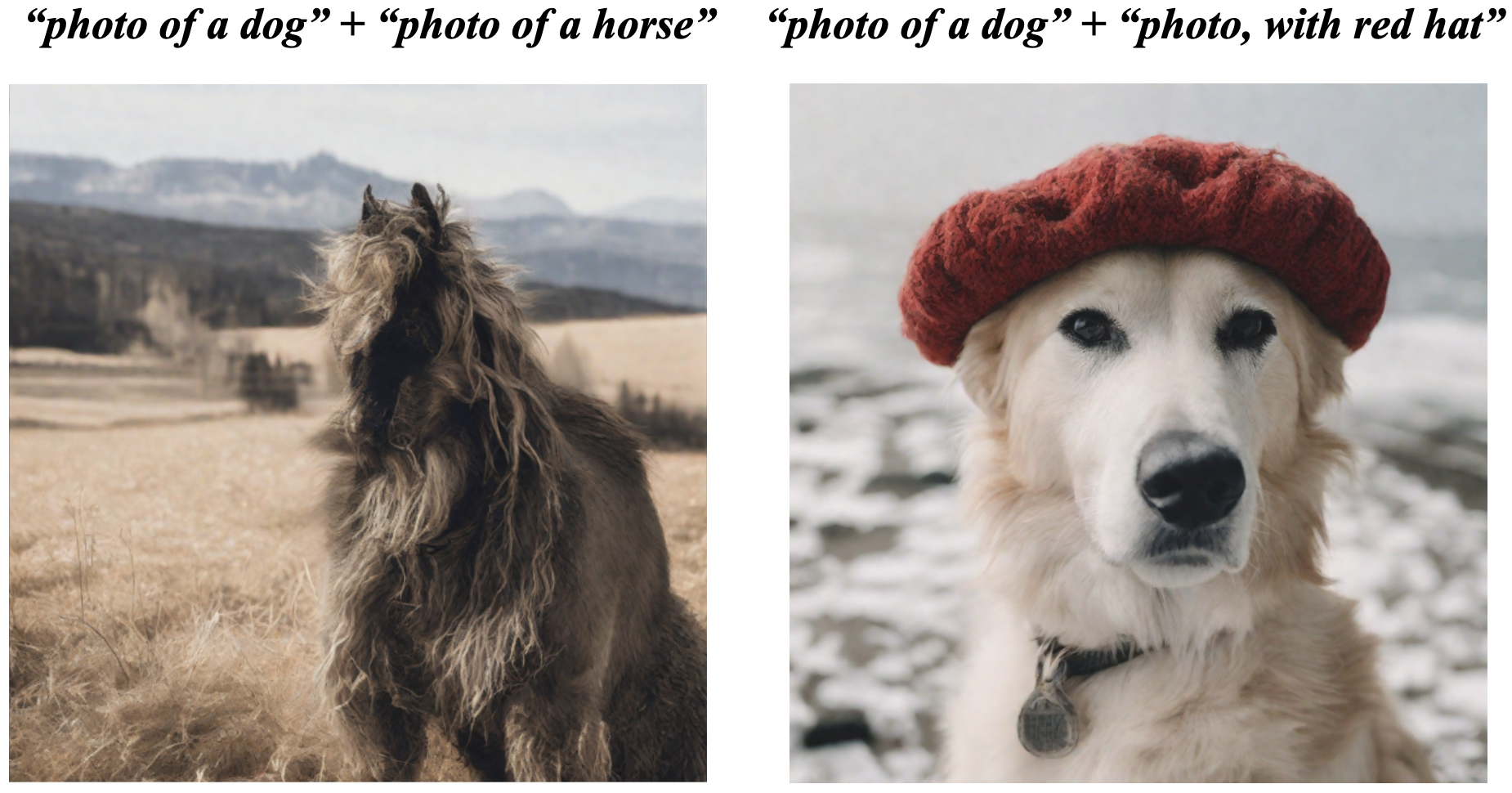}
    \caption{{\bf Composing Entangled Concepts.}
    The left image composes the text-conditions ``photo of a dog''
    with ``photo of a horse'', which both control the subject of the image,
    and produces unexpected results.
    In contrast, the right image composes ``photo of a dog''
    with ``photo, with red hat,'' which intuitively correspond
    to disentangled features.
    Both samples from SDXL using score-composition with
    an unconditional background; details in Appendix~\ref{app:sdxl_detail}.
    }
    \label{fig:dog-horse-hat}
    \vspace{-0.15in}
\end{figure}

\begin{figure}
    \centering
    \includegraphics[width=1.0\linewidth]{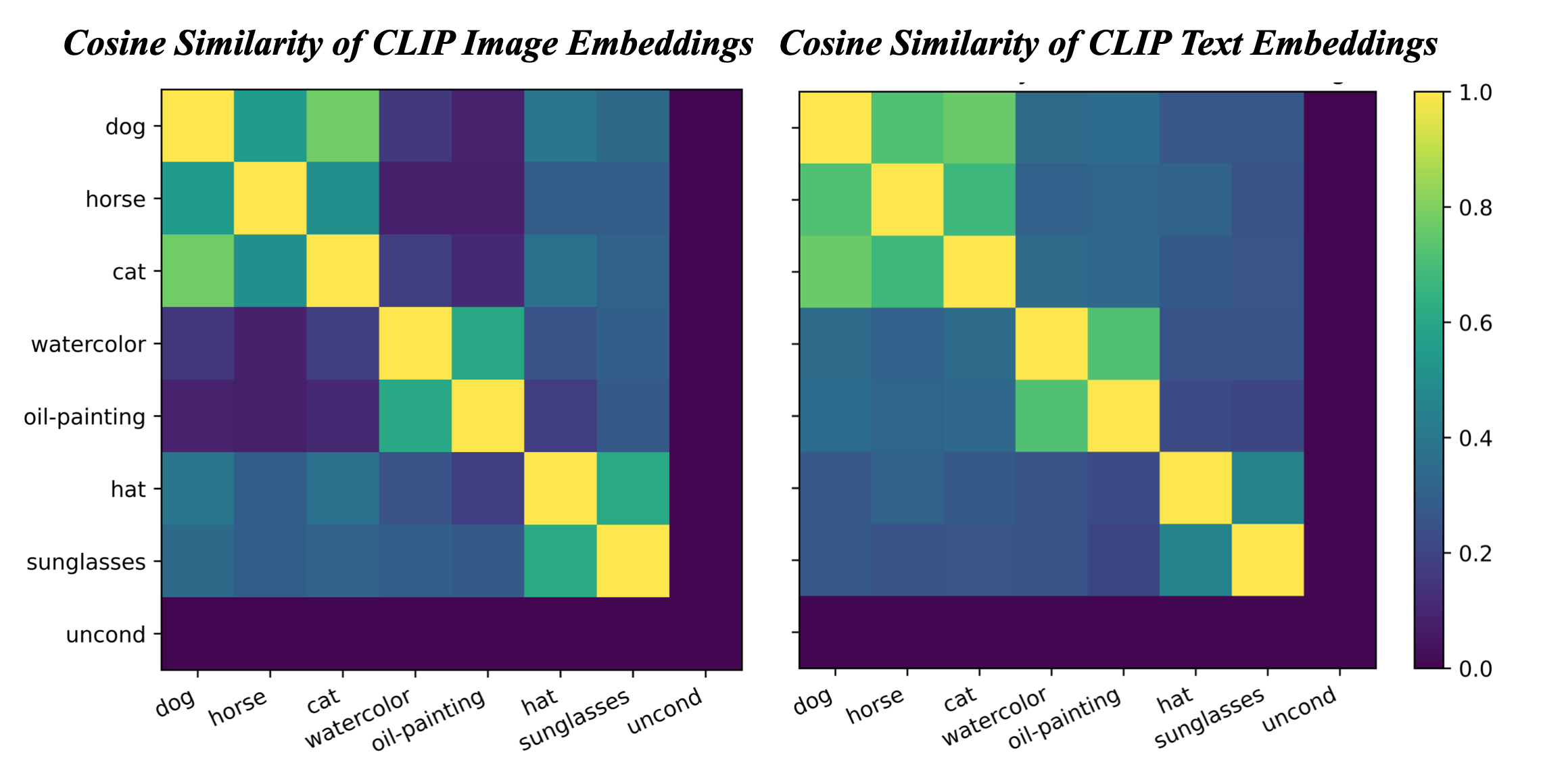}
    \caption{Cosine similarity ($\frac{u^tv}{\|u\|\|v\|}$) between mean difference vectors $\mu_i - \mu_b$ and $\mu_j - \mu_b$, for each pair of concepts $i,j$, with $\mu_i$ estimated as either the average CLIP image embedding over several images representative of concept $i$ (left) or the CLIP text embedding for a single text description of concept $i$ (right). Details in \cref{app:clip}. Note that the groups of concepts \{``dog'', ``horse'', ``cat''\}, \{``watercolor'', ``oil-painting''\}, \{``hat'', ``sunglasses''\} have high intra- and low inter-group similarity. This suggests that concepts from different groups may successfully compose (such as ``dog''+``hat'' or ``dog''+``oil-painting''), while concepts from the same group may not (such as ``dog''+``horse''), consistent with the examples in Figures \ref{fig:style-content} and \ref{fig:dog-horse-hat}.}
    \label{fig:clip}
\end{figure}

\section{Conclusion} 

In this work, we have developed a theory of
one possible mechanism of composition in diffusion models.
We study how composition can be defined, and sufficient conditions
for it to be achieved.
Our theory can help understand a range of diverse
compositional phenomena in both synthetic and practical settings,
and we hope it will inspire further work on foundations of
composition.

\section*{Acknowledgements}
Acknowledgements: We thank Miguel Angel Bautista Martin,  Etai Littwin, Jason Ramapuram, and Luca Zappella for helpful discussions and feedback throughout this work, and Preetum's dog Papaya for his contributions to Figure 1.

\bibliography{refs}

\begin{thebibliography}{58}
\providecommand{\natexlab}[1]{#1}
\providecommand{\url}[1]{\texttt{#1}}
\expandafter\ifx\csname urlstyle\endcsname\relax
  \providecommand{\doi}[1]{doi: #1}\else
  \providecommand{\doi}{doi: \begingroup \urlstyle{rm}\Url}\fi

\bibitem[Ajay et~al.(2024)Ajay, Han, Du, Li, Gupta, Jaakkola, Tenenbaum,
  Kaelbling, Srivastava, and Agrawal]{ajay2024compositional}
Ajay, A., Han, S., Du, Y., Li, S., Gupta, A., Jaakkola, T., Tenenbaum, J.,
  Kaelbling, L., Srivastava, A., and Agrawal, P.
\newblock Compositional foundation models for hierarchical planning.
\newblock \emph{Advances in Neural Information Processing Systems}, 36, 2024.

\bibitem[Bradley(2025)]{bradley2025local}
Bradley, A.
\newblock Local mechanisms of compositional generalization in conditional
  diffusion.
\newblock \emph{arXiv preprint arXiv:2509.16447}, 2025.
\newblock URL \url{https://arxiv.org/abs/2509.16447}.

\bibitem[Chen et~al.(2018)Chen, Li, Grosse, and Duvenaud]{chen2018isolating}
Chen, R.~T., Li, X., Grosse, R.~B., and Duvenaud, D.~K.
\newblock Isolating sources of disentanglement in variational autoencoders.
\newblock \emph{Advances in neural information processing systems}, 31, 2018.

\bibitem[Delon \& Desolneux(2020)Delon and Desolneux]{delon2020wasserstein}
Delon, J. and Desolneux, A.
\newblock A wasserstein-type distance in the space of gaussian mixture models.
\newblock \emph{SIAM Journal on Imaging Sciences}, 13\penalty0 (2):\penalty0
  936--970, 2020.

\bibitem[Du \& Kaelbling(2024)Du and Kaelbling]{du2024compositional}
Du, Y. and Kaelbling, L.~P.
\newblock Position: Compositional generative modeling: A single model is not
  all you need.
\newblock In \emph{Forty-first International Conference on Machine Learning},
  2024.

\bibitem[Du et~al.(2020)Du, Li, and Mordatch]{du2020visualenergy}
Du, Y., Li, S., and Mordatch, I.
\newblock Compositional visual generation and inference with energy based
  models.
\newblock \emph{arXiv preprint arXiv:2004.06030}, 2020.

\bibitem[Du et~al.(2023)Du, Durkan, Strudel, Tenenbaum, Dieleman, Fergus,
  Sohl-Dickstein, Doucet, and Grathwohl]{du2023reduce}
Du, Y., Durkan, C., Strudel, R., Tenenbaum, J.~B., Dieleman, S., Fergus, R.,
  Sohl-Dickstein, J., Doucet, A., and Grathwohl, W.~S.
\newblock Reduce, reuse, recycle: Compositional generation with energy-based
  diffusion models and mcmc.
\newblock In \emph{International conference on machine learning}, pp.\
  8489--8510. PMLR, 2023.

\bibitem[Gatys et~al.(2016)Gatys, Ecker, and Bethge]{gatys2016image}
Gatys, L.~A., Ecker, A.~S., and Bethge, M.
\newblock Image style transfer using convolutional neural networks.
\newblock In \emph{Proceedings of the IEEE conference on computer vision and
  pattern recognition}, pp.\  2414--2423, 2016.

\bibitem[Gregor et~al.(2015)Gregor, Danihelka, Graves, Rezende, and
  Wierstra]{gregor2015draw}
Gregor, K., Danihelka, I., Graves, A., Rezende, D., and Wierstra, D.
\newblock Draw: A recurrent neural network for image generation.
\newblock In \emph{International conference on machine learning}, pp.\
  1462--1471. PMLR, 2015.

\bibitem[Hinton(2002)]{hinton2002training}
Hinton, G.~E.
\newblock Training products of experts by minimizing contrastive divergence.
\newblock \emph{Neural computation}, 14\penalty0 (8):\penalty0 1771--1800,
  2002.

\bibitem[Ho \& Salimans(2022)Ho and Salimans]{ho2022classifier}
Ho, J. and Salimans, T.
\newblock Classifier-free diffusion guidance.
\newblock \emph{arXiv preprint arXiv:2207.12598}, 2022.

\bibitem[Ho et~al.(2020)Ho, Jain, and Abbeel]{ho2020denoising}
Ho, J., Jain, A., and Abbeel, P.
\newblock Denoising diffusion probabilistic models.
\newblock \emph{Advances in neural information processing systems},
  33:\penalty0 6840--6851, 2020.

\bibitem[Hu et~al.(2022)Hu, Shen, Wallis, Allen-Zhu, Li, Wang, Wang, Chen,
  et~al.]{hu2022lora}
Hu, E.~J., Shen, Y., Wallis, P., Allen-Zhu, Z., Li, Y., Wang, S., Wang, L.,
  Chen, W., et~al.
\newblock Lora: Low-rank adaptation of large language models.
\newblock \emph{ICLR}, 1\penalty0 (2):\penalty0 3, 2022.

\bibitem[Ilharco et~al.(2022)Ilharco, Ribeiro, Wortsman, Gururangan, Schmidt,
  Hajishirzi, and Farhadi]{ilharco2022editing}
Ilharco, G., Ribeiro, M.~T., Wortsman, M., Gururangan, S., Schmidt, L.,
  Hajishirzi, H., and Farhadi, A.
\newblock Editing models with task arithmetic.
\newblock \emph{arXiv preprint arXiv:2212.04089}, 2022.

\bibitem[Jacobs et~al.(1991)Jacobs, Jordan, Nowlan, and
  Hinton]{jacobs1991adaptive}
Jacobs, R.~A., Jordan, M.~I., Nowlan, S.~J., and Hinton, G.~E.
\newblock Adaptive mixtures of local experts.
\newblock \emph{Neural computation}, 3\penalty0 (1):\penalty0 79--87, 1991.

\bibitem[Janner et~al.(2022)Janner, Du, Tenenbaum, and
  Levine]{janner2022planning}
Janner, M., Du, Y., Tenenbaum, J.~B., and Levine, S.
\newblock Planning with diffusion for flexible behavior synthesis.
\newblock \emph{arXiv preprint arXiv:2205.09991}, 2022.

\bibitem[Johnson et~al.(2017)Johnson, Hariharan, Van Der~Maaten, Fei-Fei,
  Lawrence~Zitnick, and Girshick]{johnson2017clevr}
Johnson, J., Hariharan, B., Van Der~Maaten, L., Fei-Fei, L., Lawrence~Zitnick,
  C., and Girshick, R.
\newblock Clevr: A diagnostic dataset for compositional language and elementary
  visual reasoning.
\newblock In \emph{Proceedings of the IEEE conference on computer vision and
  pattern recognition}, pp.\  2901--2910, 2017.

\bibitem[Kamb \& Ganguli(2024)Kamb and Ganguli]{kamb2024analytic}
Kamb, M. and Ganguli, S.
\newblock An analytic theory of creativity in convolutional diffusion models.
\newblock \emph{arXiv preprint arXiv:2412.20292}, 2024.

\bibitem[Karras et~al.(2019)Karras, Laine, and Aila]{karras2019style}
Karras, T., Laine, S., and Aila, T.
\newblock A style-based generator architecture for generative adversarial
  networks.
\newblock In \emph{Proceedings of the IEEE/CVF conference on computer vision
  and pattern recognition}, pp.\  4401--4410, 2019.

\bibitem[Karras et~al.(2024)Karras, Aittala, Lehtinen, Hellsten, Aila, and
  Laine]{karras2024analyzing}
Karras, T., Aittala, M., Lehtinen, J., Hellsten, J., Aila, T., and Laine, S.
\newblock Analyzing and improving the training dynamics of diffusion models.
\newblock In \emph{Proceedings of the IEEE/CVF Conference on Computer Vision
  and Pattern Recognition}, pp.\  24174--24184, 2024.

\bibitem[Kim \& Mnih(2018)Kim and Mnih]{kim2018disentangling}
Kim, H. and Mnih, A.
\newblock Disentangling by factorising.
\newblock In \emph{International conference on machine learning}, pp.\
  2649--2658. PMLR, 2018.

\bibitem[Kotovenko et~al.(2019)Kotovenko, Sanakoyeu, Lang, and
  Ommer]{kotovenko2019content}
Kotovenko, D., Sanakoyeu, A., Lang, S., and Ommer, B.
\newblock Content and style disentanglement for artistic style transfer.
\newblock In \emph{Proceedings of the IEEE/CVF international conference on
  computer vision}, pp.\  4422--4431, 2019.

\bibitem[Liang et~al.(2022)Liang, Zhang, Kwon, Yeung, and Zou]{liang2022gap}
Liang, W., Zhang, Y., Kwon, Y., Yeung, S., and Zou, J.
\newblock Mind the gap: Understanding the modality gap in multi-modal
  contrastive representation learning.
\newblock In \emph{NeurIPS}, 2022.
\newblock URL \url{https://openreview.net/forum?id=S7Evzt9uit3}.

\bibitem[Liu et~al.(2021)Liu, Li, Du, Tenenbaum, and Torralba]{liu2021learning}
Liu, N., Li, S., Du, Y., Tenenbaum, J., and Torralba, A.
\newblock Learning to compose visual relations.
\newblock \emph{Advances in Neural Information Processing Systems},
  34:\penalty0 23166--23178, 2021.

\bibitem[Liu et~al.(2022)Liu, Li, Du, Torralba, and
  Tenenbaum]{liu2022compositional}
Liu, N., Li, S., Du, Y., Torralba, A., and Tenenbaum, J.~B.
\newblock Compositional visual generation with composable diffusion models.
\newblock In \emph{European Conference on Computer Vision}, pp.\  423--439.
  Springer, 2022.

\bibitem[Liu et~al.(2024)Liu, Zhang, Jaakkola, and Chang]{liu2024correcting}
Liu, Y., Zhang, Y., Jaakkola, T., and Chang, S.
\newblock Correcting diffusion generation through resampling.
\newblock In \emph{Proceedings of the IEEE/CVF Conference on Computer Vision
  and Pattern Recognition}, pp.\  8713--8723, 2024.

\bibitem[Locatello et~al.(2019)Locatello, Bauer, Lucic, Raetsch, Gelly,
  Sch{\"o}lkopf, and Bachem]{locatello2019challenging}
Locatello, F., Bauer, S., Lucic, M., Raetsch, G., Gelly, S., Sch{\"o}lkopf, B.,
  and Bachem, O.
\newblock Challenging common assumptions in the unsupervised learning of
  disentangled representations.
\newblock In \emph{international conference on machine learning}, pp.\
  4114--4124. PMLR, 2019.

\bibitem[McAllister et~al.(2025)McAllister, Tancik, Song, and
  Kanazawa]{mcallister2025decentralized}
McAllister, D., Tancik, M., Song, J., and Kanazawa, A.
\newblock Decentralized diffusion models.
\newblock \emph{arXiv preprint arXiv:2501.05450}, 2025.

\bibitem[Nichol et~al.(2021)Nichol, Dhariwal, Ramesh, Shyam, Mishkin, McGrew,
  Sutskever, and Chen]{nichol2021glide}
Nichol, A., Dhariwal, P., Ramesh, A., Shyam, P., Mishkin, P., McGrew, B.,
  Sutskever, I., and Chen, M.
\newblock Glide: Towards photorealistic image generation and editing with
  text-guided diffusion models.
\newblock \emph{arXiv preprint arXiv:2112.10741}, 2021.

\bibitem[Nie et~al.(2021)Nie, Vahdat, and Anandkumar]{nie2021controllable}
Nie, W., Vahdat, A., and Anandkumar, A.
\newblock Controllable and compositional generation with latent-space
  energy-based models.
\newblock \emph{Advances in Neural Information Processing Systems},
  34:\penalty0 13497--13510, 2021.

\bibitem[Niedoba et~al.(2024)Niedoba, Zwartsenberg, Murphy, and
  Wood]{niedoba2024towards}
Niedoba, M., Zwartsenberg, B., Murphy, K., and Wood, F.
\newblock Towards a mechanistic explanation of diffusion model generalization.
\newblock \emph{arXiv preprint arXiv:2411.19339}, 2024.

\bibitem[Okawa et~al.(2024)Okawa, Lubana, Dick, and
  Tanaka]{okawa2024compositional}
Okawa, M., Lubana, E.~S., Dick, R., and Tanaka, H.
\newblock Compositional abilities emerge multiplicatively: Exploring diffusion
  models on a synthetic task.
\newblock \emph{Advances in Neural Information Processing Systems}, 36, 2024.

\bibitem[Parisi(1981)]{parisi1981correlation}
Parisi, G.
\newblock Correlation functions and computer simulations.
\newblock \emph{Nuclear Physics B}, 180\penalty0 (3):\penalty0 378--384, 1981.

\bibitem[Park et~al.(2024)Park, Okawa, Lee, Lubana, and
  Tanaka]{park2024emergence}
Park, C.~F., Okawa, M., Lee, A., Lubana, E.~S., and Tanaka, H.
\newblock Emergence of hidden capabilities: Exploring learning dynamics in
  concept space.
\newblock \emph{Advances in Neural Information Processing Systems},
  37:\penalty0 84698--84729, 2024.

\bibitem[Podell et~al.(2023)Podell, English, Lacey, Blattmann, Dockhorn,
  M{\"u}ller, Penna, and Rombach]{podell2023sdxl}
Podell, D., English, Z., Lacey, K., Blattmann, A., Dockhorn, T., M{\"u}ller,
  J., Penna, J., and Rombach, R.
\newblock Sdxl: Improving latent diffusion models for high-resolution image
  synthesis.
\newblock \emph{arXiv preprint arXiv:2307.01952}, 2023.

\bibitem[Radford et~al.(2021)Radford, Kim, Hallacy, Ramesh, Goh, Agarwal,
  Sastry, Askell, Mishkin, Clark, et~al.]{radford2021learning}
Radford, A., Kim, J.~W., Hallacy, C., Ramesh, A., Goh, G., Agarwal, S., Sastry,
  G., Askell, A., Mishkin, P., Clark, J., et~al.
\newblock Learning transferable visual models from natural language
  supervision.
\newblock In \emph{International conference on machine learning}, pp.\
  8748--8763. PmLR, 2021.

\bibitem[Robert et~al.(1999)Robert, Casella, and Casella]{robert1999monte}
Robert, C.~P., Casella, G., and Casella, G.
\newblock \emph{Monte Carlo statistical methods}, volume~2.
\newblock Springer, 1999.

\bibitem[Roberts \& Tweedie(1996)Roberts and Tweedie]{roberts1996exponential}
Roberts, G.~O. and Tweedie, R.~L.
\newblock Exponential convergence of langevin distributions and their discrete
  approximations, 1996.

\bibitem[Rossky et~al.(1978)Rossky, Doll, and Friedman]{rossky1978brownian}
Rossky, P.~J., Doll, J.~D., and Friedman, H.~L.
\newblock Brownian dynamics as smart monte carlo simulation.
\newblock \emph{The Journal of Chemical Physics}, 69\penalty0 (10):\penalty0
  4628--4633, 1978.

\bibitem[Skreta et~al.(2024)Skreta, Atanackovic, Bose, Tong, and
  Neklyudov]{skreta2024superposition}
Skreta, M., Atanackovic, L., Bose, A.~J., Tong, A., and Neklyudov, K.
\newblock The superposition of diffusion models using the it\^{o} density
  estimator.
\newblock \emph{arXiv preprint arXiv:2412.17762}, 2024.

\bibitem[Sohl-Dickstein et~al.(2015)Sohl-Dickstein, Weiss, Maheswaranathan, and
  Ganguli]{sohl2015deep}
Sohl-Dickstein, J., Weiss, E., Maheswaranathan, N., and Ganguli, S.
\newblock Deep unsupervised learning using nonequilibrium thermodynamics.
\newblock In \emph{International conference on machine learning}, pp.\
  2256--2265. pmlr, 2015.

\bibitem[Song et~al.(2021)Song, Meng, and Ermon]{song2021denoising}
Song, J., Meng, C., and Ermon, S.
\newblock Denoising diffusion implicit models.
\newblock In \emph{International Conference on Learning Representations}, 2021.
\newblock URL \url{https://openreview.net/forum?id=St1giarCHLP}.

\bibitem[Song \& Ermon(2019)Song and Ermon]{song2019generative}
Song, Y. and Ermon, S.
\newblock Generative modeling by estimating gradients of the data distribution.
\newblock \emph{Advances in neural information processing systems}, 32, 2019.

\bibitem[Song et~al.(2020)Song, Sohl-Dickstein, Kingma, Kumar, Ermon, and
  Poole]{song2020score}
Song, Y., Sohl-Dickstein, J., Kingma, D.~P., Kumar, A., Ermon, S., and Poole,
  B.
\newblock Score-based generative modeling through stochastic differential
  equations.
\newblock \emph{arXiv preprint arXiv:2011.13456}, 2020.
\newblock URL \url{https://arxiv.org/pdf/2011.13456.pdf}.

\bibitem[Stracke et~al.(2024)Stracke, Baumann, Susskind, Bautista, and
  Ommer]{stracke2024ctrloralter}
Stracke, N., Baumann, S.~A., Susskind, J., Bautista, M.~A., and Ommer, B.
\newblock Ctrloralter: Conditional loradapter for efficient 0-shot control and
  altering of t2i models.
\newblock In \emph{European Conference on Computer Vision}, pp.\  87--103.
  Springer, 2024.

\bibitem[Su et~al.(2024)Su, Liu, Wang, Tenenbaum, and Du]{su2024decomposition}
Su, J., Liu, N., Wang, Y., Tenenbaum, J.~B., and Du, Y.
\newblock Compositional image decomposition with diffusion models.
\newblock \emph{arXiv preprint arXiv:2406.19298}, 2024.

\bibitem[Urain et~al.(2023)Urain, Li, Liu, D’Eramo, and
  Peters]{urain2023composable}
Urain, J., Li, A., Liu, P., D’Eramo, C., and Peters, J.
\newblock Composable energy policies for reactive motion generation and
  reinforcement learning.
\newblock \emph{The International Journal of Robotics Research}, 42\penalty0
  (10):\penalty0 827--858, 2023.

\bibitem[Wang et~al.(2023)Wang, Liu, and Dauwels]{wang2023slot}
Wang, Y., Liu, L., and Dauwels, J.
\newblock Slot-vae: Object-centric scene generation with slot attention.
\newblock In \emph{International Conference on Machine Learning}, pp.\
  36020--36035. PMLR, 2023.

\bibitem[Wang et~al.(2024)Wang, Gui, Negrea, and Veitch]{wang2024concept}
Wang, Z., Gui, L., Negrea, J., and Veitch, V.
\newblock Concept algebra for (score-based) text-controlled generative models.
\newblock \emph{Advances in Neural Information Processing Systems}, 36, 2024.

\bibitem[Wiedemer et~al.(2024)Wiedemer, Mayilvahanan, Bethge, and
  Brendel]{wiedemer2024compositional}
Wiedemer, T., Mayilvahanan, P., Bethge, M., and Brendel, W.
\newblock Compositional generalization from first principles.
\newblock \emph{Advances in Neural Information Processing Systems}, 36, 2024.

\bibitem[Wu et~al.(2024)Wu, Maruyama, Wei, Zhang, Du, Iaccarino, and
  Leskovec]{wu2024compositional}
Wu, T., Maruyama, T., Wei, L., Zhang, T., Du, Y., Iaccarino, G., and Leskovec,
  J.
\newblock Compositional generative inverse design.
\newblock \emph{arXiv preprint arXiv:2401.13171}, 2024.

\bibitem[Yang et~al.(2023{\natexlab{a}})Yang, Du, Dai, Schuurmans, Tenenbaum,
  and Abbeel]{yang2023probabilistic}
Yang, M., Du, Y., Dai, B., Schuurmans, D., Tenenbaum, J.~B., and Abbeel, P.
\newblock Probabilistic adaptation of text-to-video models.
\newblock \emph{arXiv preprint arXiv:2306.01872}, 2023{\natexlab{a}}.

\bibitem[Yang et~al.(2023{\natexlab{b}})Yang, Wang, Lv, and
  Zheng]{yang2023disdiff}
Yang, T., Wang, Y., Lv, Y., and Zheng, N.
\newblock Disdiff: Unsupervised disentanglement of diffusion probabilistic
  models.
\newblock \emph{arXiv preprint arXiv:2301.13721}, 2023{\natexlab{b}}.

\bibitem[Yang et~al.(2023{\natexlab{c}})Yang, Mao, Du, Wu, Tenenbaum,
  Lozano-P{\'e}rez, and Kaelbling]{yang2023compositional}
Yang, Z., Mao, J., Du, Y., Wu, J., Tenenbaum, J.~B., Lozano-P{\'e}rez, T., and
  Kaelbling, L.~P.
\newblock Compositional diffusion-based continuous constraint solvers.
\newblock \emph{arXiv preprint arXiv:2309.00966}, 2023{\natexlab{c}}.

\bibitem[Zhang et~al.(2023{\natexlab{a}})Zhang, Liu, He,
  et~al.]{zhang2023composing}
Zhang, J., Liu, J., He, J., et~al.
\newblock Composing parameter-efficient modules with arithmetic operation.
\newblock \emph{Advances in Neural Information Processing Systems},
  36:\penalty0 12589--12610, 2023{\natexlab{a}}.

\bibitem[Zhang et~al.(2023{\natexlab{b}})Zhang, Rao, and
  Agrawala]{zhang2023adding}
Zhang, L., Rao, A., and Agrawala, M.
\newblock Adding conditional control to text-to-image diffusion models.
\newblock In \emph{Proceedings of the IEEE/CVF International Conference on
  Computer Vision}, pp.\  3836--3847, 2023{\natexlab{b}}.

\bibitem[Zhang et~al.(2025)Zhang, Rao, and Agrawala]{zhang2025scaling}
Zhang, L., Rao, A., and Agrawala, M.
\newblock Scaling in-the-wild training for diffusion-based illumination
  harmonization and editing by imposing consistent light transport.
\newblock In \emph{The Thirteenth International Conference on Learning
  Representations}, 2025.

\bibitem[Zhu et~al.(2017)Zhu, Park, Isola, and Efros]{zhu2017unpaired}
Zhu, J.-Y., Park, T., Isola, P., and Efros, A.~A.
\newblock Unpaired image-to-image translation using cycle-consistent
  adversarial networks.
\newblock In \emph{Proceedings of the IEEE international conference on computer
  vision}, pp.\  2223--2232, 2017.

\end{thebibliography}
\bibliographystyle{icml2025}

\newpage
\appendix
\onecolumn
\section{Additional Related Works}
\label{app:related}

\textbf{Structured compositional generative models.} Structured generative models leverage architectural inductive biases in an encoder-decoder framework, such as recurrent attention mechanisms \cite{gregor2015draw} or slot-attention \cite{wang2023slot}. These models decompose scenes into background and parts-based representations in an unsupervised manner guided by modeling priors. While these approaches can flexibly generate scenes with single or multiple objects, they are not explicitly controllable, and require specific model pre-training on datasets containing compositions of interest.

\textbf{Controllable generation.} Composition at inference-time is one potential mechanism for exerting control over the generation process. Another way to modify compositions of style and/or content attributes is through spatial conditioning a pre-trained diffusion model on a structural attribute (e.g., pose or depth) as in  \citet{zhang2023adding}, or on multiple attributes of style and/or content as in \citet{stracke2024ctrloralter}. Another option is control through resampling, as in \citet{liu2024correcting}. These methods are complementary to single or multiple model conditioning mechanisms based on score composition that we study in the current work.

\textbf{Single model conditioning.} We distinguish the kind of composition we study in this paper from approaches that rely on a single model but use OOD conditioners to achieve novel combinations of concepts never seen together during training; for example, passing OOD text prompts to text-to-image models \citep{nichol2021glide, podell2023sdxl}, or works like \citet{okawa2024compositional, park2024emergence} where a single model conditions simultaneously on multiple attributes like shape and color, with some combinations held out during training.
In contrast, the compositions we study recombine the outputs of multiple separate models at inference time.
Though less powerful, this can still be surprisingly effective, and is more amenable to theoretical study since it disentangles the potential role of conditional embeddings.

\textbf{Multiple model composition.} Among compositions involving multiple separate models, many different variants have been explored with different goals and applications.
Some definitions of composition are inspired by logical operators like AND and OR, usually taken to mean that the composed distribution should have high probability under all of the conditional distributions to be composed, or at least one of them, respectively.
Given two conditional probabilities $p_0(x), p_1(x)$, AND is typically implemented as the product $p_0(x)p_1(x)$ and OR as sum $p_0(x) + p_1(x)$
(though these only loosely correspond to the logical operators and other implementations are also possible).
Some composition methods are based on diffusion models and use the learned scores (mainly for product compositions), others use energy-based models (which allows for OR-inspired sum compositions, as well as more sophisticated samplers, in particular sampling at $t=0$ \citep{du2020visualenergy, du2023reduce, liu2021learning}, and still others work directly with the densities \cite{skreta2024superposition} (enabling an even greater variety of compositions, including a different style of AND, taken to mean $p_0(x) = p_1(x)$). \citet{mcallister2025decentralized} explore another type of OR composition. \cite{wiedemer2024compositional} take a different approach of taking the final rendered images generated by separate diffusion models and ``adding them up'' in pixel-space, as part of a study on generalization of data-generating processes. Task-arithmetic \cite{zhang2023composing, ilharco2022editing}, often using LoRAs \cite{hu2022lora}, is a kind of composition in weight-space that has had significant practical impact.

\textbf{Product compositions.} In this work, we focus specifically on product compositions (broadly defined to allow for a ``background'' distribution, i.e. compositions of the form $\hat{p}(x) = p_b(x) \prod_i \frac{p_i(x)}{p_b(x)}$) implemented with diffusion models, which allows the composition to be implemented via a linear combinations of scores as in \citet{du2023reduce, liu2022compositional}. Our goal is not to propose
a wholly new method of composition but rather to improve theoretical understanding of existing methods.

\textbf{Learning and Generalization.}
Recently, \citet{kamb2024analytic}
demonstrated how a type of compositional generalization
arises from inductive bias in the learning procedure (equivariance
and locality).
Their findings are relevant to our broader motivation,
but complementary to the focus of this work.
Specifically, we focus only on mathematical aspects
of defining and sampling from compositional distributions,
and we do not consider any learning-theoretic aspects
such as inductive bias or sample complexity.
This allows us to study the behavior of
compositional sampling methods
even assuming perfect knowledge of the underlying distributions.

\section{CLEVR Experimental Details}
\label{app:clevr_extra}

All of our CLEVR experiments use raw conditional diffusion
scores, without applying any guidance/CFG \citep{ho2022classifier}.
Details below.

\subsection{Dataset, models, and training details}
\subsubsection{CLEVR dataset}
We used the CLEVR \cite{johnson2017clevr} dataset generation procedure\footnote{\url{https://github.com/facebookresearch/clevr-dataset-gen}} to generate datasets customized to the needs of the present work.
All default objects, shapes, sizes, colors were kept unchanged.
Images were generated in their original resolution of $320\times240$ and down-sampled to a lower resolution of $128\times128$ to facilitate experimentation and to be more GPU resources friendly.
The various datasets we generated from this procedure include:
\begin{itemize}
    \item A background dataset ($0$ objects) with 50,000 samples
    \item Single object dataset with 1,550,000 samples
    \item A dataset having 1 to 5 objects, with 500,000 samples for each object count, for a total of 2,500,000 samples.
\end{itemize}

Our experiments cover two different conditioning setups. In Figures \ref{fig:len_gen}, \ref{fig:len_gen_monster}, \ref{fig:len_gen_extra}, we condition on the 2D location of the object (or the location of one randomly-chosen object, for multi-object distributions). In Figures \ref{fig:clevr_color_comp}, \ref{fig:clever_color_comp_extra}, we condition on the color of the object. In all experiments we condition only a single attribute (either location or color) at a time, with all other attributes sampled randomly and not conditioned on.

\subsubsection{Model architecture}
We used our own PyTorch re-implementation of the EDM2 \cite{karras2024analyzing} U-net architecture.
Our re-implementation is functionally equivalent, and only differs in optimizations introduced to save memory and GPU cycles.
We used the smallest model architecture, e.g. $\texttt{edm2-img64-xs}$ from \url{https://github.com/NVlabs/edm2}.
This model has a base channel width of $128$, resulting in a total of $124\texttt{M}$ trainable weights.
Two versions of this model were used:
\begin{itemize}
    \item An unmodified version for background and class-conditioned experiments.
    \item A modified version for $(x,y)$ conditioning in which we simply replaced Fourier embeddings for the class with concatenated Fourier embeddings for $x$ and $y$.
\end{itemize}

\subsubsection{Training and inference}
In all experiments, the model is trained with a batch size of $2048$ over $128\times2^{20}$ samples by looping over the dataset as often as needed to reach that number.
In practice, training takes around $16$ hours to complete on $32$ A100 GPUs.
We used almost the same training procedure as in EDM2 \cite{karras2024analyzing}, which is basically a standard training loop with gradient accumulation.
The only difference is that we do weight renormalization after the weights are updated rather than before as the authors originally did.

For simplicity, we did not use posthoc-EMA to obtain the final weights used in inference.
Instead we took the average of weights over the last $4096$ training updates.
The denoising procedure for inference is exactly the same as in EDM2 \cite{karras2024analyzing}, e.g. $65$ model calls using a $32$-step Heun sampler.

\subsection{Factorized Conditionals in CLEVR.}
\label{app:clevr-details}

\subsubsection{Single object distributions with empty background}
Let us explicitly describe how
our definition of Factorized Conditionals
captures the CLEVR setting of Figures~\ref{fig:len_gen} and~\ref{fig:len_gen_monster}a.
Recall, the background distribution $p_b$ 
over $n$ pixels is images of an empty scene with no objects.
For each $i \in \{1,\ldots,L\}$ (where $L = 4$ in \cref{fig:len_gen} and $L=9$ in \cref{fig:len_gen_monster}(a)) define the set $M_i \in [n]$ 
as the set of pixel indices surrounding location $i$.
Each $M_i$ should be thought of as a ``mask'' that
that masks out objects at location $i$.
Then, let $M_b := (\cup_i M_i)^c$ be the remaining
pixels in the image, excluding all the masks.
Now we claim the distributions $(p_b, p_1, \ldots, p_L)$
are approximately Factorized Conditionals, with corresponding
coordinate partition $(M_b, M_1, \ldots, M_L)$.
We can confirm each criterion in Definition~\ref{def:factorized}
individually:
\begin{enumerate}
    \setlength{\itemsep}{1pt}
    \item In each distribution $p_i$, the pixels inside the
    mask $M_i$ are approximately independent from the pixels outside the mask,
    since the outside pixels always describe an empty scene.
    \item In the background $p_b$,
    the set of masks $\{M_i\}$ specify approximately
    mutually-independent sets of pixels,
    since all pixels are roughly constant.
    \item The distribution of $p_i$ and $p_b$ approximately agree 
    along all pixels outside mask $M_i$, since they both
    describe an empty scene outside this mask.
\end{enumerate}
Thus, the set of distributions approximately form
Factorized Conditionals. However the conditions of Definition~\ref{def:factorized} do not \emph{exactly} hold, since objects can cast shadows on each other and may even occlude each other. Empirically, this can significantly affect the results when composing many objects, as explored in Figure \ref{fig:len_gen_monster}(a).

\subsubsection{Cluttered distributed with unconditional background}
\begin{figure}[hb]
\vskip 0.2in
\begin{center}
\centerline{
\includegraphics[width=0.8\columnwidth]{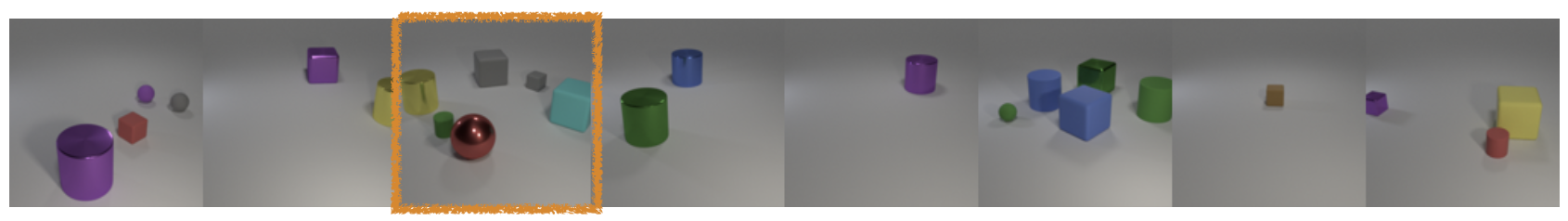}
}
\caption{Samples from unconditional model trained on images containing 1-5 objects. The sampled images sometimes contain 6 objects (circled in orange).}
\label{fig:clutter_uncond_6}
\end{center}
\vskip -0.2in
\end{figure}

Next, we discuss the setting of \cref{fig:len_gen_monster}c, which is a Bayes composition based on an unconditional distribution where each scene contains 1-5 objects (with the number of objects sampled uniformly). The locations and all other attributes of the objects are sampled independently. The conditions label the location of one randomly-chosen object. Just as in the previous case, for each $i \in \{1,2, \ldots, L\}$ ($L=9$ in \cref{fig:len_gen_monster}c), we define the set $M_i \in [n]$  as the set of pixel indices surrounding location $i$, and let $M_b := (\cup_i M_i)^c$ be the remaining pixels in the image, excluding all the masks. Again, we claim that the distributions $(p_b, p_1, \ldots, p_L)$
are approximately Factorized Conditionals, with corresponding
coordinate partition $(M_b, M_1,\ldots, M_L)$. We examine the criteria in Definition~\ref{def:factorized}:
\begin{enumerate}
    \setlength{\itemsep}{1pt}
    \item In each distribution $p_i$, the pixels inside the
    mask $M_i$ are approximately independent from the pixels outside the mask,
    since the outside pixels approximately describe a distribution containing 0-4 objects, and the locations and other attributes of all objects are independent.
    \item In the unconditional background distribution $p_b$, we argue that in practice, the set of masks $\{M_i\}$ are approximately
    mutually-independent. By assumption, the locations and other attributes of all shapes are all independent, and the masks $M_i$ are chosen in these experiment to minimize interaction/overlap. The main difficulty is the restriction to 1-5 total objects, which we discuss below.
    \item The distribution of $p_i$ and $p_b$ approximately agree 
    along all pixels outside mask $M_i$, since $p_i|_{M_i^c}$ contains 0-4 objects, while $p_b|_{M_i^c}$ contains 0-5 objects (since one object could be `hidden' in ${M_i^c}$).
\end{enumerate}
There are, however, two important caveats to the points above. First, overlap or other interaction (shadows, etc.) between objects can clearly violate all three criteria. In our experiment, this is mitigated by the fact that the masks $M_i$ are chosen to minimize interaction/overlap (though interactions start to occur as we compose more objects, leading to some image degradation). Second, since the number of objects is sampled uniformly from 1-5, the presence of one object affects the probability that another will be present. Thus, the masks $\{M_i\}$ are not perfectly independent under the background distribution, nor do $p_i$ and $p_b$ perfectly agree on $M_i^c$. Ideally, each $p_i$ would place an object in mask $M_i$ and independently follow $p_b$ on $M_i^c$, and $p_b$ would be such that the probability that an object appears in mask $M_i$ is independently Bernoulli (c.f. \cref{app:clutter}). In particular, this would imply that the distribution of the total number of objects is Binomial (which allows the total object-count to range from zero to the total-number-of-locations, as well as placing specific probabilities on each object-count), which clearly differs from the uniform distribution over 1-5 objects. However, a few factors mitigate this discrepancy:
\vspace{-1em}
\begin{itemize}
\item A Binomial with sufficiently small probability-of-success places very little probability on large $k$. For example, under $\text{Binomial}(9, 0.3)$, $\mathbb{P}(k=0:5) = 0.04, 0.156, 0.27, 0.27, 0.17, 0.07$ and $\mathbb{P}(k>5) = 0.026$.
\vspace{-0.5em}
\item Empirically, the \emph{learned} unconditional distribution does not actually enforce $k<5$; we sometimes see samples with $k=6$ for example, as seen in \cref{fig:clutter_uncond_6}.
\end{itemize}
\vspace{-1em}
Intuitively, the train distribution is ``close to Bernoulli'' and the \emph{learned} distribution seems to be even closer.

With these considerations in mind, we see that the set of distributions approximately -- though imperfectly -- form
Factorized Conditionals. One advantage of this setting compared to the single-object setting is that the models can learn how multiple objects should interact and even overlap correctly, potentially making it easier compose nearby locations. We explore the length-generalization of this composition empirically in Figure \ref{fig:len_gen_monster}c (note, however, that only compositions of more than 5 objects are actually OOD w.r.t. the distributions $p_i$ in this case).

\subsection{Additional CLEVR samples}
In this section we provide additional non-cherrypicked samples of the experiments shown in the main text.

\begin{table}[htbp]
\centering
\caption{Quantitative analysis of different methods of composition of location-conditioned CLEVR distributions. We generated 100 samples using each composition method, and manually counted (to avoid any potential error in using a classifier) the objects in correct locations (i.e. locations corresponding to the conditioners of the distributions being composed) in each generated image. The table shows the histogram of these manual counts, that is, each column lists the number of images that contained the given number of objects in correct locations. $N$ denotes number of distributions being composed (hence the expected number of objects) -- we test $N=1$ through $N=6$.
\textbf{Single-object empty} composes single-object object distributions with an empty background.
\textbf{Single-object Bayes} composes single-object object distributions with an unconditional background.
\textbf{Bayes-cluttered} composes 1-5 object distributions (with location label assigned to a single randomly-chosen object) with an unconditional background.
}
\label{tab:clevr_counts}
\begin{tabular}{| l | c | r  r  r  r  r  r  r |}
\hline
\textbf{Style} & \textbf{N} & \textbf{0} & \textbf{1} & \textbf{2} & \textbf{3} & \textbf{4} & \textbf{5} & \textbf{6} \\
\hline
Single-object empty 
                  & 1 &   & 100 &     &    &    &    &    \\
                  & 2 &   &     & 100 &    &    &    &    \\
                  & 3 &   &     & 1   & 99 &    &    &    \\
                  & 4 &   &     &     & 2  & 98 &    &    \\
                  & 5 &   &     &     &    & 2  & 98 &    \\
                  & 6 &   &     &     &    &    & 3  & 97 \\
\hline
Single-object Bayes
                  & 1 &   & 100 &     &    &    &    &    \\
                  & 2 &   & 10  & 67  & 32 &    &    &    \\
                  & 3 &   & 36  & 62  & 2  &    &    &    \\
                  & 4 & 77& 23  &     &    &    &    &    \\
                  & 5 & 66& 32  & 2   &    &    &    &    \\
                  & 6 &   & 34  & 6   & 3  &    &    &    \\
\hline
Bayes-cluttered
                  & 1 &   & 100 &     &     &     &    &    \\
                  & 2 &   &     & 100 &     &     &    &    \\
                  & 3 &   &     &     & 100 &     &    &    \\
                  & 4 &   &     &     &     & 100 &    &    \\
                  & 5 &   &     &     &     & 2   & 98 &    \\
                  & 6 &   &     &     &     &     & 2  & 98 \\
\hline
\end{tabular}
\end{table}

\begin{figure}[ht]
\vskip 0.2in
\begin{center}
\centerline{
\includegraphics[width=1.0\columnwidth]{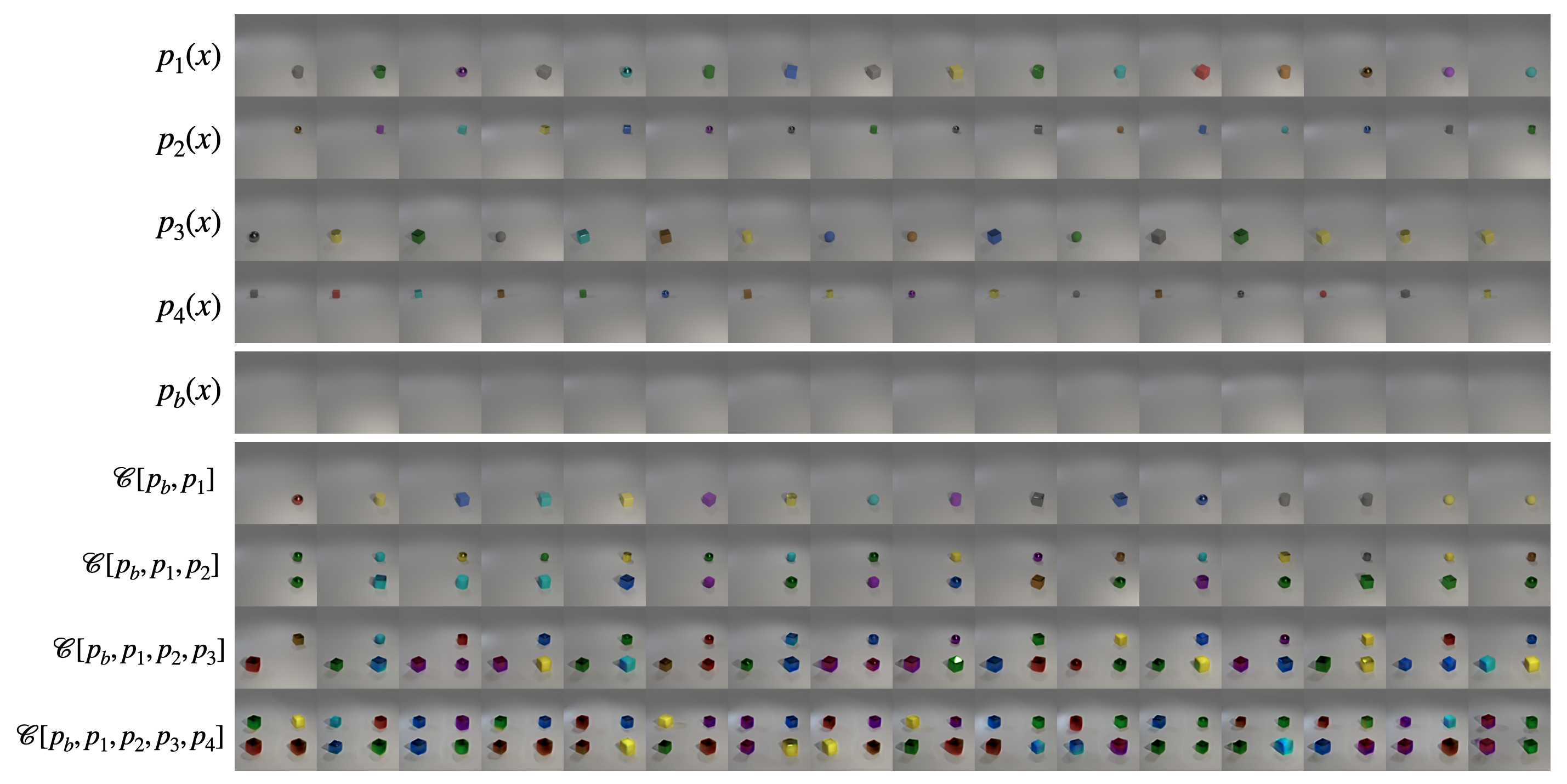}
}
\caption{Additional non-cherrypicked samples for CLEVR experiment of Figure \ref{fig:len_gen}.}
\label{fig:len_gen_extra}
\end{center}
\vskip -0.2in
\end{figure}

\begin{figure}[hb]
\vskip 0.2in
\begin{center}
\centerline{
\includegraphics[width=1.0\columnwidth]{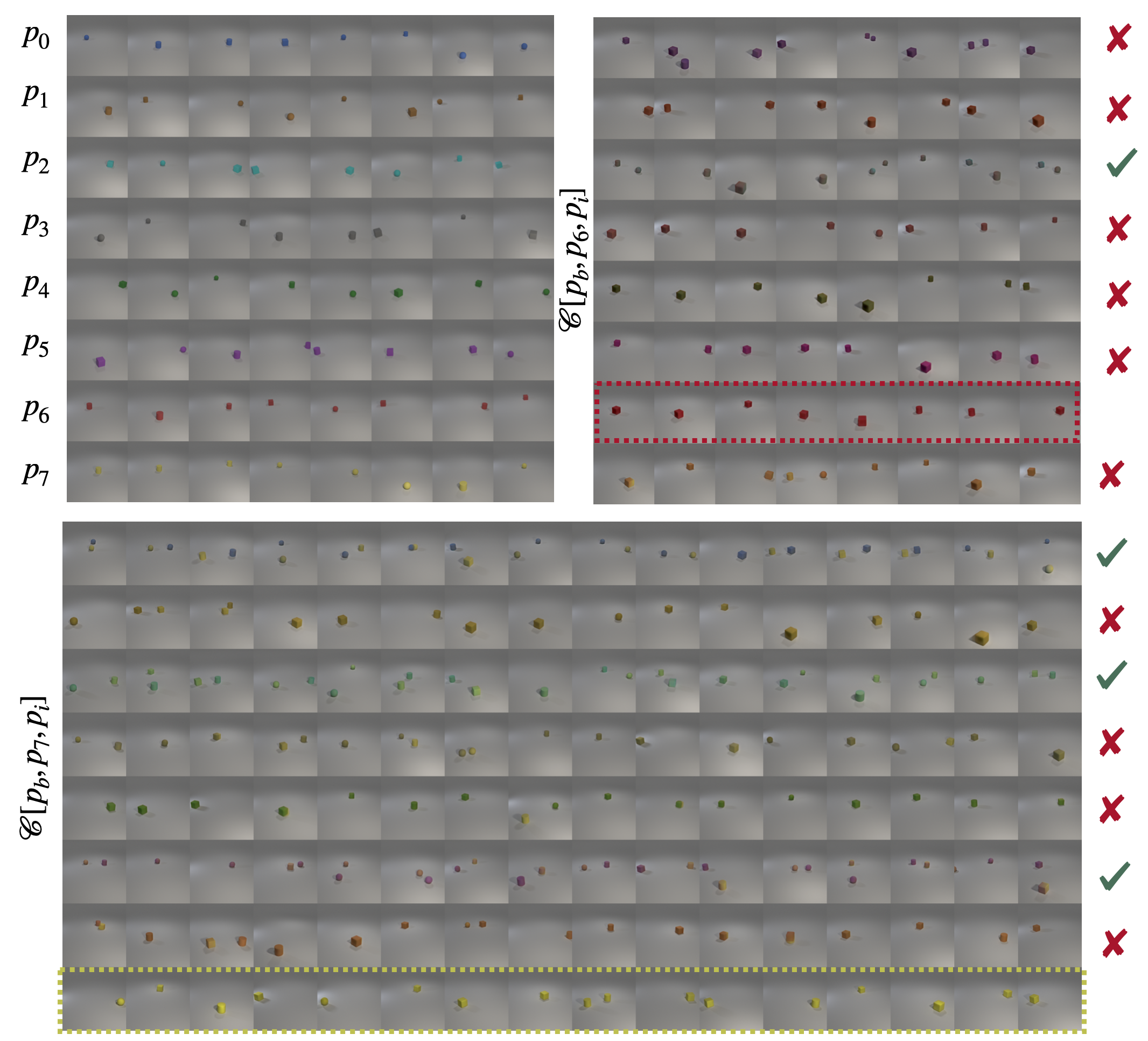}
}
\caption{Additional non-cherrypicked samples for CLEVR experiment of Figure \ref{fig:clevr_color_comp}. Top left grid shows conditional samples for each color. Top right grid shows compositions of red-colored objects ($p_6$) with objects of other colors (8 samples of each), which only succeeds for cyan-colored objects. Bottom grid shows compositions of yellow-colored objects ($p_7$) with objects of other colors (16 samples of each): these are additional samples of the exact experiment shown in Figure \ref{fig:clevr_color_comp}.}
\label{fig:clever_color_comp_extra}
\end{center}
\vskip -0.2in
\end{figure}

\section{SDXL experimental details}
\label{app:sdxl_detail}

\subsection{Figure~\ref{fig:style-content}}

The two models composed are
\begin{enumerate}
    \item An SDXL model \citep{podell2023sdxl} fine-tuned 
    on 30 personal photos of the author's dog (Papaya).
    \item SDXL-base-1.0 \citep{podell2023sdxl}
    conditioned on prompt ``an oil painting in the style of van gogh.''
\end{enumerate}

The background score distribution is the unconditional background 
(i.e. SDXL conditioned on the empty prompt).
We use the DDPM sampler \citep{ho2020denoising} with 30 steps,
using the composed score, and CFG guidance weight of $2$ \citet{ho2020denoising}.

Note that using guidance weight $1$ (i.e. no guidance)
also performs reasonably in this case, but is lower quality.

\subsection{Figure~\ref{fig:dog-horse-hat}}
{\bf Left:}
The two score models composed are
\begin{enumerate}
    \item SDXL-base-1.0 \citep{podell2023sdxl}
    conditioned on prompt ``photo of a dog''
    \item SDXL-base-1.0 \citep{podell2023sdxl}
    conditioned on prompt ``photo of a horse''
\end{enumerate}
The background score distribution is the unconditional background 
(i.e. SDXL conditioned on the empty prompt).

For improved sample quality, we use
a Predictor-Corrector method \citep{song2020score}
with the DDPM predictor and the Langevin dynamics corrector,
both operating on the composed score.
We use 100 predictor denoising steps, and 3 Langevin iterations
per step.
We do not use any guidance/CFG.

{\bf Right:}
Identical setting as above, using prompts:
\begin{enumerate}
    \item ``photo of a dog''
    \item ``photo, with red hat''
\end{enumerate}

Note that the DDPM sampler also performed reasonably in this setting,
but Predictor-Corrector methods improved quality.

\section{CLIP experiment details}
\label{app:clip}

In the CLIP experiment of \cref{fig:clip}, we used the following text prompts for each concept:
\begin{verbatim}
    "dog": "a photograph of a dog",
    "horse": "a photograph of a horse",
    "cat": "a photograph of a cat",
    "watercolor": "a watercolor painting",
    "oil-painting": "an oil painting",
    "hat": "wearing a hat",
    "sunglasses": "wearing sunglasses",
    "uncond": ""
\end{verbatim}
We did an automated collection of 10 images for each concept initially and then manually filtered them to ensure that they actually representative of each concept, leaving us with 10 for ``dog'', 2 for ``horse'', 9 for ``cat'', 3 for ``watercolor'', 4 for ``oil-painting'', 6 for ``hat'', 5 for ``sunglasses'', 10 for ``uncond'' (the latter were arbitrary images searched with no specific keyword). 

For each concept $i$, we ran the image experiment by computing the CLIP embedding of each representative image and averaging them to estimate the mean $\mu_i$. Similarly, we estimated the background mean $\mu_b$ by using the arbitrary images (representing the unconditional distribution). For the text experiment we simply estimated $\mu_i$, $\mu_b$ as the CLIP embedding of the single representative text prompt (or empty prompt).

In either case we then computed the cosine similarity between the mean difference vectors
$$ \frac{\boldsymbol{\mu}_i - \boldsymbol{\mu}_b)^T (\boldsymbol{\mu}_j - \boldsymbol{\mu}_b)}{\|\boldsymbol{\mu}_i - \boldsymbol{\mu}_b\|\|\boldsymbol{\mu}_j - \boldsymbol{\mu}_b\|}$$
to assess whether the condition of \cref{lem:heuristic}, i.e.
$\boldsymbol{\mu}_i - \boldsymbol{\mu}_b)^T (\boldsymbol{\mu}_j - \boldsymbol{\mu}_b) \approx 0$
approximately holds.

\section{Reverse Diffusion and other Samplers}
\label{app:samplers}

\subsection{Diffusion Samplers}
DDPM \citep{ho2020denoising} and DDIM \citep{song2021denoising} are standard reverse diffusion samplers \citep{sohl2015deep, song2019generative} that correspond to discretizations of a reverse-SDE and reverse-ODE, respectively (so we will sometimes refer to the reverse-SDE as DDPM and the reverse-ODE as DDIM for short).
The forward process, reverse-SDE, and equivalent reverse-ODE \citep{song2020score} for the \emph{variance-preserving} (VP) \citep{ho2020denoising} conditional diffusion are
\begin{align}
\textsf{Forward SDE}: dx &= -\half \beta_{t} x dt + \sqrt{\beta_{t}} dw. \label{eq:ddpm_sde} \\
\textsf{DDPM SDE}: \quad
dx &=
-\half \beta_{t} x ~dt 
- \beta_t \grad_x \log p_t(x|c) dt
+ \sqrt{\beta_{t}} d\bar{w}  \label{eq:ddpm} \\
\textsf{DDIM ODE}: \quad
dx &= 
-\half \beta_{t} x~dt 
- \half \beta_t \grad_x \log p_t(x|c) dt. \label{eq:ddim}
\end{align}

\subsection{Langevin Dynamics}
Langevin dynamics (LD) \citep{rossky1978brownian,parisi1981correlation}
an MCMC method for sampling from a desired distribution. It is given by the following SDE \citep{robert1999monte}
\begin{align}
    dx &= \frac{\epsilon}{2} \nabla \log \rho(x) dt + \sqrt{\epsilon} dw, \label{eq:ld}
\end{align}
which converges (under some assumptions)
to $\rho(x)$ \citep{roberts1996exponential}. That is, letting $\rho_s(x)$ denote the solution of LD at time $s$, we have $\lim_{s \to \infty} \rho_s(x) = \rho(x)$.

\section{Factorized Conditionals vs. Orthogonality}
\label{app:score_orthog}

\cref{lem:heuristic} states that Factorized Conditionals (\cref{def:factorized}) implies orthogonality between mean differences (providing a necessary-but-not-sufficient condition to check for FC). The proof is straightforward:
\begin{proof} (\cref{lem:heuristic})
\begin{align*}
p_i(x|_{M_i^c}) &= p_b(x|_{M_i^c}) \quad \text{ by FC} \\
\mu_i = \E_{p_i}[x], &\quad i=1, \ldots, k, \quad \mu_b = \E_{p_b}[x]\\
\implies (\mu_i)_{M_i^c} &= (\mu_b)_{M_i^c}  \\
\implies \text{Support}(\mu_i - \mu_b) &\subset M_i \\
\implies (\mu_i - \mu_b)^T (\mu_j - \mu_b) &= 0, \quad \text{since } M_i \cap M_j = \emptyset\\
\end{align*}
\end{proof}

Similarly, Definition \ref{def:factorized} also implies orthogonality between the score differences (recall that the score is related to the conditional mean as $\grad \log p^t(x) := \frac{1}{\sigma_t^2} \E_p[x - x_t|x_t]$, so this is closely related to \cref{lem:heuristic}). To see this:
\begin{align*}
    v_i^t(x) &:= \grad_x \log p_i^t(x_t) - \grad_x \log p_b^t(x_t) \\
    &= \grad_x \log \frac{p_i^t(x)}{p_b^t(x)} 
    = \grad_x \log \frac{p_i^t(x|_{M_i}) p_b^t(x|_{M_i^c} x)}{p_b^t(x|_{M_i}) p_b^t(x|_{M_i^c})} \\
    &= \grad_x \log \frac{p_i^t(x|_{M_i})}{p_b^t(x|_{M_i})} \\
    \implies v_i^t(x)[k] &= 0, \quad \forall k \notin M_i \\
    \implies v_i^t(x)^T v_j^t(x) &= 0, \quad \forall i \neq j, \quad \text{since } M_i \cap M_j = \emptyset,
\end{align*}
where in the second-to-last line we used the fact that the gradient of a function depending only on a subset of variables has zero entries in the coordinates outside that subset. 

In fact, the same argument implies that $\{v_i^t(x): x \in \R^n\} \subset M_i$; in other words, $\{v_i^t(x): x \in \R^n\}$ and $\{v_j^t(x): x \in \R^n\}$ occupy mutually-orthogonal subspaces. But even this latter condition does not imply the stronger condition of Definition \ref{def:factorized}. To find an equivalent definition in terms of scores we must also capture the independence of the subsets under $p_b$. Specifically:
\begin{align*}
&\left\{
    \begin{aligned}
    p_i^t(x) &= p_i^t(x|_{M_i} x) p_b^t(x|_{M_i^c} x) \\
    p_b^t(x) &= p_b^t(x|_{\bar M} x) \prod_i p_b^t(x|_{M_i})                   
    \end{aligned}
\right. \\
\iff 
&\left\{
    \begin{aligned}
     \grad_x \log p_i^t(x) &= \grad_x \log p_i^t(x|_{M_i} x) + \grad_x \log p_b^t(x|_{M_i^c} x) \\
    \grad_x \log p_b^t(x) &= \grad_x \log p_b^t(x|_{\bar M} x) + \sum_i \grad_x \log p_b^t(x|_{M_i})
    \end{aligned}
\right. \\
\iff 
&\left\{
    \begin{aligned}
     \grad_x \log p_i^t(x) - \grad_x \log p_b^t(x) 
     &= \grad_x \log \frac{p_i^t(x|_{M_i} x)}{p_b^t(x|_{M_i} x)} \\
    \grad_x \log p_b^t(x) &= \grad_x \log p_b^t(x|_{\bar M} x) + \sum_i \grad_x \log p_b^t(x|_{M_i})
    \end{aligned}
\right.
\end{align*}
So an definition equivalent to \cref{def:factorized} in terms of scores could be:
\begin{definition}
    The distributions $(p_b, p_1, p_2, \ldots)$ form \emph{factored conditionals} if the score-deltas $v_i^t := \grad_x \log p_i^t(x) - \grad_x \log p_b^t(x)$ satisfy $\{v_i^t(x): x \in \R^n\} \subset M_i$, where the $M_i$ are mutually-orthogonal subsets, and furthermore the score of the background distribution decomposes over these subsets as follows:
    $\grad_x \log p_b^t(x) = \grad_x \log p_b^t(x|_{\bar M} x) + \sum_i \grad_x \log p_b^t(x|_{M_i})$.
\end{definition}
(Note: this is actually equivalent to a slightly more general version of Definition \ref{def:factorized} that allows for orthogonal transformations, which is the most general assumption under diffusion sampling generates a projective composition, per Lemmas \ref{lem:transform_comp} and \ref{lem:orthogonal_sampling}.)

\section{Connections with the Bayes composition}
\label{app:bayes_connect}

\begin{figure}[ht]
\vskip 0.2in
\begin{center}
\centerline{
\includegraphics[width=0.48\columnwidth]{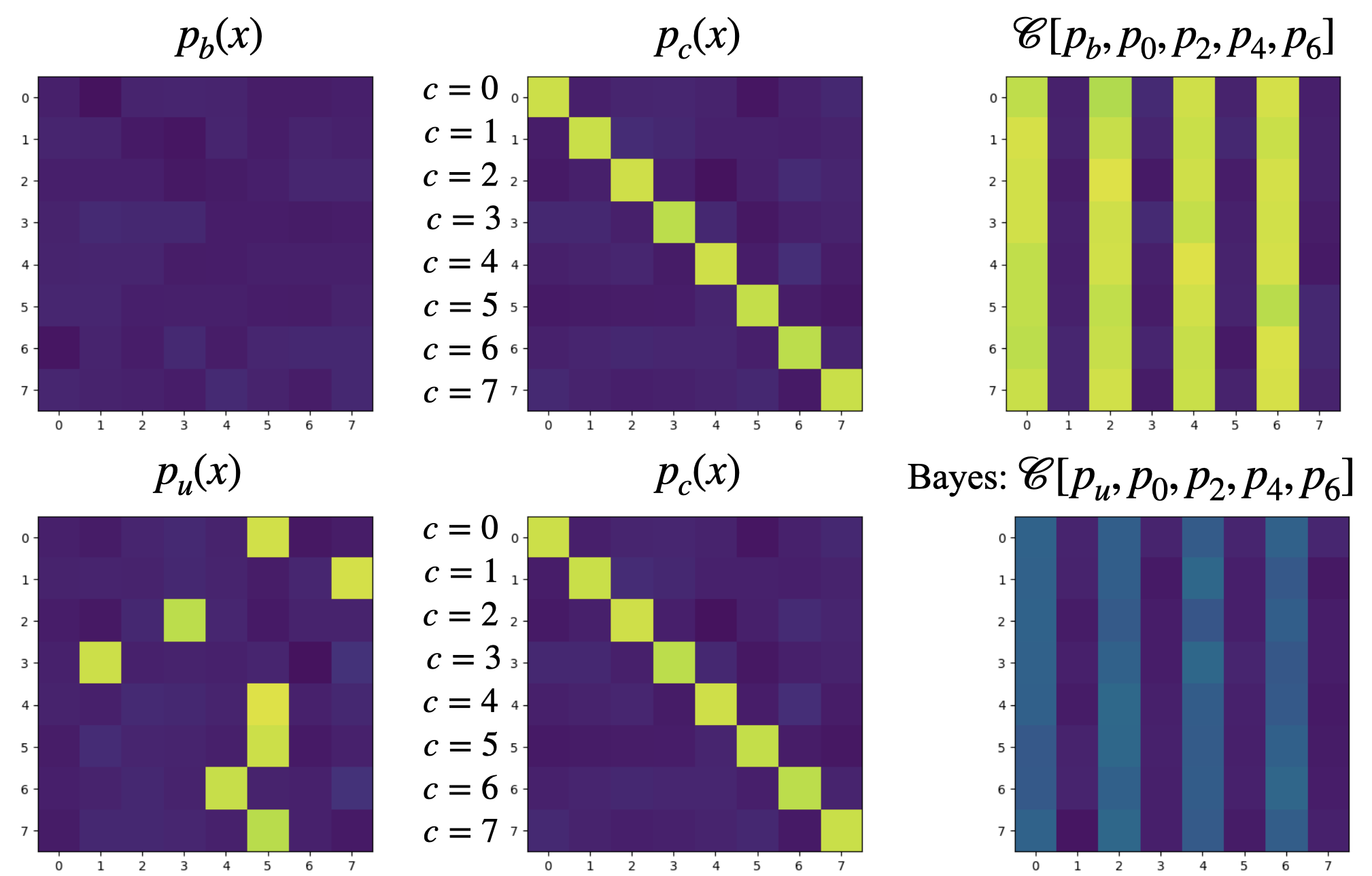}
\vline
\includegraphics[width=0.48\columnwidth]{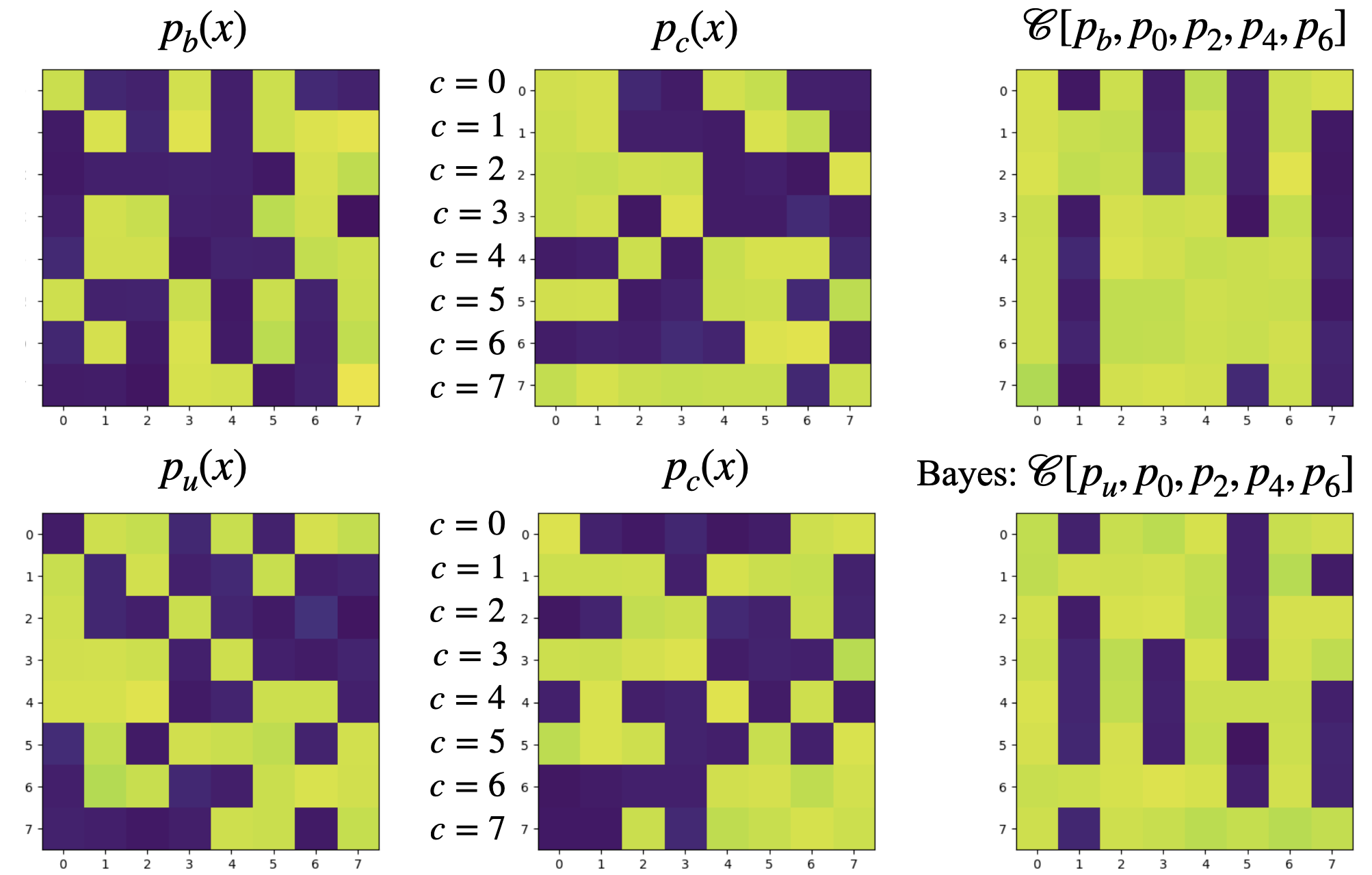}
}
\caption{Bayes composition vs. projective composition. All experiments use exact scores, which is possible since the diffusion-noised distributions are Gaussian mixtures. (Left) Distributions follow \eqref{eq:binary_single_index}: each conditional $p_i$ activates index $i$ only, unconditional $p_u$ averages over the $p_i$, and background $p_b$ is all-zeros. We attempt to compose the conditions $p_0, p_2, p_4, p_6$ and hope to obtain the result $[1, 0, 1, 0, 1, 0]$. This requires length-generalization, since each of the conditionals $p_i$ contains only a single 1. The composition using the empty background $p_b$ (top) achieves this goal, while the Bayes composition using the unconditional $p_u$ (bottom) does not. Note that $[p_b, p_1, p_2, \ldots]$ satisfy Definition \ref{def:factorized} while $[p_u, p_1, p_2, \ldots]$ does not. (Right) Distributions follow \eqref{eq:binary_bernoulli}, where each conditional $p_i$ activates index $i$ on an independently `cluttered' background. In this case the unconditional is similar to the cluttered background. Again we attempt to compose $p_0, p_2, p_4, p_6$, and in this case we find that the composition using $p_u$ works similarly well to $p_b$.}
\label{fig:bayes_binary}
\end{center}
\vskip -0.2in
\end{figure}

\subsection{The Bayes composition and length-generalization}
\label{app:bayes_counterex}

We give a counterexample for which the Bayes composition fails to length-generalize, while composition using an ``empty background'' succeeds. The example corresponds to the experiment shown in Figure \ref{fig:bayes_binary} (left). Suppose we have conditional distributions $p_i$ that set a single index $i$ to one and all other indices to zero, a zero-background distribution $p_b$, and an unconditional distribution formed from the conditionals by assuming $p(c=i)$ is uniform. That is:
\begin{align}
    p_i^t(x_t) &= \cN(x_t; e_i, \sigma_t^2) \propto \exp \left( -\frac{\|x_t - e_i\|^2}{2\sigma_t^2} \right) \notag\\
    p_b^t(x_t) &= \cN(x_t; 0, \sigma_t^2) \propto \exp \left( -\frac{\|x_t\|^2}{2\sigma_t^2} \right) \notag \\
    p_u^t(x_t) &= \frac{1}{n} \sum_{i=1}^n p_i(x_t)
    \label{eq:binary_single_index}
\end{align}
Suppose we want to compose all $n$ distributions $p_i$, that is, we want to activate all indices. It is enough to consider $x_t$ of the special form $x_t = (\alpha, \ldots, \alpha)$ since there is no reason to favor any condition over any another. Making this restriction,
\begin{align*}
    x_t = (\alpha, \ldots, \alpha) \implies 
    p_i^t(x_t) &\propto \exp \left( -\frac{(n-1)\alpha^2 + (1-\alpha)^2}{2\sigma_t^2} \right)
    = \exp \left( -\frac{n\alpha^2-2\alpha+1}{2\sigma_t^2} \right), \quad \forall i \\
    p_u^t(x_t) &= \exp \left( -\frac{n\alpha^2-2\alpha+1}{2\sigma_t^2} \right) \\
    p_b^t(x_t) &\propto \exp \left( -\frac{n\alpha^2}{2\sigma_t^2} \right)
\end{align*}

Let us find the value of $\alpha$ that maximizes the probability under the Bayes composition of all condition:
\begin{align*}
    x_t = (\alpha, \ldots, \alpha) \implies
    \frac{p_i^t(x_t)}{p_u^t(x_t)} &= 1\\
    \implies p_u^t(x_t) \prod_{i=1}^n \frac{p_i^t(x_t)}{p_u^t(x_t)} &\propto p_u^t(x_t) \propto \exp \left( -\frac{n\alpha^2-2\alpha+1}{2\sigma_t^2} \right) 
    = \exp \left( -\frac{n(\alpha - \frac{1}{n})^2 + \text{const}}{2 \sigma_t^2} \right) \\
    \implies \alpha^\star &= \frac{1}{n},
\end{align*}
so the optimum is $\alpha^\star = \frac{1}{n}$. That is, under the Bayes composition the most likely configuration places value $\frac{1}{n}$ at each index we wished to activate, rather than the desired value 1.

On the other hand, if we instead use $p_b$ in the linear score combination and optimize, we find that:
\begin{align*}
    x_t = (\alpha, \ldots, \alpha) \implies
    \implies \frac{p_i^t(x_t)}{p_b^t(x_t)} &\propto \exp \left( -\frac{1-2\alpha}{2\sigma_t^2} \right) \\
    \implies p_b^t(x_t) \prod_{i=1}^n \frac{p_i^t(x_t)}{p_b^t(x_t)} &\propto \exp \left( -\frac{n\alpha^2}{2\sigma_t^2} \right) \exp \left( -\frac{n(1-2\alpha)}{2\sigma_t^2} \right) 
    \propto \exp \left( -\frac{n(\alpha^2-2\alpha+1)}{2\sigma_t^2} \right) \\
    &\propto \exp \left( -\frac{n(\alpha-1)^2}{2\sigma_t^2} \right) \\
    \implies \alpha^\star &= 1
\end{align*}
so the optimum is $\alpha^\star = 1$. That is, the most likely configuration places the desired value $1$ at each index we wished to activate, achieving projective composition, and in particular, length-generalizing correctly.

\subsection{Cluttered Distributions}
\label{app:clutter}
In certain ``cluttered'' settings, the Bayes composition may be approximately projective. We explore this in the following simplified setting, corresponding to the experiment in Figure \ref{fig:bayes_binary} (right). Suppose that $x$ is binary-valued, $M_i = \{i\}, \forall i$, the $x_i$ are independently Bernoulli with parameter $q$ under the background, and the projected conditional distribution $p_i(x|_i)$ just guarantees that $x_i = 1$: 
\begin{align}
    p_b(x |_{i^c}) &\sim \text{Bern}_q(x|_{i^c}), \text{ i.i.d. $\forall i$}, \quad \quad
    p_i(x |_{i}) = \mathbbm{1}_{x|_{i} = 1},
    \label{eq:binary_bernoulli}
\end{align}
The distributions $(p_b, p_1, p_2, \ldots)$ then clearly satisfy Definition \ref{def:factorized} and hence guarantee projective composition. In this case, the unconditional distribution used in the Bayes composition is similar to the background distribution if number of conditions is large. Intuitively, each conditional looks very similar to the Bernoulli background except for a single index that is guaranteed to be equal to 1, and the unconditional distribution is just a weighted sum of conditionals. Therefore, we expect the Bayes composition to be approximately projective.

More precisely, we will show that the unconditional distribution converges to the background in the limit as $n \to \infty$, where $n$ is both the data dimension and number of conditions, in the following sense:
$$ \E_{x \sim p_b} \left[ \left( \frac{p_u(x) - p_b(x)}{p_b(x)} \right)^2 \right] \to 0 \quad \text{as } n \to \infty.$$

We define the conditional and background distributions by:
\begin{align*}
    x \in \R^n, \quad M_i &= \{i\} \\ 
    p_b(x|_{i}) &\sim \text{Bern}_q(x|_{i}), \text{ i.i.d. for $i=1,\ldots, n$} \\
    p_i(x|_{i}) &= \mathbbm{1}_{x|_{i} = 1}, \text{ for all $i=1,\ldots, n$} \\
    \implies p_b(x) &= q^{nnz(x)} (1-q)^{n - nnz(x)}\\ 
    p_i(x) &= \mathbbm{1}_{x|_{i} = 1} p_b(x|_{i^c}) 
    = \mathbbm{1}_{x|_{i} = 1}  
    q^{nnz(x|_{i^c})} (1-q)^{n - 1 - nnz(x|_{i^c})}
\end{align*}
We construct the unconditional distribution with assuming uniform probabibility over all labels: $p_u(x) := \frac{1}{n} \sum_i p_i(x)$.
The number-of-nonzeros (nnz) in all of these distributions follow Binomial distributions:
\begin{align*}
    x \sim p_b \implies p_b(nnz(x) = k) &\sim \text{Binom}(k; n, q) \\
    x \sim p_i \implies p_i(nnz(x) = k) &= p_b(nnz(x|_{i^c}) = k-1) \\
    &\sim \text{Binom}(k-1; n-1, q) \quad \text{if } k > 0 \text{ else } 0 \\
    x \sim p_u \implies p_u(nnz(x) = k) &= \frac{1}{n} \sum p_i(nnz(x) = k) \\
    &\sim \text{Binom}(k-1; n-1, q) \quad \text{if } k > 0 \text{ else } 0
\end{align*}
The basic intuition is that for large $k$ and $n$, $p_b \sim \text{Binom}(k; n, q)$ and $p_u \sim \text{Binom}(k-1; n-1, q)$ are similar. More precisely, we can calculate:
\begin{align*}
    \E_{x \sim p_b} \left[ \left( \frac{p_u(x) - p_b(x)}{p_b(x)} \right)^2 \right] &=
    \E_{x \sim p_b} \left[ \left(\frac{nnz(x)}{qn} - 1\right)^2 \right], \quad \text{since } \frac{B(k-1; n-1, q)}{B(k; n, q)} = \frac{k}{qn} \\
    &= \E_{k \sim \text{Binom}(n, q)} \left[ \left(\frac{k}{qn} - 1\right)^2 \right] 
    = \frac{1}{(nq)^2} \E_{k \sim \text{Binom}(n, q)} \left[ (k - nq)^2 \right] \\
    &= \frac{1}{(nq)^2} \text{Var}(k), \quad k \sim \text{Binom}(n, q) \\
    &= \frac{1}{(nq)^2} n q (1-q) 
    = \frac{1-q}{nq} \to 0 \quad \text{as } n \to \infty.
\end{align*}

\section{Proof of \cref{lem:compose} and \cref{lem:fc_relax}}
\label{app:compose_pf}
\begin{proof} (\cref{lem:compose})
For any set of distributions $\vec{q} = (q_b, q_1, q_2, \ldots)$ satisfying Definition \ref{def:factorized}, we have
\begin{align}
\cC[\vec{q}](x) &:= q_b(x) \prod_i \frac{q_i(x)}{q_b(x)} = q_b(x) \prod_i \frac{q_b(x_t|_{M_i^c}) q_i(x|_{M_i})}{q_b(x|_{M_i^c})q_b(x|_{M_i})} \notag \\
&= q_b(x) \prod_i \frac{q_i(x|_{M_i})}{q_b(x|_{M_i})} 
= q_b(x|_{M_b}) \prod_i q_i(x_t|_{M_i})
\label{eq:comp_indep}
\end{align}
(where we used \eqref{eqn:cc-cond} in the second equality). 
Since $(p_b, p_1, p_2, \ldots)$ satisfy Definition \ref{def:factorized} by assumption, applying \eqref{eq:comp_indep} gives
$$\cC[\vec{p}](x) = p_b(x|_{M_b}) \prod_i p_i(x|_{M_i}) := \hat{p}(x),$$
so the composition at $t=0$ is projective, as desired.
Now to show that reverse-diffusion sampling with the compositional scores generates $\cC[\vec{p}]$, we need to show that
$$\cC[\vec{p^t}]= N_t[\cC[\vec{p}]],$$ 
where $p^t := N_t[p]$ denotes the $t$-noisy version of distribution $p$ under the forward diffusion process.
First, notice that if $\vec{p}$ satisfies Definition \ref{def:factorized}, then $\vec{p^t}$ does as well. This is because the diffusion process adds Gaussian noise independently to each coordinate, and thus preserves independence between sets of coordinates.
Therefore by \eqref{eq:comp_indep}, we have
$\cC[\vec{p^t}](x)= p^t_b(x|_{\bar{M}}) \prod_i p^t_i(x_t|_{M_i}).$
Now we apply the same argument (that diffusion preserves independent sets of coordinates) once again, to see that $\cC[\vec{p^t}]= N_t[\cC[\vec{p}]],$
as desired.
\end{proof}

\begin{proof}(\Cref{lem:fc_relax})
    \begin{align*}
        \text{KL}(\cC^\star[\vec{p}] || \cC[\vec{p}]) &\equiv \E_{\cC^\star[\vec{p}]} \left[ \log \frac{\cC^\star[\vec{p}]}{\cC[\vec{p}]}\right] \\
        &= \E_{\cC^\star[\vec{p}]} \left[ \log \frac{p_b(x|_{M_b}) \prod_i p_i(x|_{M_i})}{p_b(x) \prod_i \frac{p_i(x)}{p_b(x)}}\right] \\
        &= \E_{\cC^\star[\vec{p}]} \left[ \sum_i \log \left(\frac{p_i(x|_{M_i})}{p_i(x)} 
        + p_b(x) \right)
        + \log \frac{p_b(x|_{M_b})}{p_b(x)} \right]
        \\
        &= \E_{\cC^\star[\vec{p}]} \left[ \sum_i \left( \log \frac{p_i(x|_{M_i}) p_b(x|_{M_i^c})}{p_i(x)} +\log p_b(x|_{M_i}) \right) 
        + \log \frac{p_b(x|_{M_b})}{p_b(x)} \right]
        \\
        &= \sum_i \E_{\cC^\star[\vec{p}]} \left[ \log \frac{p_i(x|_{M_i}) p_b(x|_{M_i^c})}{p_i(x)} \right] + \E_{\cC^\star[\vec{p}]}\left[ \log \frac{p_b(x|_{M_b})}{p_b(x|_{M_b} | x|_{M_b^c})} \right]
    \end{align*}
    
    \begin{align*}
        \E_{\cC^\star[\vec{p}]} \left[ \log \frac{p_i(x|_{M_i}) p_b(x|_{M_i^c})}{p_i(x)} \right] &= \E_{\cC^\star[\vec{p}]} \left[ \log \frac{p_i(x|_{M_i})}{p_i(x|_{M_i} | x|_{M_i^c})} \right] \\
        &= \int p_b(x|_{M_b}) \prod_j p_j(x|_{M_j}) \log \frac{p_i(x|_{M_i})}{p_i(x|_{M_i} | x|_{M_i^c})} dx \\
        &= \int p_b(x|_{M_b}) \prod_{j \neq i} p_j(x|_{M_j}) \int p_i(x|_{M_i}) \log \frac{p_i(x|_{M_i})}{p_i(x|_{M_i} | x|_{M_i^c})} dx|_{M_i} dx|_{M_i^c} \\
        &= \E_{p_b(x|_{M_b}) \prod_{j \neq i} p_j(x|_{M_j})} \left[ \text{KL}[ p_i(x|_{M_i}) || p_i(x|_{M_i} | x|_{M_i^c})] \right] \\
        &\le \sup_{x|_{M_i^c}} \quad \text{KL}[ p_i(x|_{M_i}) || p_i(x|_{M_i} | x|_{M_i^c})] \le \epsilon_i, \quad \text{and similarly for $b$.}
    \end{align*}
\end{proof}

Furthermore, note that by the Data Processing Inequality, the KL divergence between any two distributions is non-increasing as both distributions are convolved with Gaussian noise, and in fact it is known that the KL divergence between $N_t[q], N_t[r]$ decreases in $t$ for any two distinct distributions $q \neq r$. Specifically,
. This does not necessarily guarantee that the \emph{supremum} appearing in \Cref{lem:fc_relax} must decrease with $t$. However we can also show that with one additional assumption.

\section{Parameterization-Independent Compositions and Proof of Lemma \ref{lem:transform_comp}}
\label{app:param-indep}

The proof of Lemma \ref{lem:transform_comp} relies on certain general fact about parametrization-independence of certain operators, which we develop here.

Suppose we have an operator that 
takes as input two probability distributions $(p, q)$
over the same space $\cX$,
and outputs a distribution over $\cX$.
That is, $F: \Delta(\cX)\x\Delta(\cX) \to \Delta(\cX)$.
We can think of such operators as performing some kind of ``composition''
of $p, q$.

Certain operators are \emph{independent of parameterization}, meaning
for any reparameterization of the base space $A: \cX \to \cY$, we have
\[
F(p, q) = A^{-1}\sharp( F(A \sharp p, A \sharp q) )
\]
or equivalently:
\[
F(A \sharp p, A \sharp q) = A \sharp F(p, q),
\]
where $\sharp$ is the pushforward:
\[ (\mathcal{A} \sharp p)(z) := \frac{1}{\left|\grad \mathcal{A}(z)\right|} p(\mathcal{A}^{-1}(z)). \]

This means that reparameterization commutes with the operator:
it does not matter if we first reparameterize, then compose, or first compose, then reparameterize. A few examples:
\begin{enumerate}
    \item The pointwise-geometric median, $F(p, q)(x) := \sqrt{p(x)q(x)}$, is independent of reparameterization:
    \item Squaring a distribution, $F(p, q)(x) := p(x)^2$, is NOT independent of reparameterization:
    \item The ``CFG composition'' \citep{ho2022classifier}, $F(p, q)(x) := p(x)^\gamma q(x)^{1 - \gamma}$, is independent of reparameterization:
\end{enumerate}

We can analogously define parametrization-independence for operators on more than 2 distributions.
Notably, given a tuple of distributions $\vec{p} = (p_b, p_1, p_2, \dots, p_k)$,
our composition operator $\cC$ of Definition \ref{def:comp_oper}, $\cC[\vec{p}] \propto p_b(x) \prod_i \frac{p_i(x)}{p_b(x)}$ 
is independent of parameterization.

\begin{lemma}[Parametrization-independence of 1-homogeneous operators]
    If an operator $F$ is 1-homogeneous, i.e. $F(tp, tq, \ldots) = t F(p, q, \ldots)$ and operates pointwise, then it is independence of parametrization.
    \label{lem:1-homog-param-indep}
\end{lemma}

\begin{proof}
\begin{align*}
    F(\mathcal{A} \sharp p, \mathcal{A} \sharp q, \ldots)(z) &= F(\mathcal{A} \sharp p(z), \mathcal{A} \sharp q(z), \ldots), \quad \text{pointwise} \\
    &= F \left( \frac{1}{\left|\grad \mathcal{A}(z)\right|} p(\mathcal{A}^{-1}(z)), \frac{1}{\left|\grad \mathcal{A}(z)\right|} q(\mathcal{A}^{-1}(z)), \ldots \right) \\
    &= \frac{1}{\left|\grad \mathcal{A}(z)\right|}F\left(p(\mathcal{A}^{-1}(z)), q(\mathcal{A}^{-1}(z)), \ldots \right), \quad \text{1-homogeneous} \\
    &= \mathcal{A} \sharp F(p, q, \ldots)(z)
\end{align*}
\end{proof}

\begin{corollary}[Parametrization-invariance of composition]
    The composition operator $\cC$ given by Definition \ref{def:comp_oper} is independent of parametrization.
    \label{corr:comp-param-indep}
\end{corollary}
\begin{proof}
    The composition operator given by Definition \ref{def:comp_oper} is 1-homogeneous:
    \begin{align*}
        \cC(tp_b, tp_1, tp_2, \ldots)(x) &= tp_b(x) \prod_i \frac{tp_i(x)}{tp_b(x)} 
        = tp_b(x) \prod_i \frac{p_i(x)}{p_b(x)} 
        = t\cC(p_b, p_1, p_2, \dots)(x)
    \end{align*}
    and so the result follows from Lemma \ref{lem:1-homog-param-indep}.
    Alternatively, a direct proof is:
    \begin{align*}
    \cC(p_b, p_1, p_2, \ldots)(x) &:= p_b(x) \prod_i \frac{p_i(x)}{p_b(x)} \\
    \cC(\mathcal{A} \sharp p_b, \mathcal{A} \sharp p_1, \mathcal{A} \sharp p_2, \ldots)(z) 
    &= (\mathcal{A} \sharp p_b)(z) \prod_i \frac{(\mathcal{A} \sharp p_i)(z)}{(\mathcal{A} \sharp p_b)(z)} 
    = \frac{1}{\left|\grad \mathcal{A}\right|} p_b(\mathcal{A}^{-1}(z)) \prod_i \frac{p_i(\mathcal{A}^{-1}(z))}{p_b(\mathcal{A}^{-1}(z))} 
    = \mathcal{A} \sharp \cC(p_b, p_1, p_2, \ldots)(z).
    \end{align*}
\end{proof}

\cref{lem:transform_comp} follows from \cref{corr:comp-param-indep}:
\begin{proof} (\cref{lem:transform_comp})
Let $(q_b, q_1, q_2, \dots, q_k) := (\cA \sharp p_b, \cA \sharp p_1, \dots \cA\sharp p_k)$, for which Definition~\ref{def:factorized} holds by assumption. Applying an intermediate result from the proof of Theorem~\ref{lem:compose} gives:
\begin{align*}
    \cC[\vec{q}](z) &:= q_b(z) \prod_i \frac{q_i(z)}{q_b(z)} 
    = q_b(z|_{\bar{M}}) \prod_i q_i(z|_{M_i}).
\end{align*}
By Corollary \ref{corr:comp-param-indep}, $\cC$ is independent of parametrization, hence 
$$ \cA \sharp \hat{p} := \cA \sharp (\cC[\vec{p}]) = \cC[\vec{\cA \sharp p}] := \cC(\vec{q}).$$
\end{proof}

\section{Proof of \cref{lem:orthogonal_sampling}}
\label{app:orthog_sample_pf}

\begin{figure}[ht]
\vskip 0.2in
\begin{center}
\centerline{
\includegraphics[width=0.5\columnwidth]{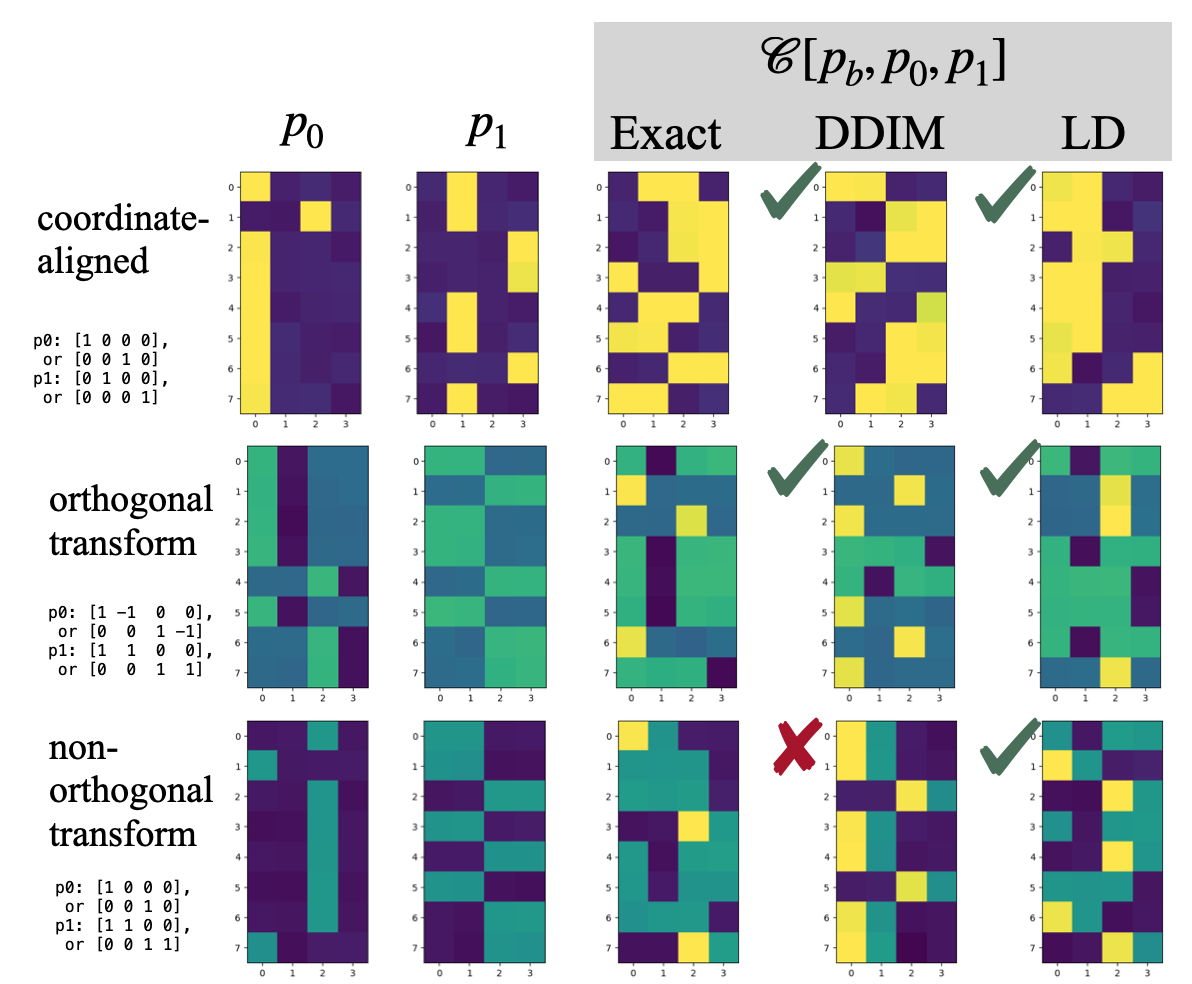}
}
\caption{Synthetic composition experiment illustrating the sampling guarantees of \cref{lem:orthogonal_sampling} in contrast to the lack-of-guarantees in the non-orthogonal case. We compare a coordinate-aligned case (which satisfies \cref{def:factorized} in the native space) (top), an orthogonal-transform case (middle) (which satisfies the assumptions of \cref{lem:orthogonal_sampling}), and a non-orthogonal-transform case (bottom) (which satisfies the assumptions of \cref{lem:transform_comp} but not of \cref{lem:orthogonal_sampling}). In the first two cases the correct composition can be sampled using either diffusion (DDIM) or Langevin dynamics (LD) at $t=0$, while in the final case DDIM sampling is unsuccessful although LD at $t=0$ still works.
The distributions are 4-dimensional and we show 8 samples (rows) for each. We show samples from the individual conditional distributions $p_0, p_1$ using DDIM, samples from the desired exact composition $\cC[p_b, p_0, p_1]$ at $t=0$ (obtained by sampling from $\cA \sharp \cC[\vec{p}]$ with DDIM and transforming by $\cA^{-1}$), samples from the composition $\cC[p_b, p_0, p_1]$ using DDIM with exact scores, and samples from the composition $\cC[p_b, p_0, p_1]$ using Langevin dynamics (LD) with exact scores at time $t=0$ in the diffusion schedule ($\sigma_{\min} = 0.02$). 
The noiseless distributions $p_0$ and $p_1$ are each 4-dimensional 2-cluster Gaussian mixtures with means as noted in the figure, equal weights, and standard deviation $\tau = 0.02$. For example, in the non-orthogonal-transform case, $p_0$ has means $[1, 0, 0, 0]$ and $[0, 0, 1, 0]$, and $p_1$ has means $[1, 1, 0, 0]$ and $[0, 0, 1, 1]$, (which can be transformed to satisfy \cref{def:factorized} via a non-orthogonal linear transform).}
\label{fig:orthog_sampling}
\end{center}
\vskip -0.2in
\end{figure}

Figure \ref{fig:orthog_sampling} shows a synthetic experiment illustrating the sampling guarantees of \cref{lem:orthogonal_sampling} in contrast to the lack-of-guarantees in the non-orthogonal case.

The proof of \cref{lem:orthogonal_sampling} relies on the fact that diffusion noising commutes with orthogonal transformation, i.e. $\cA \sharp N_t[q] = N_t[\cA \sharp q]$ if $\cA$ is orthogonal, since standard Gaussians are invariant under orthogonal transformation.

\begin{proof}(\cref{lem:orthogonal_sampling})
By assumption, $(\cA \sharp p_b, \cA \sharp p_1, \dots \cA\sharp p_k)$ satisfy Definition~\ref{def:factorized}, where $\cA(z) = Az$ with $A$ an orthonormal matrix. By Lemma \ref{lem:transform_comp}, $\hat{p} = \cC[\vec{p}]$ satisfies \eqref{eqn:p_hat_A}. To show that reverse-diffusion sampling with scores $s_t = \grad_x \log \cC[\vec{p}^t]$ generates the composed distribution $\cC[\vec{p}]$ we need to show that composition commutes with the forward diffusion process, i.e.
$$\cC[\vec{p^t}] = N_t[\cC[\vec{p}]].$$ 
Theorem~\ref{lem:compose} immediately gives us
$$ \cC[N_t[\cA \sharp p]] = N_t[\cC[\cA \sharp p]]. $$
Now we have to be careful with commuting operators. We know that composition is independent of parametrization, i.e. $\cA \sharp \cC[\vec{p}] = \cC[\vec{\cA \sharp p}]$. Diffusion noising $N_t$ commutes with orthogonal transformation, i.e. $\cA \sharp N_t[q] = N_t[\cA \sharp q]$ if $\cA$ is orthogonal, because a standard Gaussian multiplied by an orthonormal matrix $Q$ remains a standard Gaussian:  $\eta \sim \cN(0, I) \implies Q\eta \sim \cN(0, QQ^T) = \cN(0, I)$ (this is false for non-orthogonal transforms, however). Therefore, in the orthogonal case, we can rewrite:
$$ \cA \sharp  \cC[N_t[p]] = \cA \sharp  N_t[\cC[p]], $$
which implies the desired result since $\cA$ is invertible.
\end{proof}

\section{Proof and further discussion of Lemma \ref{lem:lipschitz}}
\label{app:lipschitz}

\subsection{Benefits of sampling at $t=0$}
\label{app:hmc}
Interestingly, \cite{du2023reduce} have observed that sophisticated samplers like Hamiltonian Monte Carlo (HMC) requiring energy-based formulations often outperform standard diffusion sampling for compositional quality. Lemmas \ref{lem:transform_comp} and \ref{lem:lipschitz} help explain why this may be the case. In particular, HMC (or any variant of Langevin dynamics) can enable sampling $p^0$ at time $t=0$, even when the path $p^t$ used for annealing does not necessarily represent a valid forward diffusion process starting from $p^0$ (as \citet{du2023reduce} note, $\cC[\vec{p}^t]]$ may not be). Lemma \ref{lem:transform_comp} should gives us hope that approximately-projective composition may often be possible at $t=0$, since it allows \emph{any} invertible transform $\cA$ to transform into a factorized feature space (which need not be explicitly constructed). However, that does not mean that we can actually \emph{sample} from this projection at time $t=0$. As Lemma \ref{lem:lipschitz} shows, $\cC[\vec{p}^t]]$ is not necessarily a valid diffusion path unless $\cA$ is orthogonal, so standard diffusion sampling may not work. This is consistent with \citet{du2023reduce}'s observation that non-diffusion samplers that allow sampling at $t=0$ may be necessary. Interestingly, Lemma \ref{lem:lipschitz} further cautions that sometimes $\cC[\vec{p}^t]]$ may not even be an effective annealing path for any kind of sampler (which is consistent with our own experiments but not reported by other works, to our knowledge.)

\subsection{Proof of Lemma \ref{lem:lipschitz}}
\label{app:lipschitz_proof}

\begin{figure}[ht]
\vskip 0.2in
\begin{center}
\centerline{
\includegraphics[width=0.45\columnwidth]{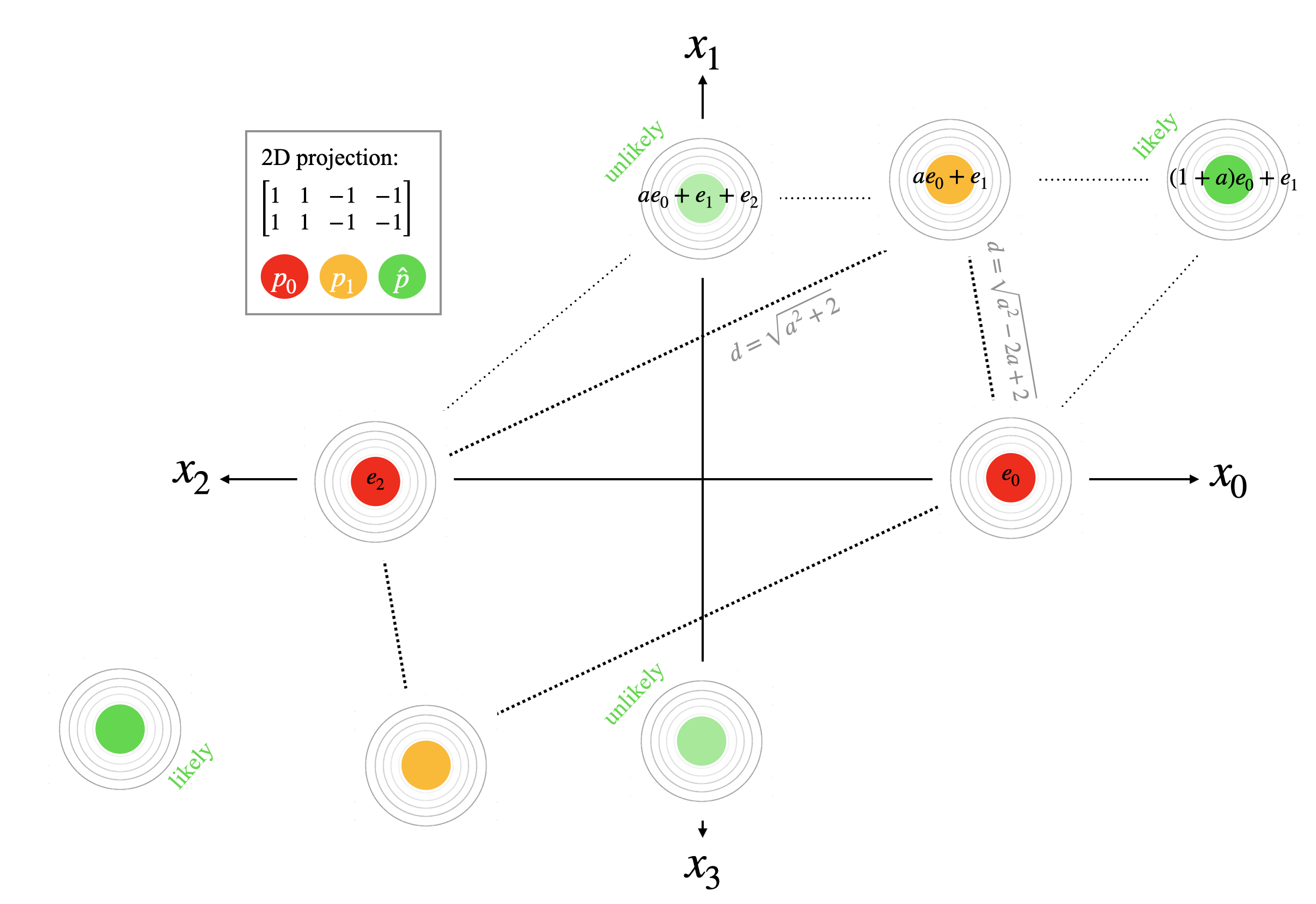}
\includegraphics[width=0.55\columnwidth]{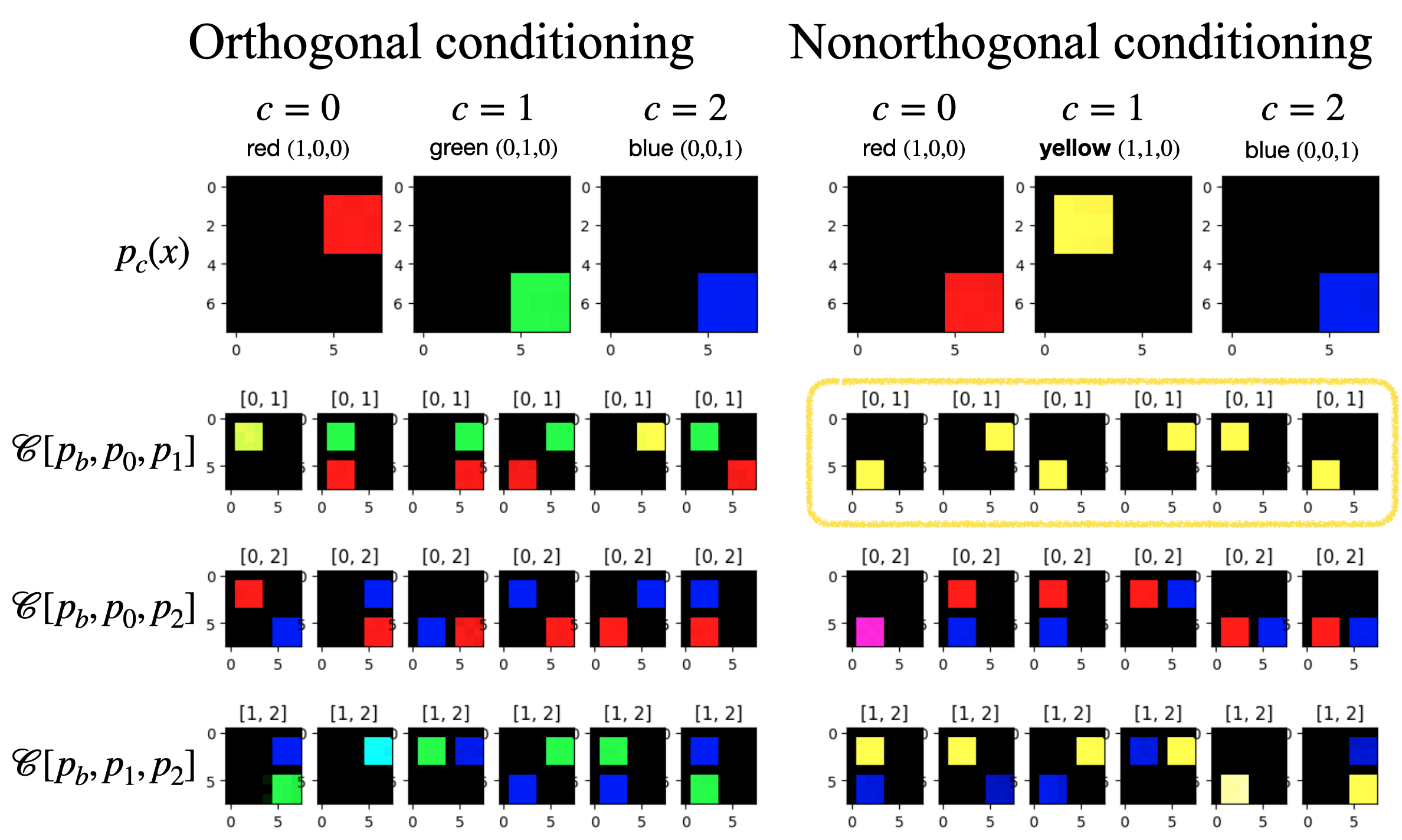}
}
\caption{(Left) A visualization of the intuition behind the proof of Lemma \ref{lem:lipschitz}, under a 2D projection. (Right) An experiment where the colors red, green, and blue all compose projectively, while the colors red and yellow do not. We trained a Unet on images each containing a single square in one of 4 locations (selected randomly) and a certain color, conditioned on the color. We then generate composed distributions by running DDIM on the composed scores. The desired result of composing red and blue is an image containing a red and a blue square, both with randomly-chosen locations (so we occasionally get a purple square when the locations overlap). When we try to compose red and yellow, we only only ever obtain a single yellow square.Note that in pixel space, the colors are represented as red $(1,0,0)$, green $(0,1,0)$, blue $(0,0,1)$, yellow $(1,1,0)$, so that red, green and blue are all orthogonal and are expected to work by Lemma \ref{lem:compose}, while red and yellow are not orthogonal, and fail as allowed by Lemma \ref{lem:lipschitz}. In fact this experiment is closely related to the counterexample used to prove Lemma \ref{lem:lipschitz}.}
\label{fig:yellow}
\end{center}
\vskip -0.2in
\end{figure}

\begin{figure}[ht]
\vskip 0.2in
\begin{center}
\centerline{
\includegraphics[width=1.0\columnwidth]{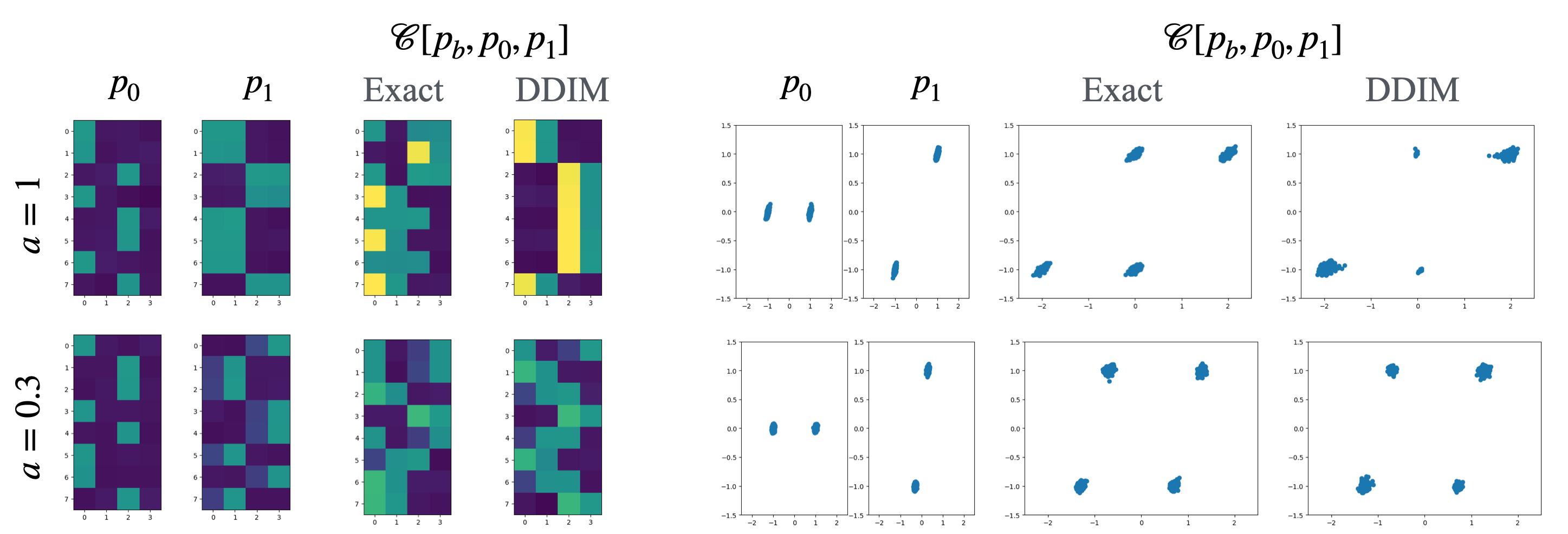}
}
\caption{Composition experiments for the setting in the proof of Lemma \ref{lem:lipschitz}. Left pane shows 8 samples (rows) of each distribution in the native 4d representation; right pane shows 1000 samples under the 2D projection used in Figure \ref{fig:yellow}. We show samples from the individual conditional distributions $p_0, p_1$ using DDIM, samples from the desired exact composition $\cC[p_b, p_0, p_1]$ at $t=0$ (obtained by sampling from $\cA \sharp \cC[\vec{p}]$ with DDIM and transforming by $\cA^{-1}$), and samples from the composition $\cC[p_b, p_0, p_1]$ using DDIM with exact scores.  We take $\tau = 0.02$, and set $\sigma_{\min} = 0.02$ in the diffusion schedule. In the top row we take $a=1$ (``very non-orthogonal'') as in the proof, and compare this to $a=0.3$ (``mildly non-orthogonal'') in the bottom row. With $a=1$, as in the proof we see that DDIM barely samples two of the clusters. With $a=0.3$, DDIM still slightly undersamples the ``hard'' clusters but the effect is much less pronounced.}
\label{fig:eg3}
\end{center}
\vskip -0.2in
\end{figure}

We will prove Lemma \ref{lem:lipschitz} using a counterexample, which is inspired by an experiment, shown in Figure \ref{fig:eg3} (left), where non-orthogonal conditions fail to compose projectively.

The basic idea for the counterexample is that given a distribution $p(x)$ with two conditions, $c=0,1$, such at $t=0$, 
\begin{align*}
    p_0(x) &\approx \half \delta_{e_0}(x) + \half \delta_{e_2}(x), \quad \quad \quad
    p_1(x) \approx \half \delta_{ae_0 + e_1}(x) + \half \delta_{ae_2 + e_3}(x),
\end{align*}
for some $0 < a \le 1$,
so the conditional distributions do not satisfy the independence assumption of Definition \ref{def:factorized},
However, there exists a (linear, but not orthogonal) $A$ such that the distribution of $z = Ax$ is axis-aligned
\begin{align*}
    (A \sharp p_0)(z) &\approx \half \delta_{e_0}(x) + \half \delta_{e_2}(x), \quad \quad \quad
    (A \sharp p_1)(z) \approx \half \delta_{e_1}(x) + \half \delta_{e_3}(x),
\end{align*}
and thus does satisfy Definition \ref{def:factorized} at $t=0$, which guarantees correct composition of $p$ at $t=0$ under Lemma \ref{lem:transform_comp}. The correct composition should sample uniformly from
$ \{ (1+a)e_0 + e_1, \quad e_0 + a e_2 + e_3, \quad ae_0 + e_2 + e_1, \quad (1+a) e_2 + e_3 \}$.
What goes wrong is that as soon as we add Gaussian noise to the distribution $p(x)$ at time $t > 0$ of the diffusion forward process, the relationship $z = Ax$ breaks and so we are no longer guaranteed correct composition of $p^t(x)$. In fact, the distribution is still a GMM but places nearly all its weight on only two of the four clusters, namely: 
$ \{ (1+a)e_0 + e_1, (1+a) e_2 + e_3 \}.$
Intuitively, let us focus on the mode $ae_0 + e_1$ of $p_1$ and consider how it interacts with the two modes $e_0, e_2$ of $p_0$, at some time $t > 0$ when we have isotropic Gaussians centered at each mode. 
Since $ae_0 + e_1$ is further away from $e_2$ (distance $\sqrt{a^2+2}$) than it is from $e_0$ (distance $\sqrt{a^2-2a+2})$, it is much less likely under $\cN(e_2, \sigma_t)$ than $\cN(e_0, \sigma_t)$, leading to a lower weight. This intuition is shown graphically in a 2D projection in Figure \ref{fig:yellow} (left).

For the detailed proof, we actually want to ensure that $p$ has full support even at $t=0$ so we add a little bit of noise to it, but choose the covariance such that $z=Ax$ still holds at $t=0$.

We begin by defining the distributions we will use for the counterexample.
\begin{definition}
\label{def:counterexample}
For any choice of $\tau > 0$, define the following \emph{counterexample} distributions:
\begin{align*}
    p_0^0(x) &= \half \cN(x; e_0, \tau^2 (A^TA)^{-1}) + \half \cN(x; e_2, \tau^2 (A^TA)^{-1}) \\
    p_1^0(x) &= \half \cN(x; ae_0 + e_1, \tau^2 (A^TA)^{-1}) + \half \cN(x; ae_2 + e_3, \tau^2 (A^TA)^{-1}) \\
    p_b^0(x) &= \cN(x; 0, \tau^2 (A^TA)^{-1}),
    \quad \text{where }
    A := \begin{bmatrix}
        1 & -a & 0 & 0 \\
        0 & 1 & 0 & 0 \\
        0 & 0 & 1 & -a \\
        0 & 0 & 0 & 1
    \end{bmatrix},
\end{align*}
so that in the transformed space
\begin{align*}
    (A \sharp p)(z) &:= p(A^{-1} z), \quad z = Ax \\
    (A \sharp p_b^0)(z) &= \cN(z; 0, \tau^2) \\
    (A \sharp p_0^0)(z) &= \half \cN(z; e_0, \tau^2) + \half \cN(z; e_2, \tau^2) \\ 
    (A \sharp p_1^0)(z) &= \half \cN(z; e_1, \tau^2) + \half \cN(z; e_3, \tau^2).
\end{align*}

The noised versions at time $t > 0$ are
\begin{align*}
    p_i^t(x_t|x_0) &:= \cN(x_t; x_0, \sigma_t^2) \\
    p_0^t(x) &= \half \cN(x_t; e_0, \sigma_t^2 I + \tau^2 (A^TA)^{-1}) + \half \cN(x; e_2, \sigma_t^2 I + \tau^2 (A^TA)^{-1}) \\
    p_1^t(x) &= \half \cN(x_t; ae_0 + e_1, \sigma_t^2 I + \tau^2 (A^TA)^{-1}) + \half \cN(x; ae_2 + e_3, \sigma_t^2 I + \tau^2 (A^TA)^{-1}). \\
\end{align*}
\end{definition}

Next, we state some intermediate results we will need for the proof.
\begin{proposition}
\label{prop:gaussian_composition}
A composition of two Gaussians with identical covariance (using a Gaussian background with zero mean and the same covariance) is a scaled Gaussian with the following parameters
\begin{align*}
    \frac{\cN(x; \mu_1; \Sigma) \cN(x; \mu_2, \Sigma)}{\cN(x; 0; \Sigma)} 
    &= C \cN(x; \mu_1 + \mu_2, \Sigma), \quad \text{where }
    C = \exp(\mu_1^T \Sigma^{-1} \mu_2).
\end{align*}
\end{proposition}

\begin{proof} %
\begin{align*}
    \frac{\cN(x; \mu_1; \Sigma) \cN(x; \mu_2, \Sigma)}{\cN(x; 0; \Sigma)} 
    &= (2\pi)^{-\frac{n}{2}} |\Sigma|^{-\half} \frac{e^{-\half (x-\mu_1)^T \Sigma^{-1} (x-\mu_1)} e^{-\half (x-\mu_2)^T \Sigma^{-1} (x-\mu_2)} }{e^{-\half x^T\Sigma^{-1}x}} \\
    &= (2\pi)^{-\frac{n}{2}} |\Sigma|^{-\half} \exp \left( -\half (x-\mu_1)^T \Sigma^{-1} (x-\mu_1) -\half (x-\mu_2)^T \Sigma^{-1} (x-\mu_2) + \half x^T\Sigma^{-1}x \right) \\
    &= (2\pi)^{-\frac{n}{2}} |\Sigma|^{-\half} \exp \left( -\half x^T\Sigma^{-1}x + x^T \Sigma^{-1} (\mu_1 + \mu_2) - \half \mu_1^T \Sigma^{-1} \mu_1 - \half \mu_2^T \Sigma^{-1} \mu_2 \right) \\
    &= C (2\pi)^{-\frac{n}{2}} |\Sigma|^{-\half} \exp \left( -\half (x - \mu_1 - \mu_2)^T \Sigma^{-1}  (x - \mu_1 - \mu_2) \right) \\
    &= C \cN(x; \mu_1 + \mu_2, \Sigma) \\
    C &= \exp \left( - \half \mu_1^T \Sigma^{-1} \mu_1 - \half \mu_2^T \Sigma^{-1} \mu_2 + \half (\mu_1 + \mu_2)^T \Sigma^{-1} (\mu_1 + \mu_2)   \right) \\
    &= \exp(\mu_1^T \Sigma^{-1} \mu_2)
\end{align*}
\end{proof}

\begin{proposition}
\label{prop:counterex_comp}
With $(p_b, p_1, p_2)$ from \cref{def:counterexample}, defining $\hat p^t(x) := \cC[p_b^t, p_0^t, p_1^t]$, we have that $\hat p^0(x)$, $\hat p^t(x)$, and $N_t[\hat p^0](x)$ are all Gaussian mixtures (GMs) with identical means:
$$ \vec{\mu} = \{(1+a)e_0 + e_1, \quad e_0 + a e_2 + e_3, \quad ae_0 + e_2 + e_1, \quad (1+a) e_2 + e_3\}, $$
and the following weights and covariances: 
\begin{align*}
    \hat p^0(x): \quad \text{weights: } w^0 &:= \frac{1}{4} [ 1, 1, 1, 1 ], \\
    \text{covariance: } \tilde{\Sigma_0} &:= \tau^2 (A^TA)^{-1} \\
    \\
    \hat p^t(x): \quad \text{weights: } w^t &:= [\frac{1}{2}-\epsilon, \epsilon, \epsilon, \frac{1}{2}-\epsilon], \\ \epsilon &:= \half S(-\xi), \quad \xi := \frac{a\sigma_t^2}{(a^2 + 2) \sigma_t^2 \tau^2 + \sigma_t^4 + \tau^4} \\
    & \quad \text{ where } S(z) := \frac{1}{e^{-z} + 1} \quad \text{(logistic function)}, \\
    \text{covariance: } \tilde{\Sigma_t} &:= \sigma_t^2 I + \tau^2 (A^TA)^{-1} \\
    \\
    N_t[\hat p^0](x): \quad \text{weights: } w^0 &:= \frac{1}{4} [ 1, 1, 1, 1 ], \\
    \text{covariance: } \tilde{\Sigma_t} &:= \sigma_t^2 I + \tau^2 (A^TA)^{-1} \\
\end{align*}
\end{proposition}

\begin{proof}
    We apply \cref{prop:gaussian_composition} to the distributions of \cref{def:counterexample} to analyze $\hat p^0(x)$, $\hat p^t(x)$, and $N_t[\hat p^0](x)$. \cref{prop:gaussian_composition} gives that all three distributions are Gaussian mixtures with identical means $\vec{\mu}$, 
    and variances $\tilde{\Sigma_0}$, $\tilde{\Sigma_t}$, and $\tilde{\Sigma_t}$, respectively. The weights for $\hat p^0(x)$ and $N_t[\hat p^0](x)$ are uniform ($w^0 = \frac{1}{4} [ 1, 1, 1, 1 ]$). We just need to calculate the weights for $\hat p^t(x)$.

    First we compute the covariance and inverse covariance:
    \begin{align*}
        \tilde{\Sigma_t} &:= \sigma_t^2 I + \tau^2 (A^TA)^{-1} = \sigma_t^2 I + \tau^2
        \begin{bmatrix}
            1+a^2 & a & 0 & 0 \\
            a & 1 & 0 & 0 \\
            0 & 0 & 1+a^2 & a \\
            0 & 0 & a & 1
        \end{bmatrix} \\ 
        \\
        \tilde{\Sigma_t}^{-1} &= \frac{1}{(a^2 + 2) \sigma_t^2 \tau^2 + \sigma_t^4 + \tau^4} \begin{bmatrix}
            \sigma_t^2 + \tau^2 & -a \tau^2 & 0 & 0 \\
            -a \tau^2 & (a^2 + 1)\tau^2 + \sigma_t^2 & 0 & 0 \\
            0 & 0 & \sigma_t^2 + \tau^2 & -a \tau^2 \\
            0 & 0 & -a \tau^2 & (a^2 + 1) \tau^2 + \sigma_t^2
        \end{bmatrix}.
    \end{align*}
    After some algebra (namely, computing $C = \exp(\mu_1^T \tilde{\Sigma_t}^{-1} \mu_2)$ for each cluster), we find that
    \begin{align*}
        \hat p^t(x): \quad w^t &\propto [\exp(\xi), 1, 1, \exp(\xi)], \quad \xi := \frac{a\sigma_t^2}{(a^2 + 2) \sigma_t^2 \tau^2 + \sigma_t^4 + \tau^4} \\
        &= [\frac{1}{2}-\epsilon, \epsilon, \epsilon, \frac{1}{2}-\epsilon], \quad \epsilon := \half S(-\xi), \text{ where } S(z) := \frac{1}{e^{-z} + 1} \text{ (logistic function)}.
    \end{align*}
\end{proof}

The intuition for the proof of \cref{lem:lipschitz} will be that when $\epsilon$ is small, clusters (1,2) have much lower weight than clusters (0,3) in the GM. In that case, we can lower-bound the $W_2$ distance by noting that since $w^t$ has almost no mass on the two of the clusters, we will need to move a little less than $1/4$ probability over to those clusters. For example we need to move $1/4$ probability onto cluster $e_0 + a e_2 + e_3$ from either $(1+a)e_0 + e_1$ (L2 distance between means is $\sqrt{2a + 2}$) or $(1+a) e_2 + e_3$ (L2 distance $\sqrt{2}$). So overall we will have to move a bit less that $1/2$ probability at least $\sqrt{2}$ distance.

We restate the following results from \citet{delon2020wasserstein}, which we will need to help bound the $W_2$ distance.
\begin{theorem} (Mixture Wasserstein distance \citep{delon2020wasserstein})
    \label{thm:delon}
    \begin{align*}
    MW_2(q_0, q_1) &:= \inf_{\gamma \in \Pi(q_0, q_1) \cap \text{GMM}_{2d}(\infty)} \int \|y_0 - y_1\|^2 d\gamma(y_0,y_1), \\
    MW_2^2(q_0, q_1) &= \min_{c \in \Pi(w_0, w_1)} \sum_{k,l} c_{k,l} W_2^2(q_0^k, q_1^l) \quad \text{(Delon Prop. 4)},\\
    W_2(q_0, q_1) &\le MW_2(q_0, q_1) \le W_2(q_0, q_1) + 2 \sum_{i=0,1} \sum_{k=1}^{K_i} w_i^k \text{Tr}(\Sigma_i^k) \quad \text{(Delon Prop. 6)},
\end{align*}
where $\Pi(q_0, q_1)$ denotes the set of all joint distributions with marginals $q_0$ and $q_1$, and $\text{GMM}_{d}(\infty) := \cup_{K \ge 0} \text{GMM}_d(K)$ denotes the set of all finite GMMs.
\end{theorem}

We will also need one more standard fact about the $W_2$ distance between Gaussians.
\begin{proposition} ($W_2$ distance between Gaussians; standard)
\label{prop:W2_Gaussian}
\begin{align*}
W_2^2(\cN(\mu_x, \Sigma_x), \cN(\mu_y, \Sigma_y)) 
&= \|\mu_x - \mu_y\|_2^2 + \text{Tr}(\Sigma_x + \Sigma_y - 2 (\Sigma_x^\half \Sigma_y \Sigma_x^\half)^\half) 
\ge \| \mu_x - \mu_y \|_2^2.
\end{align*}
\end{proposition}

\bigskip
\begin{proposition}
    \label{prop:gmm_lipschitz}
    If $p$ is a Gaussian mixture distribution of the form
    $$p(x) := \sum_{k=1}^K w_i \cN(\mu_k, C_k), \quad x \in \R^n $$
    then $p$ is $(nK)^\half$-Lipschitz w.r.t Wasserstein 2-distance:
    $$W_2(p^{t'}, p^t) \le (nK)^\half |\sigma_t - \sigma_{t'}|,$$
    (that is, $\cO(1)$-Lipschitz, where $\cO$ only hides constants depending on ambient dimension and number-of-clusters). 
\end{proposition}
\begin{proof}
\begin{align*}
    p(x) &:= \sum_{k=1}^K w_i \cN(\mu_k, C_k)\\
    p^t(x) &:= \sum_{k=1}^K w_i \cN(\mu_k, C_k + \sigma_t^2 I) \\
    W_2^2(\cN(\mu_x, \Sigma_x), \cN(\mu_y, \Sigma_y)) &:= \|\mu_x - \mu_y\|_2^2 + \text{Tr}(\Sigma_x + \Sigma_y - 2 (\Sigma_x^\half \Sigma_y \Sigma_x^\half)^\half) \\
    &:= \|\mu_x - \mu_y\|_2^2 + \| \Sigma_x^\half - \Sigma_y^\half\|_F^2 \quad \text{if $\Sigma_x, \Sigma_y$ commute} \\
    \implies W_2^2(p^{t'}[k], p^t[k]) &= \| (C_k + \sigma_{t}^2 I)^\half - (C_k + \sigma_{t'}^2 I)^\half\|_F^2 \\
    &= \| (\Lambda + \sigma_{t}^2 I)^\half - (\Lambda + \sigma_{t'}^2 I)^\half\|_F^2, \quad \text{where $C_k = U \Lambda U^T$ is eigendecomposition}  \\
    &\le \| (\sigma_t - \sigma_{t'}) I \|_F^2, \quad \text{(by concavity of square root and $\Lambda \succeq 0$)} \\
    &= n (\sigma_t - \sigma_{t'})^2
\end{align*}
\begin{align*}
    W_2^2(p^{t'}, p^t) &\le MW_2^2(p^{t'}, p^t) \\
    &:= \min_{c \in \Pi(w, w)} \sum_{k, l} c_{k,l} W_2^2(p^{t'}[k], p^t[l]) \\
    &\le \sum_{k}^K W_2^2(p^{t'}[k], p^t[k]), \quad \text{(since $c=I \in \Pi(w, w)$)} \\
    &\le nK (\sigma_t - \sigma_{t'})^2 \\
    \implies W_2(p^{t'}, p^t) &\le (nK)^\half |\sigma_t - \sigma_{t'}|.
\end{align*}
Thus $p$ is $(nK)^\half$-Lipschitz w.r.t. $W_2$ distance.
\end{proof}

\begin{proposition}
\label{prop:w_2_bounds_smooth}
With $(p_b, p_1, p_2)$ from \cref{def:counterexample} and $\hat p^t(x) := \cC[p_b^t, p_0^t, p_1^t]$, the $W_2$-distance between $\hat{p}^t$ and $\hat{p}^0$ is bounded as follows 
    $$ (1 - 4\epsilon)^{\half} - 4(\tau^2(4+2a^2) + 2\sigma_t^2) \le W_2(\hat{p}^0, \hat{p}^t) \le \left((1 - 4\epsilon) (1+a) + 2( \tau^2(4 + 2 a^2) + 2 \sigma_t^2) \right)^\half.$$
\end{proposition}

\begin{proof}
Using \cref{prop:W2_Gaussian},
\begin{align*} 
    MW_2^2(\hat{p}^0, \hat{p}^t) &= \min_{c \in \Pi(w^0, w^t)} \sum_{k, l} c_{k,l} W_2^2(\hat{p}^0[k], \hat{p}^t[l]) \\
    &= \min_{c \in \Pi(w^0, w^t)} \sum_{k, l} c_{k,l} \left( \|\mu_k - \mu_l\|_2^2 + \text{Tr}(\tilde \Sigma_0 + \tilde \Sigma_t - 2 (\tilde \Sigma_0^\half \tilde \Sigma_t \tilde \Sigma_t^\half)^\half) \right) \\
    c_{(0 \text{ or } 3), (1 \text{ or } 2)} &= 1/4 - \epsilon \\
    2 &\le \|\mu_{(0 \text{ or } 3)} - \mu_{(1 \text{ or } 2)}\|_2^2 \le 2(1+a) \\
    0 &\le \text{Tr}(\tilde \Sigma_0 + \tilde \Sigma_t - 2 (\tilde \Sigma_0^\half \tilde \Sigma_t \tilde \Sigma_t^\half)^\half) \le \text{Tr}(\tilde \Sigma_0 + \tilde \Sigma_t) = 2 \tau^2(4 + 2 a^2) + 4 \sigma_t^2 \\
    \implies 1 - 4\epsilon &\le MW_2^2(N_t[\hat{p}^0], \hat{p}^t) \le (1 - 4\epsilon) (1+a) + 2 \tau^2(4 + 2 a^2) + 4 \sigma_t^2 \\
    \implies (1 - 4\epsilon)^\half &\le MW_2(N_t[\hat{p}^0], \hat{p}^t) \le \left((1 - 4\epsilon) (1+a) + 2 \tau^2(4 + 2 a^2) + 4 \sigma_t^2\right)^\half.
\end{align*}
Above, we noted that any $c \in \Pi(w^0, w^t)$ has to move at least $\frac{1}{4} - \epsilon$ probability each away from indices 1 or 2 and onto indices either 0 or 3, and for any of these moves we can bound the squared L2 distance the mass must move between 2 and $2(1+a)$. We also used simple bounds on the trace term.

We can use \cref{thm:delon} to bound the $W_2$ distance by the $MW_2$ distance:
Using \cref{thm:delon} \citep{delon2020wasserstein},
\begin{align*}
    W_2(\hat{p}^0, \hat{p}^t) &\le MW_2(\hat{p}^0, \hat{p}^t) \\
    &\le \left((1 - 4\epsilon) (1+a) + 2 \tau^2(4 + 2 a^2) + 4 \sigma_t^2\right)^\half \\
    W_2(\hat{p}^0, \hat{p}^t) &\ge MW_2(\hat{p}^0, \hat{p}^t) - 2 \sum_{k}^{} (w^0[k] \text{Tr}(\tilde{\Sigma}^0) + w^t[k]\text{Tr}(\tilde{\Sigma}^t))\\
    &\ge MW_2(\hat{p}^0, \hat{p}^t) - 2(\text{Tr}(\tilde{\Sigma}^0) + \text{Tr}(\tilde{\Sigma}^t))\\
    &\ge (1 - 4\epsilon)^{\half} - 2( 4\sigma_t^2 + 2 \tau^2(4+2a^2)).\\
\end{align*}
\end{proof}

\begin{proposition}
\label{prop:w_2_bounds_diff}
With $(p_b, p_1, p_2)$ from \cref{def:counterexample} and $\hat p^t(x) := \cC[p_b^t, p_0^t, p_1^t]$, the $W_2$-distance between $\hat{p}^t$ and $N_t[\hat{p}^0]$ is bounded as follows 
    $$ (1 - 4\epsilon)^{\half} - 4(\tau^2(4+2a^2) + 4\sigma_t^2) \le W_2(N_t[\hat{p}^0], \hat{p}^t) \le (1 - 4\epsilon)^\half (1+a)^\half.$$
\end{proposition}
\begin{proof}
Using \cref{prop:W2_Gaussian},
    \begin{align*} 
    MW_2^2(N_t[\hat{p}^0], \hat{p}^t) &= \min_{c \in \Pi(w^0, w^t)} \sum_{k, l} c_{k,l} W_2^2(N_t[\hat{p}^0][k], \hat{p}^t[l]) \\
    &= \min_{c \in \Pi(w^0, w^t)} \sum_{k, l} c_{k,l} \|\mu_k - \mu_l\|_2^2 \\
    c_{(0 \text{ or } 3) \to (1 \text{ or } 2)} &= 1/4 - \epsilon \\
    2 &\le \|\mu_{(0 \text{ or } 3)} - \mu_{(1 \text{ or } 2)}\|_2^2 \le 2(1+a) \\
    \implies 1 - 4\epsilon &\le MW_2^2(N_t[\hat{p}^0], \hat{p}^t) \le (1 - 4\epsilon) (1+a) \\
    \implies (1 - 4\epsilon)^\half &\le MW_2(N_t[\hat{p}^0], \hat{p}^t) \le (1 - 4\epsilon)^\half (1+a)^\half.
\end{align*}
Using \cref{thm:delon} \citep{delon2020wasserstein},
\begin{align*}
    W_2(N_t[\hat{p}^0], \hat{p}^t) &\le MW_2(N_t[\hat{p}^0], \hat{p}^t) \\
    &\le (1 - 4\epsilon)^\half (1+a)^\half \\
    W_2(N_t[\hat{p}^0], \hat{p}^t) &\ge MW_2(N_t[\hat{p}^0], \hat{p}^t) - 2 \sum_{k} (w^0[k] + w^t[k]) \text{Tr}(\tilde{\Sigma}^t))\\
    &\ge MW_2(N_t[\hat{p}^0], \hat{p}^t) - 4\text{Tr}(\tilde{\Sigma}^t)\\
    &\ge (1 - 4\epsilon)^{\half} - 4(4 \sigma_t^2 + \tau^2(4+2a^2)).
\end{align*}
\end{proof}

Now we have all the pieces to prove \cref{lem:lipschitz}.
\begin{proof} (\cref{lem:lipschitz})

 We will show that the distributions $(p_b, p_1, p_2)$ of \cref{def:counterexample} satisfy \cref{lem:lipschitz}. We make the choices $a=1$ and $\sigma_t := t$ for simplicity, and define $\hat p^t(x) := \cC[p_b^t, p_0^t, p_1^t]$.

Lemma \ref{lem:transform_comp} applied to the distributions of \cref{def:counterexample} implies that at time $t=0$,
\begin{align*}
    \hat p^0(x) :=\hat \cC[\vec{p}] &:= \frac{p_0^0(x)p_1^0(x)}{p_b^0(x)} = p_0^0(x|_{(0,2)})p_1^0(x|_{(1,3)}).
\end{align*}
That is, we have a projective composition at time $t=0$.

For Part 1 of Lemma \ref{lem:lipschitz}, we note that the $p_i$ are Gaussian mixtures, therefore \cref{prop:gmm_lipschitz} gives that each $p_i$ is $\cO(1)$-Lipschitz w.r.t Wasserstein 2-distance:
$$\forall i: \quad W_2(p_i^{t}, p_i^{t'}) \leq \sqrt{8} |t - t'|.$$

For Part 2, we need to bound the Lipschitz constant of $\hat p^t := \hat \cC[\vec{p}^t]$.

\cref{prop:w_2_bounds_smooth} gives:
\begin{align*}
    &(1 - 4\epsilon)^{\half} - 4(\tau^2(4+2a^2) + 2\sigma_t^2) \le W_2(\hat{p}^0, \hat{p}^t) \le \left((1 - 4\epsilon) (1+a) + 2( \tau^2(4 + 2 a^2) + 2 \sigma_t^2) \right)^\half. \\
    &\quad \text{ where }\epsilon := \half S(-\xi), \quad \xi := \frac{a\sigma_t^2}{(a^2 + 2) \sigma_t^2 \tau^2 + \sigma_t^4 + \tau^4}, \quad
    \text{ with } S(z) := \frac{1}{e^{-z} + 1}.
\end{align*}
Plugging in $a=1$,
\begin{align*}
    &(1 - 4\epsilon)^{\half} - 24\tau^2 - 8 \sigma_t^2 \le W_2(\hat{p}^0, \hat{p}^t) \le \left(2(1 - 4\epsilon) + 12\tau^2 + 4 \sigma_t^2\right)^\half \\
    &\quad \text{ where }\epsilon := \half S(-\xi), \quad \xi := \frac{\sigma_t^2}{3 \sigma_t^2 \tau^2 + \sigma_t^4 + \tau^4}.
\end{align*}

We will to show that
$$\exists t, t': \quad \half \tau^{-1}|t - t'| \leq W_2(q^{t}, q^{t'}) \leq 2 \tau^{-1}|t - t'|.$$

After some algebra, we find that for any fixed $\tau$, the minimum and maximum of $\epsilon$ are
\begin{align*}
    \sigma_t &\to 0 \implies \xi \to 0 \implies \epsilon \to \frac{1}{4} \text{ (max)}\\
    \sigma_t &= \tau \implies \xi = \xi^\star(\tau) := \frac{1}{5\tau^2} \implies \epsilon = \half S \left(-\xi^\star(\tau) \right) \text{ (min)}\\
    \sigma_t &= \tau \to 0 \implies \xi^\star(\tau) \to \infty \implies \epsilon \to 0 \text{ (min)}
\end{align*}

Thus, taking $\sigma_t = \tau$ (the minimizer of $\epsilon$) and choosing any $\tau^2 < \frac{1}{66}$ (somewhat arbitrarily, but small enough), gives
\begin{align*}
    \sigma_t = \tau &\implies (1 - 4\epsilon)^{\half} - 32 \tau^2 \le W_2(\hat{p}^0, \hat{p}^t) \le \left(2(1 - 4\epsilon) + 32 \tau^2\right)^\half,
    \quad \text{ where }\epsilon := \half S(-\frac{1}{5 \tau^2}) \\
    \tau^2 < \frac{1}{66} &\implies \epsilon \approx 10^{-6} \ll \frac{1}{4}\tau^2 \implies (1 - 4\epsilon)^{\half} \ge 1 - \tau^2 \\
    &\implies 0.5 \le 1 - 33 \tau^2 \le W_2(\hat{p}^0, \hat{p}^t) \le \left(2 + 32 \tau^2\right)^\half < 2.
\end{align*}

Thus, for any choice of $\tau^2 < \frac{1}{66}$, if we take $\sigma_t := t = \tau$ and $t'=0$, we have as desired that
\begin{align*}
     0.5 \tau^{-1}|t| \equiv 0.5 \leq W_2(\hat{p}^0, \hat{p}^t) \leq 2 \equiv 2 \tau^{-1}|t|,
\end{align*}
that is,
$$\exists t, t': \quad W_2(\hat{p}^{t'}, \hat{p}^t) = \Theta(\tau^{-1}|t - t'|).$$
\end{proof}

We can also prove another lemma using the same counterexample as \cref{lem:lipschitz}:
\begin{lemma}
\label{lem:bad_diffus_path}
Let $q^t$ denote the composed distribution at time $t$: $q^t := \cC[\vec{p}^t]$, and $N_t$ be the Gaussian-noising operator. There exist distributions $\{p_b, p_1, \dots p_k\}$ over $\R^n$ and a value of $t$ such that $N_t[q^0]$ (the ideal diffusion path to $q^0$) differs from $q^t$ (the path actually followed) by at least $\Omega(1)$:
    $$\exists t: W_2(N_t[q^0], q^t) \geq \half.$$
\end{lemma}

\begin{proof}

We will show that the distributions of \cref{def:counterexample} satisfy \cref{lem:bad_diffus_path}. We make the choices $a=1$ and $\sigma_t := t$ for simplicity.
 
By \cref{prop:w_2_bounds_diff},
\begin{align*}
    &(1 - 4\epsilon)^{\half} - 4(\tau^2(4+2a^2) + 4\sigma_t^2) \le W_2(N_t[\hat{p}^0], \hat{p}^t) \le (1 - 4\epsilon)^\half (1+a)^\half \\
    &\quad \text{ where }\epsilon := \half S(-\xi), \quad \xi := \frac{a\sigma_t^2}{(a^2 + 2) \sigma_t^2 \tau^2 + \sigma_t^4 + \tau^4}, \quad
    \text{ with } S(z) := \frac{1}{e^{-z} + 1}.
\end{align*}

Taking $a=1$:
\begin{align*}
    &(1 - 4\epsilon)^{\half} - 24\tau^2 - 16\sigma_t^2 \le W_2(N_t[\hat{p}^0], \hat{p}^t) \le \sqrt{2}(1 - 4\epsilon)^\half \\
    &\quad \text{ where }\epsilon := \half S(-\xi), \quad \xi := \frac{\sigma_t^2}{3 \sigma_t^2 \tau^2 + \sigma_t^4 + \tau^4}.
\end{align*}

For any fixed $\tau$, the minimum and maximum of $\epsilon$ are
\begin{align*}
    \sigma_t &\to 0 \implies \xi \to 0 \implies \epsilon \to \frac{1}{4} \text{ (max)}\\
    \sigma_t &= \tau \implies \xi = \xi^\star(\tau) := \frac{1}{5\tau^2} \implies \epsilon = \half S \left(-\xi^\star(\tau) \right) \text{ (min)}\\
    \sigma_t &= \tau \to 0 \implies \xi^\star(\tau) \to \infty \implies \epsilon \to 0 \text{ (min)}
\end{align*}

First we want to show that
$$\exists t: W_2(N_t[q^0], q^t) \geq \half.$$

Taking $\sigma_t = \tau$ (the minimizer of $\epsilon$) and choosing any $\tau^2 < \frac{1}{82}$ (somewhat arbitrarily, but small enough), gives
\begin{align*}
    \sigma_t = \tau &\implies (1 - 4\epsilon)^{\half} - 40\tau^2 \le W_2(N_t[\hat{p}^0], \hat{p}^t) \le \sqrt{2}(1 - 4\epsilon)^\half,
    \quad \text{ where }\epsilon := \half S(-\frac{1}{5 \tau^2}) \\
    \tau^2 < \frac{1}{82} &\implies \epsilon \approx 10^{-7} \ll \frac{1}{4}\tau^2 \implies (1 - 4\epsilon)^{\half} \ge 1 - \tau^2 \\
    &\implies 0.5 \le 1 - 41 \tau^2 \le W_2(\hat{p}^0, \hat{p}^t) \le \sqrt{2}.
\end{align*}
Thus, fixing a $\tau^2 < \frac{1}{82}$ and taking $\sigma_t := t = \tau$,  we have as desired that
\begin{align*}
      W_2(N_t[\hat{p}^0], \hat{p}^t) \geq 0.5.
\end{align*}

The proof is now complete, but we can make one more interesting observation. The bound above was obtained by choosing a small value of $\tau$, but the diffusion path (specifically, for distributions of the form of \cref{def:counterexample}) is much less problematic for larger $\tau$:
$$\forall t: W_2(N_t[\hat p^0], \hat p^t) \leq \cO(\tau^{-1}).$$
That is, even for our counterexample distributions, diffusion \emph{can} still approximately work to sample from the composition $\hat p^0$, if $\tau$ is large enough.
To see this, we note that
\begin{align*}
    \forall t: \quad W_2(N_t[\hat{p}^0], \hat{p}^t) &\le \sqrt{2}(1 - 4\epsilon)^\half \\
    \text{ where }\epsilon &:= \half S(-\xi), \quad \xi := \frac{\sigma_t^2}{3 \sigma_t^2 \tau^2 + \sigma_t^4 + \tau^4} \\
    & \geq \half S \left(-\xi^\star(\tau) \right), \quad \xi^\star(\tau) := \frac{1}{5\tau^2} \quad \text{for any fixed $\tau$} \\
    \implies W_2(N_t[\hat{p}^0], \hat{p}^t) &\le \sqrt{2}(1 - 2 S \left(-\xi^\star(\tau) \right))^\half
    \le \frac{\sqrt{2}}{\tau}, \quad \forall \tau.
\end{align*}
\end{proof}

\section{Learned vs. explicit compositions in CLEVR trained on multiple objects}
\label{app:clevr_multi}

\begin{figure}[ht]
  \centering
\includegraphics[width=.49\linewidth]{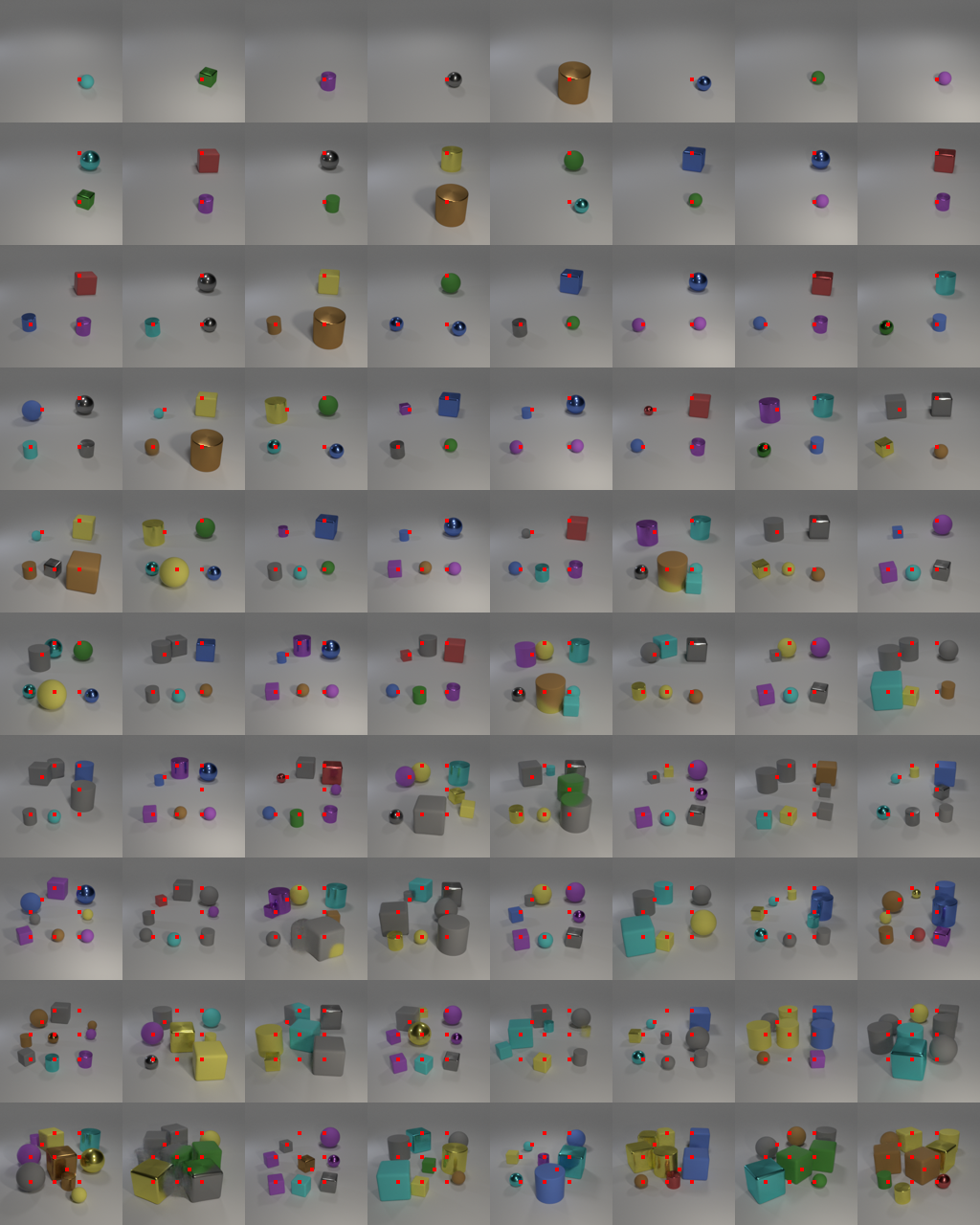}
\includegraphics[width=.49\linewidth]{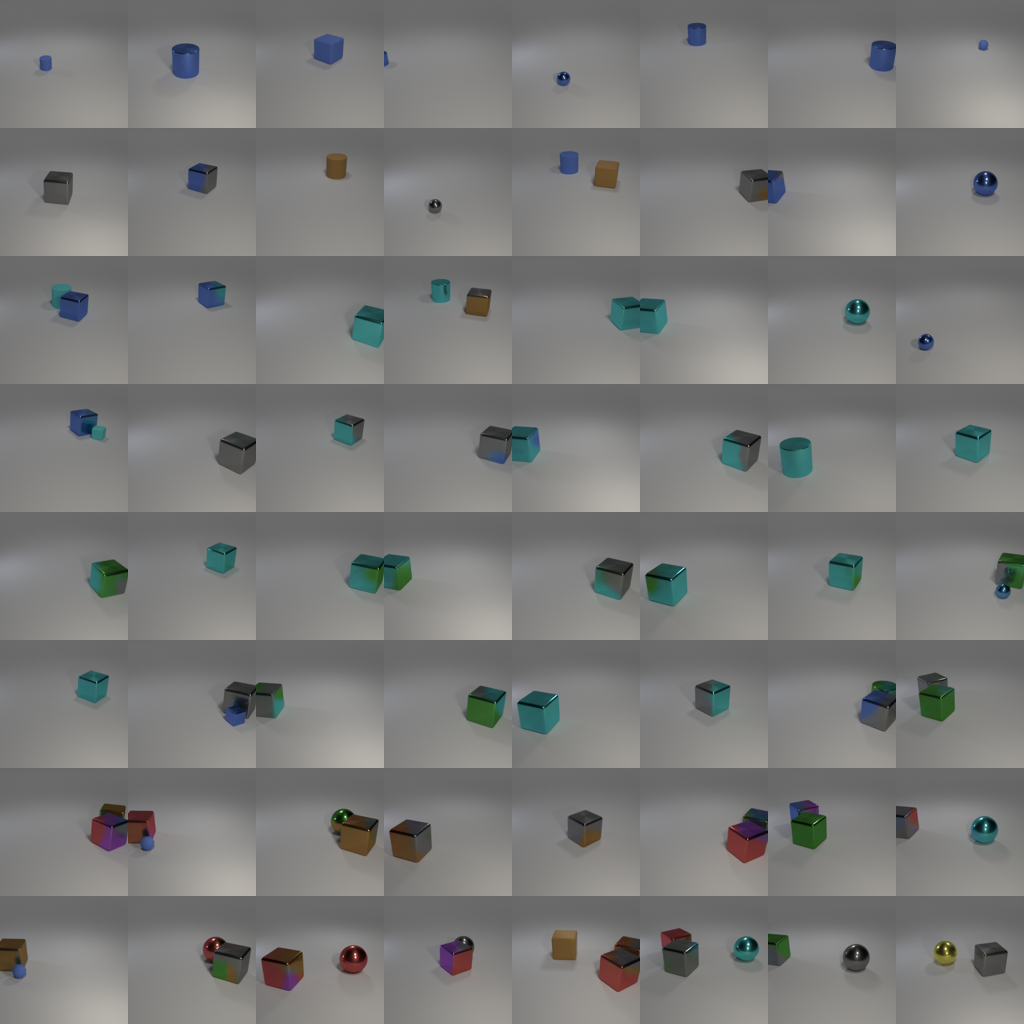}
\caption{Compositions of models trained on multiple objects with a learned background. (Left) Location-conditioned model, trained on 1-3 objects, testing length-generalization from 1-10 objects. We compose models conditioned on one location with a background conditioned on zero locations. (Right) Color-conditioned model, trained on 1-5 objects, testing length-generalization from 1-8 objects.}
\label{fig:explicit_location_color}
\end{figure}

In contrast to the \emph{explicit composition} of multiple models via linear score combination that we study in this paper, \citet{bradley2025local} demonstrated length-generalization in a \emph{single} location-conditioned model trained on CLEVR images with multiple objects -- which we will call \emph{learned composition}. For example, \citet{bradley2025local} showed that a location-conditioned model trained on 1-3 objects can length-generalize to 9 objects when given an OOD conditioner specifying 9 locations. Their location-conditioning is implemented as a 2d integer array representing a 16x16 grid over the image, where each entry is the count of objects whose center falls within the grid cell as shown in (usually 0 or 1). To test length-generalization, they constructed new conditioners with $K$ nonzero entries corresponding to $K$ desired locations. Similarly, to explore feature-space compositionality, they also showed partial length-generalization in a single color-conditioned model, where the color conditioning used an 8-dimensional integer array (one dimension per color) indicating object counts for each color. They showed that a color-conditioned model trained on 1-5 objects length-generalized up to 7 objects. Notably, \citet{bradley2025local}'s learned compositions, particular with the location-conditioned model, length-generalize to higher object counts with fewer artifacts than our explicit compositions shown in \Cref{fig:len_gen_monster}.

In this section, we explore whether explicit compositions of multi-object models trained with \citet{bradley2025local}'s conditioning strategy also exhibit improved length-generalization. \Cref{fig:explicit_location_color} shows explicit compositions using \citet{bradley2025local}'s location-conditioned model (trained on 1-3 objects) and color-conditioned model (trained on 1-5 objects). For these explicit compositions, individual concept models are created by conditioning the multi-object model on a single location or color at a time. A \emph{learned background} model is obtained by passing the multi-object model an empty conditioner (which is also OOD). We find the location-conditioned explicit composition length-generalizes up to about 8 objects, while the color-conditioned explicit composition shows very limited length-generalization. 

Table \ref{table:xy-learned-comp-variations} provides quantitative results for length-generalization of learned vs. explicit compositions for location-conditioned models, trained on varying numbers of objects. For the explicit compositions, we also compare different backgrounds. The evaluation criterion and the results of the learned composition are from \citet{bradley2025local}. \Cref{fig:rel_err_learned_explicit} measures the difference between the learned and explicit composition scores in location-conditioned models across denoising times.

\paragraph{Location-conditioned models} For location-conditioned multi-object models, the explicit composition with learned background length-generalizes almost as well as the learned composition. A simple explanation for this is that if the true joint distribution $p(x|c_1, c_2, \ldots)$ is a projective composition, a model trained on sufficient conditioner-combinations may learn this underlying structure, and also, $p(x|c_1, c_2, \ldots)$ is exactly equal to the explicit composition of its individual conditionals $\{p(x|\emptyset), p(x|c_1), p(x|c_2), \ldots\}$. However, at high object counts the objects become crowded, making the distribution less compositional. This may be mitigated by the fact that at low noise levels the model appears to behave like an \emph{local, unconditional denoiser} \citep{kamb2024analytic, niedoba2024towards}. Thus, if the projective structure approximately holds at high noise then the compositional mechanism may establish the global structure, with details resolved at low noise. We offer the following evidence for this.
First, the score visualizations in \Cref{fig:scores_learn_v_expl} show that at high noise ($t=20$), the scores of the learned and explicit compositions look similar, while at low noise \emph{all} scores resemble the unconditional score. Second, \Cref{fig:rel_err_learned_explicit} shows that the gap between the unconditional and learned scores is large at high noise but nearly zero at low noise. Conversely, the gap between the learned and explicit compositions is relatively small at high noise (suggesting a shared compositional structure), widens at intermediate noise, and finally disappears at low noise. We therefore hypothesize that both the learned and explicit compositions exploit the (approximate) compositional structure at high noise, while reverting to local unconditional denoising at low noise.

Finally, we observe that the learned and explicit compositions do differ somewhat: the learned composition length-generalizes to more objects, and \Cref{fig:rel_err_learned_explicit} shows that their scores diverge, particularly at intermediate noise levels. Why might this be? First, even at high noise levels, the compositional structure of the data is only approximate. In fact, it is rather surprising that the learned composition length-generalizes to as many objects as it does, given that they become quite crowded (and thus less independent) even at intermediate noise levels.  One hypothesis proposed by \citet{bradley2025local} is that a model trained on 1-3 objects might learn to represent \emph{clusters} of 1-3 objects, which it can then compose.\footnote{This is consistent with our existing projective composition theory, applied to concept distributions corresponding to \emph{subsets} of conditioners. If the underlying data has this type of compositional structure, a trained model could learn to group conditioners to exploit it.} If this is the case, the model could generate, for instance, 12 objects by composing four 3-object clusters. In contrast, the explicit composition (as currently constructed, one object at a time) assumes independence between objects, at least at high noise.

\begin{table}
  \caption{Limits of length-generalization in learned vs. explicit compositions. The table lists the maximum value, $K_{\max}$, such that the model ``succeeds'' for every $1 \le K \le K_{\max}$, as defined in \citet{bradley2025local}. The learned composition of \citet{bradley2025local}, conditions a single model on $K$ locations simultaneously. The explicit compositions are $\cC[p_b, \{p_j\}_{j=1:K}]$ where $p_j(x) := p(x|c_j)$, conditions on a single location $j$ at a time. We test various backgrounds for the composition: the learned background $p_b(x) := p(x|\emptyset)$ (i.e. conditioned on zero locations), the empty background (a separate model trained only on empty images), and an unconditional background (another separate unconditional model trained on images containing 1-5 objects). The ``switch'' background experiment uses learned background for the first 12 denoising steps and then switches to the unconditional background. The maximum values in each row are shown in bold.}
  \label{table:xy-learned-comp-variations}
  \centering
  \begin{tabular}{l|l|llll}
    \toprule
    $K_{\max}$ & Learned & $\cC$ (learned bg) & $\cC$ (empty bg) & $\cC$ (uncond bg) & $\cC$ (switch bg) \\
    \midrule
    Trained on 1   & 1 & 2 & \textbf{4} & 1 & 3 \\
    Trained on 1-2 & \textbf{5} & \textbf{5} & 4 & 3 & 3\\
    Trained on 1-3 & \textbf{9} & 8 & 4 & 2 & 8\\
    Trained on 1-4 & \textbf{10} & 9 & 4 & 3 & 8\\
    Trained on 1-5 & \textbf{10} & 9 & 4 & 2 & 8\\
    Trained on 1-6 & \textbf{11} & 9 & 4 & 2 & 8\\
    \bottomrule
  \end{tabular}
\end{table} 

A few additional comments about \Cref{table:xy-learned-comp-variations}:
\begin{itemize}
    \item For single-object models (trained on 1 object), explicit composition with empty background (as in \Cref{fig:len_gen_monster} Exp. A) length-generalizes the best, up to 4 objects. 
    \begin{itemize}
        \item The learned composition unsurprisingly does not length-generalize, having never seen multiple objects.
        \item The model also fails to generalize on the empty conditioner, hence the explicit composition with learned background performs poorly.
        \item The explicit composition with unconditional background (as in to \Cref{fig:len_gen_monster} Exp. B) also fails to length-generalize consistent with our previous observations.
    \end{itemize}
    \item The multi-object explicit compositions with learned background length-generalize more successfully than their counterparts with either empty or unconditional backgrounds. This suggests that the learned background -- which appears to transition from empty to unconditional -- is beneficial for multi-object explicit compositions. The `switch background' experiment (learned background early in denoising followed by unconditional) also supports this hypothesis.
\end{itemize}

\paragraph{Color-conditioning}

The color-conditioned explicit composition was far less successful at length-generalization. Although the learned composition for the model trained on 1-5 objects generalized to 7 colors in \citet{bradley2025local}, the explicit composition of this same model (with learned background) does not length-generalize beyond 2-3 objects. We hypothesize that this is due to the compositional structure existing in \emph{feature-space}. As shown in this paper, sampling explicit compositions in feature-space is challenging. In contrast, as discussed in \citet{bradley2025local}, trained models can potentially learn both the appropriate feature-space transform as well as the compositional structure, leading to better generalization.

We tested the following locations and colors as in \citet{bradley2025local}.
\begin{small}
\begin{verbatim}
locations = ([[0.65, 0.65], [0.65, 0.25], [0.25, 0.65], [0.35, 0.35], [0.45, 0.65], 
[0.45, 0.25], [0.65, 0.45], [0.25, 0.45], [0.45, 0.45], [0.55, 0.55], [0.35, 0.55], 
[0.55, 0.35]]).
colors = ([blue, brown, cyan, gray, green, purple, red, yellow]).
\end{verbatim}
\end{small}

\begin{figure}
  \centering
  \includegraphics[width=0.3\linewidth]{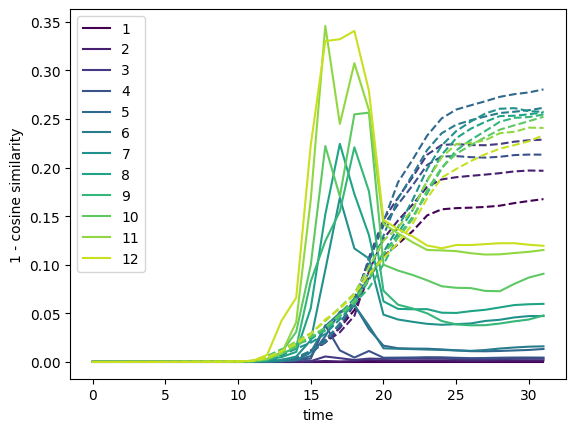}
  \includegraphics[width=0.3\linewidth]{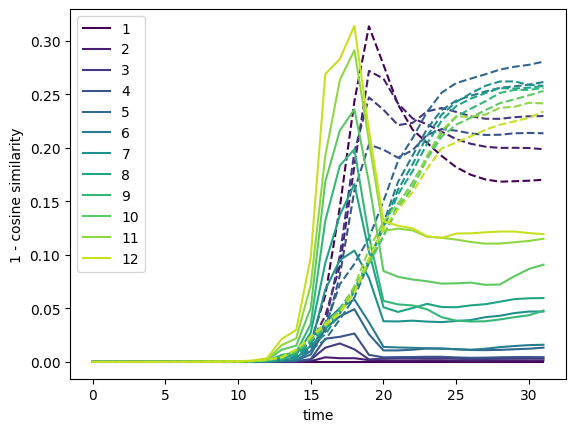}
  \includegraphics[width=0.3\linewidth]{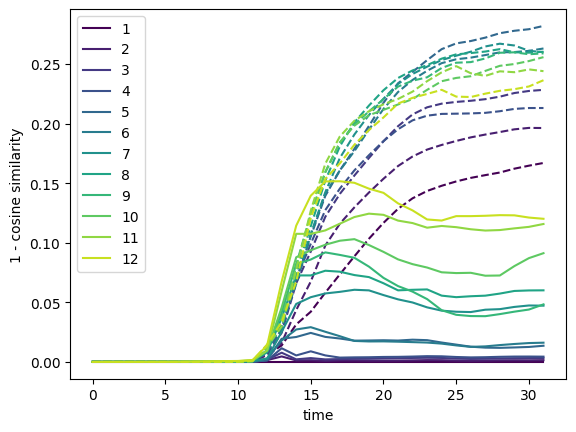}
  \caption{Difference between learned and explicit score (solid lines) and between learned and unconditional score (dashed lines), across denoising timesteps, with a location-conditioned model trained on 1-3 objects. Different colors represent compositions of varying numbers of objects (e.g., 0:6 indicates a composition of 7 distinct locations). The clean image $x_0$ is generated by the learned model, with different conditioning for each subplot: (right) conditioned on the current composition, (middle) conditioned on 6 fixed locations, (left) conditioned on 0 objects.}
  \label{fig:rel_err_learned_explicit}
\end{figure}

\begin{figure}
  \centering
  \includegraphics[width=1.0\linewidth]{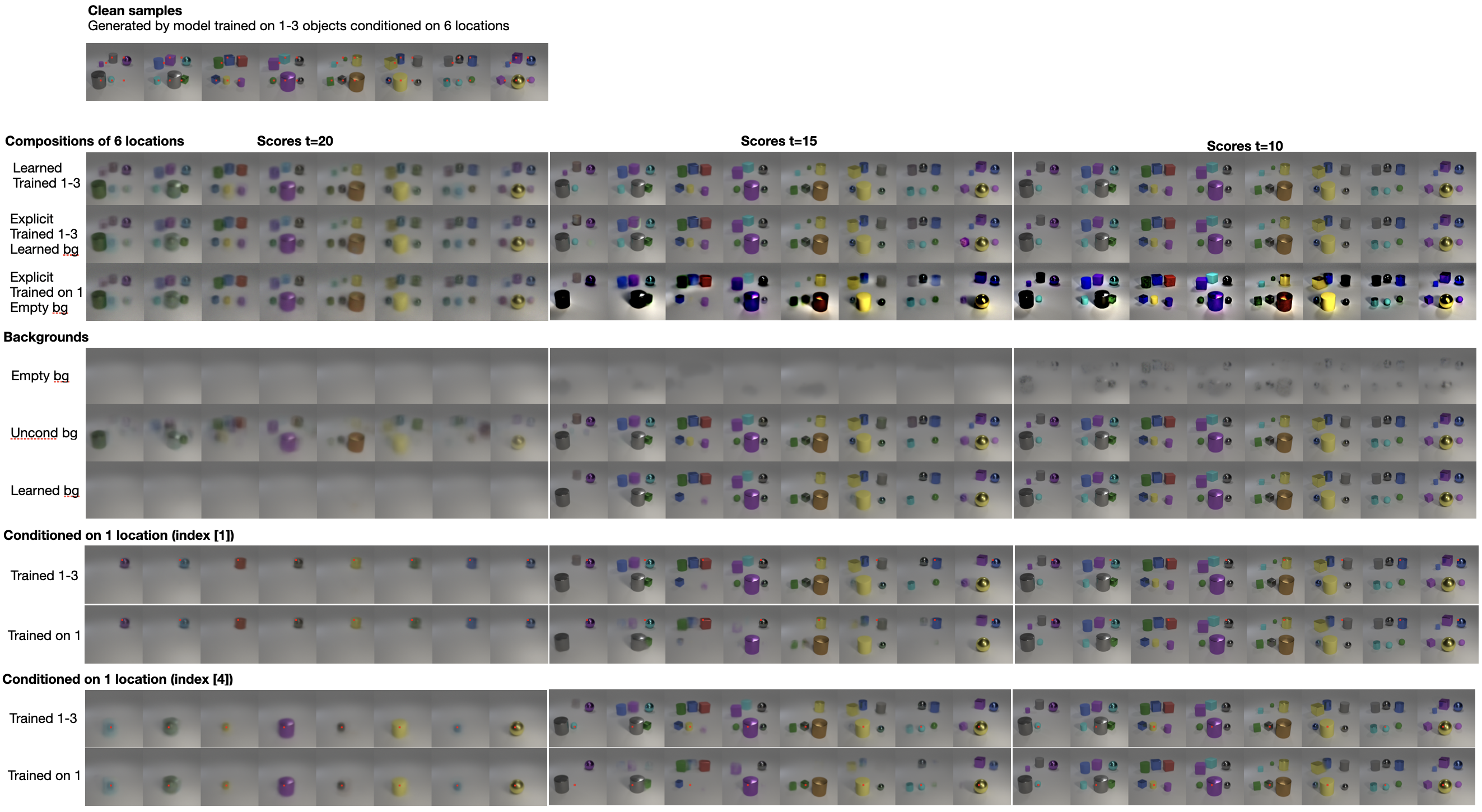}
  \caption{Score visualizations for different models at various noise levels ($t=20, 15, 10$. The clean image $x_9$ contains 6 objects. Models include: Learned composition (trained 1-3 objects), Explicit composition (trained 1-3 objects), and Explicit composition (trained 1 object).}
  \label{fig:scores_learn_v_expl}
\end{figure}

\end{document}